%% file: oipart_preprint.tex
\newcount\Comments  
\documentclass[letterpaper,10pt]{scrartcl}

\usepackage[utf8]{inputenc}
\usepackage[english]{babel}
\usepackage[automark]{scrlayer-scrpage}
\usepackage{amsmath,amssymb,upref}
\usepackage{amsfonts}
\usepackage{amsthm,array,makecell}
\usepackage{color}
\usepackage{graphicx}
\usepackage{sidecap}
\usepackage{multirow}
\usepackage{booktabs}
\usepackage{fullpage}
\usepackage[title]{appendix}

\usepackage[font=small]{caption}
\usepackage{subcaption}

\usepackage{longtable}
\usepackage{rotating}

\usepackage{algorithm}
\usepackage{algcompatible}
\extrafloats{100}
\usepackage{lineno}
\usepackage[hidelinks]{hyperref}

\graphicspath{{figures/}}
\DeclareGraphicsExtensions{.pdf,.eps,.png,.jpg,.jpeg}

\newcommand{\R}{\mathbb{R}}

\newcommand{\bfeps}{\boldsymbol \epsilon}

\newcommand{\bfx}{\boldsymbol x}
\newcommand{\bfX}{\boldsymbol X}

\newcommand{\bff}{\boldsymbol f}

\newcommand{\bfy}{\boldsymbol y}

\newcommand{\Dcal}{\mathcal{D}}

\newcommand{\Vcal}{\mathcal{V}}

\newcommand{\bfmu}{\boldsymbol \mu}

\newcommand{\bfu}{\boldsymbol u}

\newcommand{\bfv}{\boldsymbol v}
\newcommand{\bfw}{\boldsymbol w}
\newcommand{\bfz}{\boldsymbol z}

\newcommand{\bfA}{\boldsymbol A}
\newcommand{\bfB}{\boldsymbol B}
\newcommand{\bfC}{\boldsymbol C}
\newcommand{\bfD}{\boldsymbol D}
\newcommand{\bfE}{\boldsymbol E}
\newcommand{\bfF}{\boldsymbol F}
\newcommand{\bfG}{\boldsymbol G}
\newcommand{\bfH}{\boldsymbol H}

\newcommand{\bfO}{\boldsymbol O}

\newcommand{\bfR}{\boldsymbol R}
\newcommand{\bfU}{\boldsymbol U}
\newcommand{\bfV}{\boldsymbol V}

\newcommand{\bfQ}{\boldsymbol Q}

\newcommand{\bfW}{\boldsymbol W}
\newcommand{\bfY}{\boldsymbol Y}
\newcommand{\bfZ}{\boldsymbol Z}

\newcommand{\PartialObs}{\boldsymbol{P}_o}
\newcommand{\PartialObsPerp}{\PartialObs^{\perp}}
\newcommand{\basisPerp}{\bfV^{\perp}}
\newcommand{\Qperp}{\bfQ^{\perp}}

\newcommand{\nr}{n}

\newcommand{\Nreproj}{N_{r}}
\newcommand{\StageDiscrepState}{\boldsymbol \Delta_{\bfz}}
\newcommand{\StageDiscrepOut}{\boldsymbol \Delta_{\bfy}}

\newcommand{\hbfA}{\hat{\bfA}}
\newcommand{\hbfB}{\hat{\bfB}}
\newcommand{\hbfE}{\hat{\bfE}}
\newcommand{\hbfF}{\hat{\bfF}}
\newcommand{\hbfG}{\hat{\bfG}}
\newcommand{\hbfH}{\hat{\bfH}}

\newcommand{\rbfx}{\bar{\bfx}}
\newcommand{\rbfz}{\bar{\bfz}}

\newcommand{\tbfA}{\tilde{\bfA}}
\newcommand{\tbfB}{\tilde{\bfB}}
\newcommand{\tbfC}{\tilde{\bfC}}
\newcommand{\tbfx}{\tilde{\bfx}}

\newcommand{\rbfX}{\bar{\bfX}}

\newcommand{\hbfO}{\hat{\bfO}}

\newcommand{\bbfO}{\breve{\bfO}}

\newcommand{\tbfy}{\tilde{\bfy}}
\newcommand{\tbfz}{\tilde{\bfz}}
\newcommand{\hbfz}{\hat{\bfz}}
\newcommand{\tbfY}{\tilde{\bfY}}

\newcommand{\hbfD}{\hat{\bfD}}
\newcommand{\bbfD}{\breve{\bfD}}

\newcommand{\tbfZ}{\tilde{\bfZ}}

\newcommand{\bbfE}{\breve{\bfE}}
\newcommand{\bbfF}{\breve{\bfF}}
\newcommand{\bbfG}{\breve{\bfG}}
\newcommand{\bbfH}{\breve{\bfH}}

\newcommand{\Treproj}{K_{r}}
\newcommand{\paramSpace}{\mathcal{D}}
\newcommand{\param}{\boldsymbol{\mu}}

\newcommand{\errOut}{\bfeps^{\text{output}}}
\newcommand{\errStateProj}{\bfeps_{\bfz}^{\text{proj}}}
\newcommand{\errStateStage}{\bfeps_{\bfz}^{\text{stage}}}
\newcommand{\errStateBatch}{\bfeps_{\bfz}^{\text{batch}}}

\newcommand{\closeState}{\bfE}
\newcommand{\closeSignal}{\bfF}
\newcommand{\closeOState}{\bfG}
\newcommand{\closeOSignal}{\bfH}

\newcommand{\kibitz}[2]{\ifnum\Comments=1\textcolor{#1}{#2}\fi}

\newenvironment{keywords}%
   {\begin{trivlist}\item[]{\bfseries\sffamily Keywords:}\ }
   {\end{trivlist}}

\ihead[]{}
\ohead[]{}
\ifoot[]{}
\ofoot[]{}

\newtheorem{theorem}{Theorem}

\newtheorem{proposition}[theorem]{Proposition}

\newtheorem{remark}[theorem]{Remark}

\numberwithin{equation}{section}

\title{Operator inference of non-Markovian terms for learning reduced models from partially observed state trajectories}

\author{Wayne Isaac Tan Uy and Benjamin Peherstorfer\thanks{\{wayne.uy,pehersto\}@cims.nyu.edu, Courant Institute of Mathematical Sciences, New York University, New York, NY 10012}}

\begin{document}

\maketitle

\begin{abstract}
This work introduces a non-intrusive model reduction approach for learning reduced models from partially observed state trajectories of high-dimensional dynamical systems. The proposed approach compensates for the loss of information due to the partially observed states by constructing non-Markovian reduced models that make future-state predictions based on a history of reduced states, in contrast to traditional Markovian reduced models that rely on the current reduced state alone to predict the next state. The core contributions of this work are a data sampling scheme to sample partially observed states from high-dimensional dynamical systems and a formulation of a regression problem to fit the non-Markovian reduced terms to the sampled states. Under certain conditions, the proposed approach recovers from data the very same non-Markovian terms that one obtains with intrusive methods that require the governing equations and discrete operators of the high-dimensional dynamical system. Numerical results demonstrate that the proposed approach leads to non-Markovian reduced models that are predictive far beyond the training regime. Additionally, in the numerical experiments, the proposed approach learns non-Markovian reduced models from trajectories with only 20\% observed state components that are about as accurate as traditional Markovian reduced models fitted to trajectories with 99\% observed components. 

\end{abstract}


\section{Introduction}
There is an increasing interest in methods to learn reduced models of high-dimensional dynamical systems from data. Such non-intrusive---scientific-machine-learning---model reduction approaches are applicable even if little is known about the governing equations and instead mostly data are available. This differs from traditional intrusive model reduction that is typically a rather manual process requiring extensive expertise and full knowledge of the high-dimensional governing equations and their discretizations \cite{RozzaPateraSurvey,paper:BennerGW2015,Quarteroni2011,book:HesthavenRS2016,doi:10.1080/00207170410001713448}. This work proposes a non-intrusive model reduction method to learn reduced models from partially observed state trajectories, rather than fully observed states, of high-dimensional systems.
To compensate for the information lost due to the partially observed states, the proposed reduced models make future-state predictions based on a time history of the reduced states and thus describe non-Markovian dynamics \cite{Givon_2004,chorin2006,doi:10.1098/rspa.2014.0446}. Notice that many traditional reduced models from non-/intrusive model reduction are Markovian and thus make future-state predictions based on the current reduced state alone \cite{RozzaPateraSurvey,paper:BennerGW2015,Quarteroni2011,book:HesthavenRS2016,doi:10.1080/00207170410001713448}. The two core contributions of this work are a data sampling scheme and a formulation of a regression problem to learn the non-Markovian terms of the proposed reduced models from partially observed state trajectories.  
Under certain conditions, in particular if the high-dimensional system dynamics are linear, the proposed data sampling scheme together with the proposed regression problem guarantees recovery of the very same non-Markovian reduced models that one would obtain with intrusive projection-based methods; note that the latter require the high-dimensional system operators in either assembled form or implicitly via methods that provide the action of the high-dimensional operators on vectors in contrast to the proposed approach.
If the high-dimensional system dynamics have nonlinear state dependencies, then the proposed non-Markovian reduced models are approximations of the models obtained with intrusive methods.

There is a growing body of literature on learning reduced models from data. If frequency-response or impulse-response measurements are available, then widely used approaches for learning dynamical-system models include the Loewner approach \cite{ANTOULAS01011986,5356286,Mayo2007634,BeaG12,paper:AntoulasGI2016,paper:GoseaA2018,paper:IonitaA2014}, vector fitting \cite{772353,paper:DrmacGB2015}, and eigensystem realization \cite{doi:10.2514/3.20031,KramerERZ}.
With respect to learning from time-domain state trajectories, dynamic mode decomposition (DMD) \cite{SchmidDMD,FLM:7843190,FLM:6837872,Tu2014391,NathanBook,Williams2015} has been widely adopted which fits linear dynamical systems in the $L_2$ norm. Tools from sparse regression and compressive sensing have been applied in
\cite{Schaeffer6634,paper:BruntonPK2016,Schaeffer2018} to extract sparse representations of governing equations from data. These methods focus on learning the high-dimensional dynamical system while our objective is to learn reduced models that reduce the computational costs of many-query applications. One approach that explicitly targets reduced models is operator inference \cite{paper:PeherstorferW2016}, which learns reduced models of nonlinear dynamical systems from projected state trajectories; see also \cite{paper:BennerGKPW2020,mcquarrie2020datadriven,doi:10.2514/1.J058943,paper:Peherstorfer2019,pehersto15dynamic}. Together with the re-projection data sampling scheme, which alternates between querying the full system and performing projection, operator inference guarantees recovery of the very same reduced models that would be obtained with intrusive model reduction for nonlinear polynomial systems \cite{paper:Peherstorfer2019}. There are  also  \emph{a posteriori} error estimators available for models learned with operator inference from data of high-dimensional linear dynamical systems \cite{uy2020probabilistic}. The work on Lift \& Learn introduced in \cite{QIAN2020132401} proposes a non-intrusive model reduction approach that is applicable to dynamical systems with nonlinear terms beyond polynomials. 

This work builds on operator inference to learn non-Markovian reduced models from partially observed states. The re-projection scheme \cite{paper:Peherstorfer2019} is extended to generate data for learning non-Markovian operators. There are other methods for learning non-Markovian models from data. Closest to our work is the use of time-delay coordinates in \cite{LeClainche2017,Champion2019} to generalize DMD  to non-Markovian models if under-resolved and incomplete dynamics are observed; however, our focus is on learning polynomial systems rather than DMD models and to recover the same non-Markovian dynamics that one would obtain with intrusive projection approaches in certain situations. There are also model reduction methods for fitting reduced models with delay to data sampled from high-dimensional dynamical systems that include a time delay term \cite{BEATTIE2009225,SCHULZE2016125,SCHULZE2018250}. In contrast, the high-dimensional systems that we consider do not necessarily have time-delay terms but rather we introduce non-Markovian (time-delay) terms in our reduced models to compensate for partial state information. Widely used non-Markovian models in nonlinear system identification  \cite{billings2013nonlinear} and time series analysis \cite{shumway2011time} are based on the nonlinear autoregressive moving average model with exogenous inputs (NARMAX) and its linear counterparts such as the  autoregressive moving average model (ARMA) and the autoregressive moving average model with exogenous inputs (ARMAX). Hidden Markov models \cite{zucchini2016hidden} seek to quantify the uncertainty on the unobserved Markovian states given measurements on observed states which are non-Markovian. Non-Markovian dynamics also play an important role in molecular dynamics simulations \cite{Wang2016,Wang2018,Surez2016}. 
Another area of model reduction and data-fit modeling where non-Markovian models have been investigated is in the fluids and reduced modeling community. There, closure modeling aims to capture the discrepancy in the low-dimensional approximation due to truncation; see for example large eddy simulation \cite{sagaut2006large} and closure models for reduced systems obtained via proper orthogonal decomposition \cite{Wang2012}. More recently,  data-driven closure for reduced models \cite{paper:PanK2018,paper:WangRH2020,Wan2018,Maulik2020,Mou2020,doi:10.1137/17M1145136} have been proposed, many of which are based on time-delay coordinates and the Mori-Zwanzig formalism; note that the works \cite{paper:PanK2018,paper:WangRH2020,Wan2018} consider intrusive settings where the full system and/or its discretization are used to construct the closure model. 
In contrast, the goal of this work is learning models from data, rather than deriving a closure. Another key distinction is that we derive non-Markovian models explicitly to compensate for the limited information due to the partially observed state trajectories, rather than closing reduced models and compensating for other data limitations.

This manuscript is organized as follows. Section~\ref{sec:preliminaries} discusses preliminaries on intrusive model reduction for nonlinear polynomial systems and reviews non-intrusive model reduction based on operator inference and re-projection \cite{paper:PeherstorferW2016,paper:Peherstorfer2019}. A formulation of the problem is presented and an example is discussed that highlights the challenges of constructing reduced models from partially observed states. Section~\ref{sec:LearnNonMarkovian} introduces the proposed approach to learn the non-Markovian terms from data. The proposed reduced models describe the dynamics of the observed states and are learned via a novel extended re-projection algorithm and linear regression that can be implemeted with off-the-shelf linear algebra packages. The numerical experiments in Section~\ref{sec:NumExp} show that the proposed non-intrusive approach and the resulting non-Markovian reduced models achieve approximations which are orders of magnitude more accurate than models obtained by directly applying operator inference to partially observed state trajectories.

\section{Preliminaries} \label{sec:preliminaries}

We summarize fundamental concepts of intrusive model reduction in Section~\ref{subsec:TradMOR}. Section~\ref{subsec:OpInfReproj} then discusses non-intrusive model reduction from data via operator inference and re-projection. Section~\ref{subsec:ProbStatement} formulates the problem of learning reduced models from partially observed states. 

\subsection{Intrusive (traditional) model reduction} \label{subsec:TradMOR}

Let $\param \in \paramSpace \subseteq \R^d$ be a parameter in the parameter domain $\paramSpace$ and consider the full model that is a dynamical system with nonlinear polynomial terms up to degree $\ell \in \mathbb{N}$
    \begin{align} \label{eq:FOM}
        \bfx_{k+1}(\param) & = \bff(\bfx_k(\param),\bfu_k(\param)) = \sum_{j=1}^{\ell} \bfA_j(\param) \bfx_k^j(\param) + \bfB(\param) \bfu_k(\param)\,,\quad k =0,\dots,K-1, \\
        \bfy_{k+1}(\param) & = \bfC(\param) \bfx_{k+1}(\param)\,, \notag
    \end{align}
where $k \in \mathbb{N}$ is the time step, $\bfx_k(\param) \in \R^N$ is the state, $\bfu_k(\param) \in \R^p$ is the input, $\bfy_{k+1}(\param) \in \R^s$ is the output,  $\bfB(\param) \in \R^{N \times p}, \, \bfC(\param) \in \R^{s \times N}$, $\bfA_j(\param) \in \R^{N \times N_j}$ are the system matrices with $N_j = \binom{N+j-1}{j}$  for $j=1,\dots,\ell$
and $K \in \mathbb{N}$ with $K \ge 1$. The nonlinear terms in \eqref{eq:FOM} depend on vectors $\bfx_k^j(\param) \in \R^{N_j}$ for $j \ge 2$ which are defined by retaining only the components of the Kronecker product $$\underbrace{\bfx_k (\param) \otimes \cdots \otimes \bfx_k (\param)}_{j}$$ whose factors are unique up to permutation \cite{paper:PeherstorferW2016}. 

To obtain a reduced model with state  dimension $n \ll N$ from \eqref{eq:FOM},  a projection-based approach via the proper orthogonal decomposition (POD) is performed as follows. For the parameter values $\param_1,\dots, \param_m \in \paramSpace$ with their corresponding state trajectories $\bfX(\param_i) = [\bfx_0(\param_i),\dots,\bfx_{K-1} (\param_i)] \in \R^{N \times K}$ for $i=1,\dots,m$, the first $n$ left singular vectors of $[\bfX(\param_1),\dots,\bfX(\param_m)] \in \R^{N \times mK }$ are computed to form the reduced basis matrix $\bfV_{\bfx} = [\bfv_1,\dots,\bfv_n]\in \R^{N\times n}$. A solution to \eqref{eq:FOM} is sought in the subspace $\mathcal{V}_{\bfx}$ spanned by the columns of $\bfV_{\bfx}$ to obtain the approximation $\bfV_{\bfx} \tbfx_k(\param)$ of $\bfx_k(\bfmu)$ where $\tbfx_k \in \R^n$ is the state of the reduced model.

Applying Galerkin projection onto $\mathcal{V}_{\bfx}$, the reduced model is
\begin{align} \label{eq:ROM}
    \tbfx_{k+1}(\param) & = \sum_{j=1}^{\ell} \tbfA_j (\param) \tbfx_k^j (\param) + \tbfB(\param) \bfu_k (\param),  \quad k = 0,\dots,K-1, \\
    \tbfy_{k+1}(\param) & = \tbfC(\param) \tbfx_{k+1} (\param)\,, \notag
\end{align}
for $\param \in \{\param_1,\dots,\param_m\}$ where $\tbfB(\param) = \bfV_{\bfx}^T \bfB(\param) \in \R^{n \times p}$, $\tbfC(\param) = \bfC(\param) \bfV_{\bfx} \in \R^{s \times n}$, $\tbfA_1(\param) = \bfV_{\bfx}^T \bfA_1(\param) \bfV_{\bfx} \in \R^{n \times n}$. The reduced operators $\tbfA_j(\param) \in \R^{n \times n_j}$ with $n_j = \binom{n+j-1}{j}$ for $j \ge 2$ can be derived in a similar manner; see, e.g., \cite{paper:Peherstorfer2019}. The process of constructing the reduced operators via matrix-matrix multiplications of the basis matrix $\bfV_{\bfx}$ and the full model operators is intrusive in the sense that the full model operators are required either in assembled form or implicitly through a routine that provides the action of the full model operators on a vector. To derive the reduced model for parameters $\param \in \paramSpace \setminus \{\param_1,\dots,\param_m\}$, the operators of the  system \eqref{eq:ROM} for $\param \in \{\param_1,\dots,\param_m\}$ are interpolated element-wise \cite{paper:DegrooteVW2009}.

\subsection{Non-intrusive model reduction with operator inference and re-projection} \label{subsec:OpInfReproj}
\begin{algorithm}[t]
\caption{Data sampling with re-projection}
\begin{algorithmic}[1]
  \STATE Set $\rbfx_0(\param) = \bfV_{\bfx}^T\bfx_0 (\param)$ 
  \FOR{$k = 0, \dots, K-1$}
        \STATE Query \eqref{eq:FOM} for a single time step to obtain $\bfx_{\text{tmp}} = \sum_{j=1}^{\ell} \bfA_j(\param)(\bfV_{\bfx} \rbfx_k(\param))^j + \bfB(\param) \bfu_k(\param)$
        \STATE Set $\rbfx_{k+1}(\param) = \bfV_{\bfx}^T \bfx_{\text{tmp}}$
  \ENDFOR
  \STATE Return $[\rbfx_0(\param), \dots, \rbfx_{K}(\param)]$
\end{algorithmic}
\label{alg:reprojection}
\end{algorithm}
Operator inference \cite{paper:PeherstorferW2016} is a non-intrusive model reduction approach that learns operators of a low-dimensional dynamical-system model from data without requiring the full-model operators \eqref{eq:FOM} in assembled or implicit form. Together with the data sampling  via re-projection \cite{paper:Peherstorfer2019}, operator inference can exactly recover the reduced operators $\tbfA(\param)_1,\dots,\tbfA_{\ell}(\param), \tbfB(\param), \tbfC(\param)$. 

In data sampling with re-projection, for each $\param \in \{\param_1,\dots,\param_m\}$, the full system \eqref{eq:FOM} is queried at an initial condition  $\bfx_0(\param) \in \mathcal{V}_{\bfx}$ and inputs $\bfu_0(\param),\dots,\bfu_{K-1}(\param)$  for a single time step followed by a projection step to derive the re-projected trajectories $\rbfx_0(\param),\dots,\rbfx_K(\param) \in \mathbb{R}^n$; see  Algorithm~\ref{alg:reprojection}. 
Let $\rbfX^{j}(\param) = [\rbfx_0^{j}(\param),\dots,\rbfx_{K-1}^{j}(\param)] \in \R^{n_j \times K}$ for $j=1,\dots,\ell$, $\bfW(\param) = [\rbfx_1(\param),\dots,\rbfx_K(\param)] \in \R^{n \times K}$, and $\bfU(\param) = [\bfu_0(\param),\dots,\bfu_{K-1}(\param)] \in \R^{p \times K}$. For each $\param \in \{\param_1,\dots,\param_m\}$, the least squares problem
\begin{align} \label{eq:MarkovianLS}
    \min_{\hbfO(\param)} \|\bfD(\param)^T \hbfO(\param)^T - \bfW(\param)^T\|_F^2
\end{align}
with the data matrix 
\begin{align*}
    \bfD(\param) = \begin{bmatrix} \rbfX^{1}(\param) \\
    \vdots \\ \rbfX^{\ell}(\param) \\ \bfU(\param)
    \end{bmatrix} \in \R^{(\sum_{j=1}^{\ell} n_j + p) \times K}
\end{align*}
is then solved for 
\begin{align*}
   \hbfO(\param) = \begin{bmatrix}
    \hbfA_1(\param) & \dots & \hbfA_{\ell}(\param) & \hbfB(\param)
   \end{bmatrix} \in \R^{n \times (\sum_{j=1}^{\ell} n_j + p)}.
\end{align*}
Under mild conditions on $\bfD(\param)$, it was shown in \cite{paper:Peherstorfer2019} that the unique optimal solution to \eqref{eq:MarkovianLS} is attained at $\hbfA_j(\param) = \tbfA_j(\param)$ for $j=1,\dots,\ell$ and $\hbfB(\param) = \tbfB(\param)$ which are exactly the reduced operators from intrusive model reduction as discussed in Section~\ref{subsec:TradMOR}.

\subsection{Problem formulation} \label{subsec:ProbStatement}
\begin{figure}
    \centering
    \includegraphics[width=0.9\columnwidth]{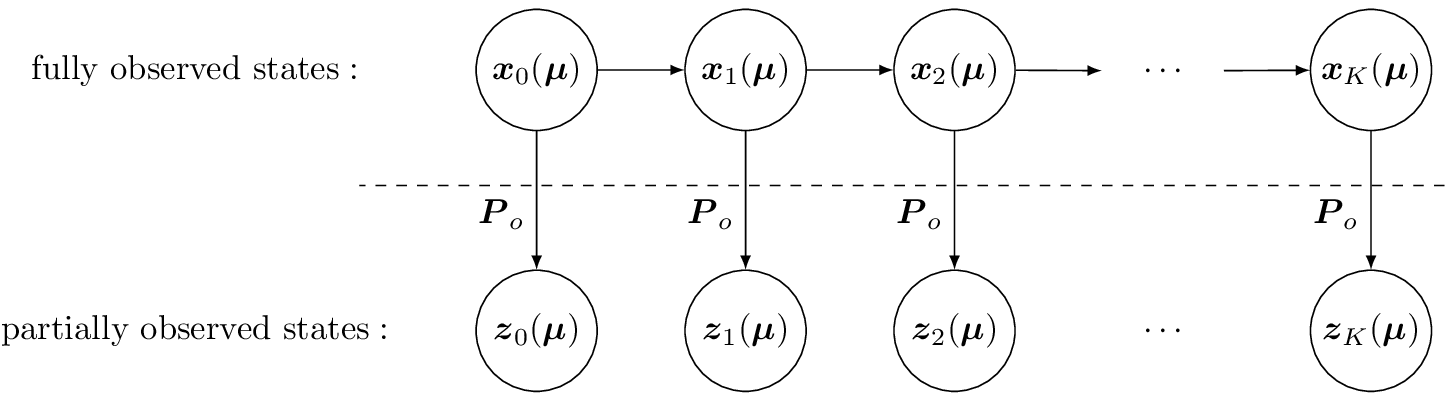}
\caption{The goal of this work is to learn reduced models from observations $\bfz_k(\param) = \PartialObs \bfx_k(\param)$ which represent a few components of the full-model states $\bfx_k(\param)$.}
\label{fig:PartialObs}
\end{figure}

Consider now the situation where only partially observed state trajectories $\bfZ(\bfmu)$ are available, rather than fully observed state trajectories $\bfX(\bfmu)$ of the full model. Formally, define $\PartialObs \in \{0, 1\}^{r \times N}$ the selection matrix whose $r$ rows are a subset of the rows of the $N \times N$ identity matrix. We consider the setting in which the system \eqref{eq:FOM} can be simulated at various initial conditions and inputs to generate observation trajectories  $\bfZ(\bfmu) = [\bfz_0(\bfmu), \dots, \bfz_K(\bfmu)]$ and outputs $\bfy_0(\bfmu), \dots, \bfy_K(\bfmu)$ that are related to the (unobserved) state trajectories $\bfX(\bfmu)$ of the full model \eqref{eq:FOM} via
    \begin{align} \label{eq:PartialObsFOM}
        \bfz_{k}(\param) & = \PartialObs \bfx_{k}(\param)\,,\qquad k = 0,\dots,K\,,
    \end{align}
see Figure~\ref{fig:PartialObs}. Implicit and explicit availability of the matrix $\PartialObs$ is not required in the following.

The aim is now to derive a low-dimensional model with operator inference that can predict the observations \eqref{eq:PartialObsFOM} of the full model at new parameter values in $\Dcal$ and new inputs. Directly applying operator inference and re-projection (Section~\ref{subsec:OpInfReproj}) to trajectories of observations $\bfZ(\bfmu)$ can lead to models that poorly approximate the dynamics of the full model. To see this, consider a POD basis matrix $\bfV$ with $n$ columns derived from the snapshot matrix of observations $\bfZ = [\bfZ(\bfmu_1), \dots, \bfZ(\bfmu_m)]$. Then, the projection from the fully observed, high-dimensional state $\bfx_k(\bfmu)$ to the reduced observation $\tbfz_k(\param) \in \R^n, n < r,$ is given by $(\PartialObs^T\bfV)(\PartialObs^T\bfV)^T$. If a reduced model of the form \eqref{eq:ROM} is learned from trajectories of observations, then the reduced states are in the subspace spanned by columns of $\PartialObs$; however, in contrast to, e.g., the POD basis $\bfV_x$ of the fully observed states, the columns of $\PartialObs^T$ are canonical unit vectors that form a basis of a subspace that typically offers an inadequate approximation quality of the full model states.
    
As an illustration, consider the nonlinear polynomial full model of degree $\ell = 3$ stemming from the spatio-temporal discretization of the parameter-independent Chafee-Infante equation; details of this numerical experiment are presented in Section~\ref{subsec:Chafee}. First, consider a reduced model constructed with operator inference and re-projection as discussed in Section~\ref{subsec:OpInfReproj} when all components of the states can be observed. Figure~\ref{fig:ChafeeProbStatementLeft} shows that with $n = 10$ dimensions, the reduced model output approximates the full model output well. In contrast, learning a reduced model with the non-intrusive model reduction process described in Section~\ref{subsec:OpInfReproj}  with observations that contain only 60\% of the components of the full model states, rather than all components, leads to the results in  Figure~\ref{fig:ChafeeProbStatementRight}. The reduced model trajectories plotted for the output are computed using a test input; cf.~Section~\ref{subsec:Chafee} for details. The model learned from the partially observed states provides a poor approximation of the output of the full model.

\begin{figure}

  \begin{subfigure}[b]{0.45\textwidth}
    \begin{center}
{{\Large\resizebox{1.15\columnwidth}{!}{\input{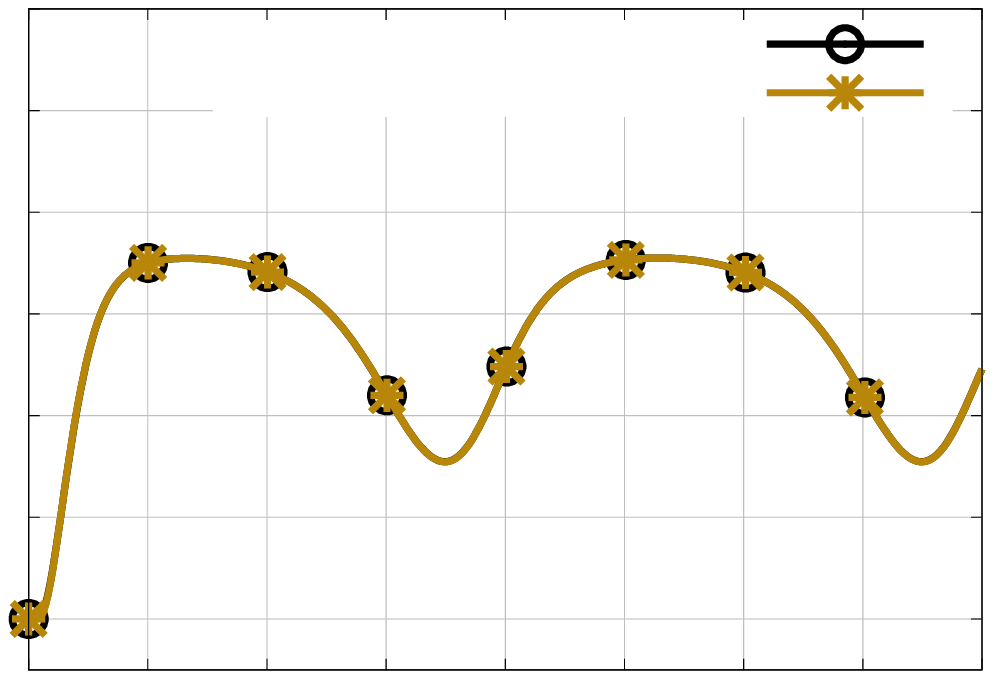}}}}
\end{center}
    \caption{fully observed state}
    \label{fig:ChafeeProbStatementLeft}
  \end{subfigure} 
  \quad \quad \quad
   \begin{subfigure}[b]{0.45\textwidth}
    \begin{center}
{{\Large\resizebox{1.15\columnwidth}{!}{\input{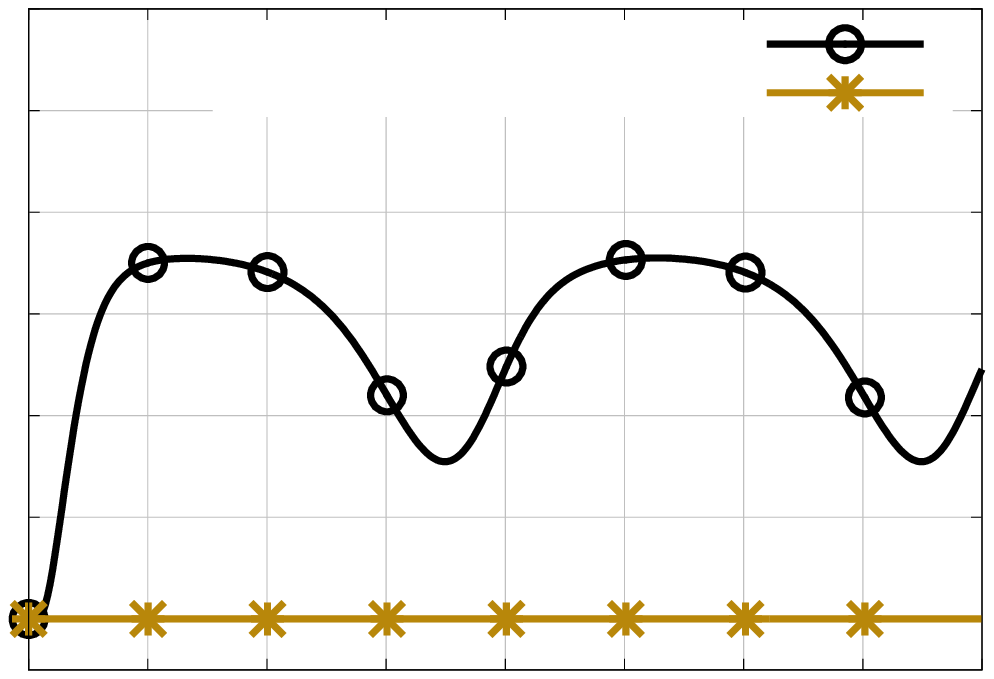}}}}
\end{center}
    \caption{60\% observed state components}
    \label{fig:ChafeeProbStatementRight}
  \end{subfigure} 
  \caption{Chafee-Infante: The left panel shows that a reduced model learned with operator inference and re-projection (Section~\ref{subsec:OpInfReproj}) from fully observed state trajectories approximates well the output of the full model. However, applying operator inference when only 60\% of the state components (observations) are available, rather than all state components, leads to a reduced model that fails to predict the output of the full model in this example, as shown in the right panel.}
  \label{fig:ChafeeProbStatement}
\end{figure}

\section{Learning non-Markovian reduced models with operator inference} \label{sec:LearnNonMarkovian}

To compensate for the loss of information of learning from partially observed state trajectories, we learn non-Markovian terms that take into account the history (memory) of reduced states at previous time steps to correct reduced models; cf.~Mori-Zwanzig formalism \cite{paper:ChorinHK2002,Chorin2009,paper:LinL2021} and neural-network architectures in machine learning such as the long short-term memory network (LSTM) \cite{Hochreiter1997}. This is in stark contrast to traditional, Markovian reduced models of the form \eqref{eq:ROM} where the reduced state at the current time step only depends on a single reduced state at the previous time step. 

Sections~\ref{subsec:Linear} and~\ref{subsec:Nonlinear}
propose a parametrization of reduced models with non-Markovian terms for linear and  nonlinear polynomial full models. A computational procedure to learn the reduced model operators of the non-Markovian terms is introduced in Section~\ref{subsec:OpInfClosure}, which discusses an extension of the re-projection algorithm as well as two modes in which operator inference can be applied.

For ease of notation, we drop the parametric dependence on the states, inputs, and operators because the non-Markovian terms described below are learned for each $\param \in \{\param_1,\dots,\param_m\}$ separately, analogous to the Markovian operators as discussed in Section~\ref{subsec:TradMOR}.
    
\subsection{Reduced systems for partially observed states with linear full dynamics} \label{subsec:Linear}

We first address the setting in which the full model \eqref{eq:FOM} is linear in the state variables. We have available observation trajectories $\bfZ = [\bfz_0, \dots, \bfz_{K}]$ that contain selected components of the state trajectory $\bfX = [\bfx_0, \dots, \bfx_{K}]$ as defined in \eqref{eq:PartialObsFOM}.  Let $\bfV \in \R^{r \times n}$ be the reduced basis matrix obtained via POD from snapshots of observations and $\Vcal$ be the subspace spanned by its columns; cf~Section~\ref{subsec:ProbStatement}. The goal is to derive a model that describes exactly the dynamics of the projected observations $\tbfz_k = \bfV^T \bfz_k \in \R^n$ and the output $\bfy_k = \bfC \bfx_k$. Note that the projected observations $\tbfz_k$ are obtained by first projecting the state $\bfx_k$ according to $\PartialObs$ and then projecting the observation $\bfz_k$ onto the space spanned by the columns of $\bfV$. The output can also be viewed as a projection from $\mathbb{R}^N$ to $\mathbb{R}^s$ if the output dimension $s$ satisfies $s < N$. 

\subsubsection{Dynamics of the projected observations} \label{subsubsec:LinearStateClosure}

Define $\PartialObsPerp \in \R^{(N-r) \times N}$ to be the matrix that extracts the components of $\bfx_k$ which are missing from $\bfz_k$. If $\boldsymbol{0}_{m \times n}$ is an $m \times n$ matrix of zeros, $\PartialObsPerp$ satisfies the relationship $\PartialObs (\PartialObsPerp)^T = \boldsymbol{0}_{r \times N-r}$ such that the columns of $\PartialObs$ and $\PartialObsPerp$ are orthonormal. We can therefore express $\bfx_k$ as a direct sum via
\begin{align} \label{eq:FullStateDirSum}
        \bfx_k = \PartialObs^T \bfz_k+ (\PartialObsPerp)^T \bfz_k^{\perp}\,,\qquad k = 0, \dots, K\,,
\end{align}
for $\bfz_k^{\perp} \in \R^{N-r}$. In addition, denote by $\basisPerp \in \R^{r \times (r-n)}$ the matrix whose columns form an orthonormal basis for the orthogonal complement of $\mathcal{V}$. If $\hbfz_k \in \R^{r-n}$, $\bfz_k$ can also be expressed as a direct sum 
\begin{align}\label{eq:PartialStateDirSum}
        \bfz_k = \bfV \tbfz_k + \basisPerp \hbfz_k\,,\qquad k = 0, \dots, K\,.
\end{align}
Substituting \eqref{eq:PartialStateDirSum} into \eqref{eq:FullStateDirSum}, $\bfx_k$ thus admits the direct sum decomposition
\begin{align} \label{eq:ROMDirSum}
        \bfx_k = \bfQ \tbfz_k + \Qperp \bfw_k\,,\qquad k = 0, \dots, K\,,
\end{align}
where $\bfQ = \PartialObs^T \bfV \in \R^{N \times n}$, $\Qperp = [ \PartialObs^T \bfV^{\perp} \,\, (\PartialObsPerp)^T] \in \R^{N \times (N-n)}$, and $\bfw_k = [\hbfz_k^T \,\, (\bfz_k^{\perp})^T]^T$. It can be verified that the columns of the matrices $\bfQ,\Qperp$ are orthonormal and that $\bfQ^T \Qperp = \boldsymbol{0}_{n \times N-n}$.

A system of equations of the time evolution of the projected states $\tbfz_k$ and the orthogonal complement $\bfw_k$ can be obtained by substituting \eqref{eq:ROMDirSum} to \eqref{eq:FOM} and pre-multiplying the resulting expression by $\bfQ^T$ or $(\Qperp)^T$, thereby giving the system
\begin{align} 
\label{eq:ResSys}
\tbfz_{k+1} & = \bfQ^T \bfA_1\bfQ \tbfz_k + \bfQ^T \bfA_1 \Qperp \bfw_k + \bfQ^T \bfB \bfu_k\,,\qquad k = 0, \dots, K-1\,, \\
\bfw_{k+1} & = (\Qperp)^T \bfA_1 \bfQ \tbfz_k + (\Qperp)^T \bfA_1 \Qperp \bfw_k + (\Qperp)^T\bfB \bfu_k\,,\qquad k = 0, \dots, K-1\,. \label{eq:UnResSys}
\end{align}

To extract a dynamical system in terms of $\tbfz_k$ only, we choose the initial condition $\bfx_0$ such that $\bfw_0 = (\Qperp)^T \bfx_0 = \boldsymbol{0}_{N-n}$. We view $\tbfz_k$ as a constant in \eqref{eq:UnResSys}, solve for $\bfw_k$ as
\begin{align} \label{eq:UnResSoln}
        \bfw_k = \sum_{l=0}^{k-1} ((\Qperp)^T \bfA_1 \Qperp)^{k-l-1}( (\Qperp)^T \bfA_1 \bfQ \tbfz_l + (\Qperp)^T\bfB \bfu_l) ,
\end{align}
and combine \eqref{eq:UnResSoln} with \eqref{eq:ResSys} to deduce
\begin{align} \label{eq:ResSoln}
        \tbfz_{k+1} = \underbrace{\tbfA_1 \tbfz_k + \tbfB \bfu_k}_{\substack{\text{Markovian}\\\text{term}}} +  \underbrace{\sum_{l=0}^{k-1} (\closeState_{k-l} \tbfz_l + \closeSignal_{k-l} \bfu_l)}_{\text{non-Markovian term}}\,,\qquad k = 0, \dots, K-1\,,
\end{align}
where $\tbfA_1 = \bfQ^T \bfA_1 \bfQ$ and $\tbfB = \bfQ^T \bfB$ are the operators for the Markovian term while the operators
\begin{align} \label{eq:CloseOpStateTrue}
        \closeState_{k-l} & = \bfQ^T \bfA_1 \Qperp ((\Qperp)^T \bfA_1 \Qperp)^{k-l-1} (\Qperp)^T \bfA_1 \bfQ  \in \R^{n \times n}\,,  \\
        \closeSignal_{k-l} & = \bfQ^T \bfA_1 \Qperp ((\Qperp)^T \bfA_1 \Qperp)^{k-l-1} (\Qperp)^T \bfB \in \R^{n \times p}\,, \notag
\end{align}
for  $l = 0,\dots,k-1$ give rise to the non-Markovian term. Deriving the dynamical system \eqref{eq:ResSoln} is analogous to what the Mori-Zwanzig formalism suggests \cite{paper:ChorinHK2002,Chorin2009,paper:LinL2021} for linear systems with inputs. Observe that \eqref{eq:ResSoln} shows the dependence of the projected observation $\tbfz_{k + 1}$ at time $k + 1$ on all previous projected observations and inputs. It is comprised of the Markovian term $\tbfA_1 \tbfz_k + \tbfB \bfu_k$ and the non-Markovian term that introduces dependence on time steps before $k$. 

\begin{remark}
The condition that $(\Qperp)^T \bfx_0 = \boldsymbol{0}_{N-n}$ is met, for instance, if $\bfx_0 = \PartialObs^T \bfV \bfV^T \bfz_0$ where $\bfz_0$ is the initial state in the available observation trajectory $\bfZ$. If this is not the case, the dynamical system \eqref{eq:ResSoln} for $\tbfz_k$ becomes
\begin{align} \label{eq:ResSolnICnonZero}
    \tbfz_{k+1} = \tbfA_1 \tbfz_k + \tbfB \bfu_k +  \sum_{l=0}^{k-1} (\closeState_{k-l} \tbfz_l + \closeSignal_{k-l} \bfu_l) + \boldsymbol{\Psi}_k \bfw_0\,,\qquad k = 0, \dots, K-1\,,
\end{align}
where $\boldsymbol{\Psi}_k = (\bfQ^T \bfA_1 \Qperp) ((\Qperp)^T \bfA_1 \Qperp)^k \in \mathbb{R}^{n \times N-n}$. Throughout this work, $\bfx_0$ is chosen such that $(\Qperp)^T \bfx_0 = \boldsymbol{0}_{N-n}$ is satisfied. 
\end{remark}
	
\subsubsection{Dynamics of the output} \label{subsubsec:LinearOutClosure}

Following analogous steps as above, the system describing the time evolution of the output can be derived as
    \begin{align*}
        \bfy_k = \bfC \bfx_k = \bfC \bfQ \tbfz_k + \bfC \Qperp \bfw_k\,,\qquad k = 0, \dots, K.
    \end{align*}
With \eqref{eq:UnResSoln},
    this results into 
    \begin{align} \label{eq:OutputResolved}
        \bfy_k = \tbfC \tbfz_k + \sum_{l=0}^{k-1} (\closeOState_{k-l} \tbfz_l + \closeOSignal_{k-l} \bfu_l) \,,\qquad k = 0, \dots, K\,,
    \end{align}
where $\tbfC = \bfC \bfQ$ and
    \begin{align} \label{eq:CloseOpOutTrue}
        \closeOState_{k-l} & = \bfC \Qperp ((\Qperp)^T \bfA_1 \Qperp)^{k-l-1} (\Qperp)^T \bfA_1 \bfQ  \in \R^{s \times n}\,,  \\
        \closeOSignal_{k-l} & = \bfC \Qperp ((\Qperp)^T \bfA_1 \Qperp)^{k-l-1} (\Qperp)^T \bfB \in \R^{s \times p}\,, \notag
    \end{align}
for $l=0,\dots,k-1$.

\subsubsection{Linear reduced models for the state and output with non-Markovian terms} \label{subsubsec:LinearProposedClosure}

Typically, the norm of the operators in the non-Markovian term in \eqref{eq:ResSoln} and \eqref{eq:OutputResolved} decay as one goes further back in time. This motivates truncating the non-Markovian term and taking only at most the previous $L \in \mathbb{N}$ states and inputs into account; we refer to $L$ as the lag. Thus, the reduced models we seek to learn from observation trajectories are parameterized as
\begin{align} \label{eq:ClosureLinear}
       \tbfz_{k+1}^{(L)} & = \tbfA_1 \tbfz_k^{(L)} + \tbfB \bfu_k +  \sum_{l=k-L}^{k-1} (\closeState_{k-l} \tbfz_l^{(L)} + \closeSignal_{k-l} \bfu_l), \quad k = 0,\dots,K-1, \\
        \notag
        \bfy_k^{(L)} & = \tbfC \tbfz_k^{(L)} + \sum_{l=k-L}^{k-1} (\closeOState_{k-l} \tbfz_l^{(L)} + \closeOSignal_{k-l} \bfu_l)\,,
\end{align}
where we use the convention that $\tbfz_l^{(L)} = \boldsymbol{0}_{n \times 1}$ and $\bfu_l = \boldsymbol{0}_{p \times 1}$ for negative integers $l \in \mathbb{Z}^-$ in the remainder of this work.

Let us remark on the error of non-Markovian versus Markovian reduced models. Denote by $\tbfz_k^{(0)}$ the reduced state resulting from retaining the Markovian term only in \eqref{eq:ResSoln}, i.e. the state of the Markovian reduced model  $\tbfz_{k+1}^{(0)} = \tbfA_1 \tbfz_k^{(0)} + \tbfB \bfu_k$. The initial conditions for the non-Markovian and Markovian model are set to be identical  $\tbfz_0 = \tbfz_0^{(0)} = \tbfz_0^{(L)}$.
For a fixed time step $k$, if 
\begin{align} \label{eq:TruncClosureErrCondition}
        \|\bfV(\tbfz_k - \tbfz_l^{(L)})\|_2 = \|\tbfz_k - \tbfz_l^{(L)}\|_2 < \|\bfz_k - \bfV \tbfz_k^{(0)}\|_2 - \|\bfz_k - \bfV \tbfz_k\|_2
\end{align}
holds, then, by the triangle inequality, we obtain 
\begin{align} \label{eq:TruncClosureErr}
        \|\bfz_k - \bfV \tbfz_k^{(L)} \|_2 = \|\bfz_k - \bfV \tbfz_k +\bfV (\tbfz_k - \tbfz_k^{(L)}) \|_2 < \|\bfz_k - \bfV \tbfz_k^{(0)}\|_2.
\end{align}
Inequality \eqref{eq:TruncClosureErr} implies 
that, if \eqref{eq:TruncClosureErrCondition} holds, the error of the non-Markovian state $\tbfz_k^{(L)}$ is lower than the error of the Markovian state $\tbfz_k^{(0)}$. From \eqref{eq:PartialStateDirSum}, we see that $\tbfz_k$ is the projection of $\bfz_k$ onto $\mathcal{V}$. Therefore, $ \|\bfz_k - \bfV \tbfz_k^{(0)}\|_2 - \|\bfz_k - \bfV \tbfz_k\|_2 \ge 0$. The condition \eqref{eq:TruncClosureErrCondition} holds, for example, if $k < L+2$ since $\tbfz_k = \tbfz_k^{(L)}$. 

In general, however, examples can be constructed for which \eqref{eq:TruncClosureErr} is violated for selected time steps $k$ with $L$ fixed; see Appendix \ref{appendix}. A theoretical analysis remains future work.

\subsection{Reduced models with non-Markovian terms for nonlinear polynomial systems} \label{subsec:Nonlinear}

For polynomial full models \eqref{eq:FOM}, an analogous procedure as in Section \ref{subsec:Linear} leads to polynomial dynamics of the projected observations $\tbfz_k$ with a polynomial non-Markovian term that takes the history of states and inputs into account. However, the number of summands in the non-Markovian term grows combinatorially in the degree $\ell$ of the full model and involves powers of the projected observations $\tbfz_k$ and the inputs $\bfu_k$ and the Kronecker products of their powers. Thus, even truncating the non-Markovian term at lag $L$---retaining only the summands that are functions of the Markovian state and input and the $L$ previous states and inputs---quickly becomes computationally intractable in terms of number of degrees of freedom and costs of simulating the corresponding model. 

Instead, we propose a two step approximation to design non-Markovian terms for nonlinear polynomial full models:    first, we consider a linear approximation of the non-Markovian term and second, truncate the linearized non-Markovian term at lag $L$. This results in the reduced model
\begin{align} \label{eq:ClosureNonLinear}
       \tbfz_{k+1}^{(L)} & = \sum_{j=1}^{\ell} \tbfA_j (\tbfz_k^{(L)})^j + \tbfB \bfu_k +  \sum_{l=k-L}^{k-1} (\closeState_{k-l} \tbfz_l^{(L)} + \closeSignal_{k-l} \bfu_l), \quad k = 0,\dots,K-1, \\
        \notag
        \bfy_k^{(L)} & = \tbfC \tbfz_k^{(L)} + \sum_{l=k-L}^{k-1} (\closeOState_{k-l} \tbfz_l^{(L)} + \closeOSignal_{k-l} \bfu_l)\,,       
\end{align}
with the reduced state $\tbfz_k^{(L)}$ and output $\bfy_k^{(L)}$ at time $k$. Notice that the model \eqref{eq:ClosureNonLinear} is nonlinear in the state variable only in the Markovian term but linear in the non-Markovian term.

\subsection{Operator inference for non-Markovian terms} \label{subsec:OpInfClosure}

We now introduce a procedure to learn the non-Markovian operators in \eqref{eq:ClosureLinear} and \eqref{eq:ClosureNonLinear} from data. In particular, we will show that if the full model \eqref{eq:FOM} is linear, then the operators $(\bfE_l,\bfF_l)$ and $(\bfG_l,\bfH_l), l = 1,\dots,L$, in the non-Markovian term of \eqref{eq:ClosureLinear} defined in \eqref{eq:CloseOpStateTrue} and \eqref{eq:CloseOpOutTrue} are recovered with our procedure.

We first discuss a data sampling scheme to generate trajectories of observations. The non-Markovian operators are then learned either simultaneously (``batch'') or in a stagewise manner with operator inference.

\subsubsection{Data generation: Extended re-projection algorithm}

We extend the re-projection algorithm of Section~\ref{subsec:OpInfReproj} to generate data for learning the non-Markovian terms. We emphasize that in the following, re-projection is applied to observations, rather than state trajectories. Instead of performing re-projection after each time step as in Algorithm~\ref{alg:reprojection}, we propose to query the full model for $\Treproj>1$ time steps before the next re-projection step is performed; this cycle is repeated as illustrated in Figure~\ref{fig:ReprojExtend}.

Denote the initial condition by $\bfz_0^{(0)} \in \mathbb{R}^r$ and let $\Nreproj$ be the number  re-projection steps. For the $i$-th re-projection step, denote by $\rbfz_0^{(i)} = \bfV^T\bfz^{(i-1)}_{\Treproj} \in \mathbb{R}^n$ the projected observation and set $\bfz_0^{(i)} = \bfV \rbfz_0^{(i)} \in \mathbb{R}^r$. We then query the full model for $\Treproj$ time steps starting from $\bfz_0^{(i)}$ (or the corresponding $\PartialObs^T\bfz_0^{(i)} \in \mathbb{R}^N$) with inputs $[\bfu_0^{(i)}, \dots, \bfu_{\Treproj-1}^{(i)}]$ to compute the observation trajectory $[\bfz_0^{(i)}, \dots, \bfz_{\Treproj-1}^{(i)},\bfz_{\Treproj}^{(i)}]$ and the corresponding output trajectory $[\bfy_0^{(i)}, \dots, \bfy_{\Treproj-1}^{(i)}]$. Set $\rbfz_k^{(i)} = \bfV^T\bfz_k^{(i)}$ for $k = 1, \dots, \Treproj-1$. The $(i+1)$-th re-projection step is subsequently initiated with $\rbfz_0^{(i+1)} = \bfV^T \bfz_{\Treproj}^{(i)}$ and the full model is sampled again for $\Treproj$ time steps. Thus, the proposed data sampling scheme yields the re-projected trajectories $\{[\rbfz_0^{(i)},\dots,\rbfz_{\Treproj - 1}^{(i)}]\}_{i=1}^{\Nreproj}$ and output data $\{[\bfy_0^{(i)},\dots,\bfy_{\Treproj-1}^{(i)}]\}_{i=1}^{\Nreproj}$ as a result. The extended re-projection procedure is summarized in Algorithm~\ref{alg:ClosureReproj}.

\begin{figure}
    \centering
    \includegraphics[width=1\columnwidth]{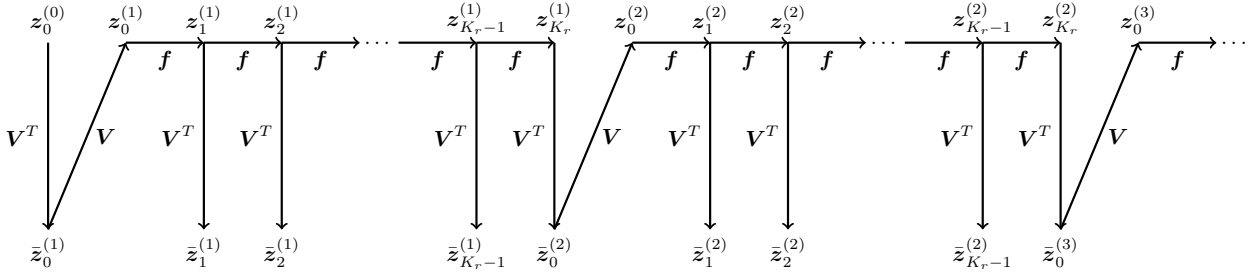}
\caption{Diagram of the extended re-projection algorithm. Instead of performing re-projection steps after each time step, the full model is queried for $K_r$ time steps before re-projection is performed. The observations resulting from the queried full model are then projected to the reduced subspace. }
\label{fig:ReprojExtend}
\end{figure}

\begin{remark} If querying the full model requires an initial condition at each of the $N$ state components, in contrast to Algorithm~\ref{alg:ClosureReproj} that provides an initial condition $\bfz_0^{(0)}$ only at the $r$ observed state components, then such an initial condition can be derived with $\PartialObs^T\bfz_0^{(0)}$ and the proposed re-projection procedure is still applicable. 
\end{remark}

\begin{algorithm}[t]
\caption{Data sampling with extended re-projection}
\begin{algorithmic}[1]
  \STATE Set $\rbfz^{(1)}_0 = \bfV^T\bfz^{(0)}_0$ 
  \FOR{$i=1,\dots,\Nreproj$}
    \STATE Set $\bfz_0^{(i)} = \bfV \rbfz_0^{(i)}$ 
    \STATE Query \eqref{eq:PartialObsFOM} at $\bfz_0^{(i)}$ (or at $\PartialObs^T \bfz_0^{(i)}$) and $[\bfu_0^{(i)},\dots,\bfu_{\Treproj-1}^{(i)}]$ to obtain $[\bfz_0^{(i)},\dots,\bfz_{\Treproj-1}^{(i)},\bfz_{\Treproj}^{(i)}]$ and $[\bfy_0^{(i)},\dots,\bfy_{\Treproj-1}^{(i)}]$  
    \STATE Set $\rbfz_{k}^{(i)} = \bfV^T \bfz_k^{(i)}$ for $k=1,\dots,\Treproj-1$
    
    \STATE Set $\rbfz_0^{(i+1)} = \bfV^T  \bfz_{\Treproj}^{(i)}$
  \ENDFOR

  \STATE Return $\{[\rbfz^{(i)}_0, \dots, \rbfz^{(i)}_{\Treproj-1}]\}_{i=1}^{\Nreproj}$ and $\{[\bfy_0^{(i)},\dots,\bfy_{\Treproj-1}^{(i)}]\}_{i=1}^{\Nreproj}$
\end{algorithmic}
\label{alg:ClosureReproj}
\end{algorithm}

\subsubsection{Stagewise operator inference of non-Markovian operators} \label{subsubsec:stagewise}

As shown in \cite{paper:Peherstorfer2019}, under appropriate conditions on $\Nreproj$ and the resulting data matrix in the least squares problem, the trajectories $\{[\rbfz_0^{(i)},\rbfz_1^{(i)}]\}_{i=1}^{\Nreproj}$ and $\{(\rbfz_0^{(i)},\bfy_0^{(i)})\}_{i=1}^{\Nreproj}$ are sufficient to  recover the Markovian operators in \eqref{eq:ClosureLinear} and \eqref{eq:ClosureNonLinear}. This is because for $i=1,\dots,\Nreproj$, $\rbfz_0^{(i)}$ and $\rbfz_1^{(i)}$ satisfy the Markovian dynamics
$$\rbfz_1^{(i)} = \sum_{j=1}^{\ell} \tbfA_j (\rbfz_0^{(i)})^j + \tbfB \bfu^{(i)}_0.$$ Thus, we will only focus on inferring the non-Markovian  operators. Observe that for fixed $i$, the dynamical system satisfied by $\rbfz_k^{(i)}$ is exactly that of $\tbfz_k$ defined in Sections~\ref{subsec:Linear} and~\ref{subsec:Nonlinear} since $\rbfz_k^{(i)} = \bfV^T\bfz_k^{(i)}$ for $k = 1, \dots, \Treproj-1$. 

We now propose stagewise operator inference where the non-Markovian operators $(\bfE_{l},\bfF_{l})$ and $(\bfG_{l},\bfH_{l})$ in \eqref{eq:ClosureLinear} and  \eqref{eq:ClosureNonLinear} are learned sequentially from data for each $l = 1,\dots, L$. Denote by $(\hbfE_{l},\hbfF_{l})$ and $(\hbfG_{l},\hbfH_{l})$ the estimates of $(\bfE_{l},\bfF_{l})$ and $(\bfG_{l},\bfH_{l})$ for $l \in \mathbb{N}$. For fixed $l$, suppose that the operators $(\hbfE_{j},\hbfF_{j})$ and $(\hbfG_{j},\hbfH_{j})$  for $j=1,\dots,l-1$ have already been inferred. The operators $(\bfE_{l},\bfF_{l})$ and $(\bfG_{l},\bfH_{l})$ are learned by solving for $(\hbfE_{l},\hbfF_{l})$ and $(\hbfG_{l},\hbfH_{l})$ in the least squares problems given by
\begin{align} \label{eq:StageLSState}
        \min_{\hbfE_{l} \in \R^{n \times n}, \hbfF_{l} \in \R^{n \times p}} \sum_{i=1}^{\Nreproj} \|\StageDiscrepState^{(i)}(l,1) -(\hbfE_{l} \rbfz_0^{(i)} + \hbfF_{l} \bfu_0^{(i)})\|_2^2
\end{align}
and
\begin{align} \label{eq:StageLSOut}
        \min_{\hbfG_{l} \in \R^{s \times n}, \hbfH_{l} \in \R^{s \times p}} \sum_{i=1}^{\Nreproj} \|\StageDiscrepOut^{(i)}(l,1) -(\hbfG_{l} \rbfz_0^{(i)} + \hbfH_{l} \bfu_0^{(i)})\|_2^2\,,
\end{align}
where for $\alpha \in \{0,1\}$,
\begin{align} \label{eq:StageDiscrepComp}
        \StageDiscrepState^{(i)}(l,\alpha) &= \rbfz_{l+1}^{(i)} - \left( \sum_{j=1}^{\ell} \tbfA_j (\rbfz^{(i)}_l)^j + \tbfB \bfu_l^{(i)} + \alpha \sum_{j=1}^{l-1} (\hbfE_{l-j} \rbfz_j^{(i)} + \hbfF_{l-j} \bfu_j^{(i)}) \right)\,, \\
        \StageDiscrepOut^{(i)}(l,\alpha) & = \bfy_{l}^{(i)} - \left( \tbfC \rbfz_l^{(i)} + \alpha \sum_{j=1}^{l-1} (\hbfG_{l-j} \rbfz_j^{(i)} + \hbfH_{l-j} \bfu_j^{(i)}) \right). \notag
\end{align}
The quantities $\StageDiscrepState^{(i)}(l,1)$ and $\StageDiscrepOut^{(i)}(l,1)$ are the residuals of $\rbfz_{l+1}^{(i)}$ and $\bfy_l^{(i)}$ with respect to the reduced model with non-Markovian term with lag $l - 1$.

To write the least-squares problems in matrix form, set $\StageDiscrepState(l,1) = [\StageDiscrepState^{(1)}(l,1),\dots,\StageDiscrepState^{(\Nreproj)}(l,1)] \in \R^{n \times \Nreproj}$ and $\StageDiscrepOut(l,1) = [\StageDiscrepOut^{(1)}(l,1),\dots,\StageDiscrepOut^{(\Nreproj)}(l,1)] \in \R^{s \times \Nreproj}$ and let $\hbfO_l^{\bfz} = [\hbfE_{l} \,\,\, \hbfF_{l}] \in \R^{n \times (n+p)}$ and $\hbfO_l^{\bfy} = [\hbfG_{l} \,\,\, \hbfH_{l}] \in \R^{s \times (n+p)}$ be the matrices of unknowns. Set the data matrix as 
\begin{align} \label{eq:StageDataMatrix}
        \hbfD = \begin{bmatrix}
                \rbfz_0^{(1)} & \dots & \rbfz_0^{(\Nreproj)}\\
                \bfu_0^{(1)} & \dots & \bfu_0^{(\Nreproj)}
               \end{bmatrix} \in \R^{(n+p) \times \Nreproj}
\end{align}
so that \eqref{eq:StageLSState} and \eqref{eq:StageLSOut} are, in matrix form,
\begin{align} \label{eq:StageLSStateMatrix}
        \min_{\hbfO_l^{\bfz} \in \R^{n \times (n+p)}} \|\hbfD^T (\hbfO_l^{\bfz})^T - \StageDiscrepState(l,1)^T\|_F^2
\end{align}
and     
\begin{align} \label{eq:StageLSOutMatrix}
        \min_{\hbfO_l^{\bfy} \in \R^{s \times(n+p)}} \|\hbfD^T (\hbfO_l^{\bfy})^T - \StageDiscrepOut(l,1)^T\|_F^2,
\end{align}
respectively. Notice that the data matrix in the least squares problems \eqref{eq:StageLSStateMatrix} and \eqref{eq:StageLSOutMatrix} is the same for all $l = 1, \dots, L$.

Algorithm~\ref{alg:Stagewise} summarizes the stagewise operator inference procedure for the non-Markovian term. To learn the operators up to lag $L$, the number of time steps in the extended re-projection algorithm has to be $\Treproj \ge L+2$.  If the full model \eqref{eq:FOM} is linear, the non-Markovian operators as defined in \eqref{eq:CloseOpStateTrue} and \eqref{eq:CloseOpOutTrue} can be recovered exactly under appropriate conditions on $\Nreproj$ and $\hbfD$. This is shown by the following proposition. 

\begin{proposition} \label{prop:StageLinear}
Let the full model in \eqref{eq:FOM} be linear. If $\Treproj \ge L+2$, $\Nreproj \ge n+p$  and the data matrix $\hbfD$ in \eqref{eq:StageDataMatrix} has full rank, then for $l = 1,\dots,L$, the least squares problems \eqref{eq:StageLSState} and  \eqref{eq:StageLSOut} are uniquely solved at $(\hbfE_{l},\hbfF_{l}) = (\bfE_{l},\bfF_{l})$ and $(\hbfG_{l},\hbfH_{l}) = (\bfG_{l},\bfH_{l})$ defined in \eqref{eq:CloseOpStateTrue} and \eqref{eq:CloseOpOutTrue}, respectively, with objective value 0. 
\end{proposition}

\begin{proof}
The equations \eqref{eq:ResSoln} and \eqref{eq:OutputResolved} describe the dynamics of $\tbfz_k = \bfV^T \bfz_k$ and $\bfy_k$. Therefore, following the extended re-projection procedure in Algorithm \ref{alg:ClosureReproj}, for fixed $i=1,\dots,\Nreproj$, the reduced state $\rbfz_k^{(i)}$, the input $\bfu_k^{(i)}$, and the output  $\bfy_k^{(i)}$ satisfy
    \begin{align*}
        \rbfz^{(i)}_{k+1} & = \tbfA_1 \rbfz^{(i)}_k + \tbfB \bfu_k^{(i)} +  \sum_{l=0}^{k-1} (\closeState_{k-l} \rbfz^{(i)}_l + \closeSignal_{k-l} \bfu^{(i)}_l)\,, \\
        \bfy^{(i)}_k & = \tbfC \rbfz^{(i)}_k + \sum_{l=0}^{k-1} (\closeOState_{k-l} \rbfz^{(i)}_l + \closeOSignal_{k-l} \bfu^{(i)}_l)\,,
    \end{align*}
for $k=1,\dots,\Treproj-2$, where $(\bfE_{l},\bfF_{l})$ and $(\bfG_{l},\bfH_{l})$ are defined in \eqref{eq:CloseOpStateTrue} and \eqref{eq:CloseOpOutTrue}. This implies that \eqref{eq:StageLSState} and \eqref{eq:StageLSOut} have objective value 0 when $(\hbfE_{l},\hbfF_{l}) = (\bfE_{l},\bfF_{l})$ and $(\hbfG_{l},\hbfH_{l}) = (\bfG_{l},\bfH_{l})$. The solutions to these least squares problems are unique since the problems \eqref{eq:StageLSState} and \eqref{eq:StageLSStateMatrix} and the problems \eqref{eq:StageLSOut} and \eqref{eq:StageLSOutMatrix} each have the same solution if $\hbfD$ is full rank.
\end{proof}
    
\begin{algorithm}[t]
\caption{Stagewise operator inference for learning non-Markovian correction terms}
\begin{algorithmic}[1]
  \STATE Formulate the data matrix $\hbfD$ in \eqref{eq:StageDataMatrix}
  \FOR{$l=1,\dots,L$}
    \STATE Compute the components of $\StageDiscrepState(l,1)$ and $\StageDiscrepOut(l,1)$ defined in \eqref{eq:StageDiscrepComp}
    \STATE Solve for $\hbfO_l^{\bfz}$ and $\hbfO_l^{\bfy}$ in \eqref{eq:StageLSStateMatrix} and \eqref{eq:StageLSOutMatrix}
  \ENDFOR 
  \STATE Return $(\hbfE_{l},\hbfF_{l})$ and $(\hbfG_{l},\hbfH_{l})$ for $l = 1,\dots,L$
\end{algorithmic}
\label{alg:Stagewise}
\end{algorithm}

\subsubsection{Batch operator inference of non-Markovian operators} \label{subsubsec:batch}

We now propose batch operator inference that infers all non-Markovian operators  simultaneously instead of proceeding in a sequential manner as the stagewise approach. Batch operator inference learns the operators $(\bfE_{l},\bfF_{l})$ and $(\bfG_{l},\bfH_{l})$ by solving for $(\bbfE_{l},\bbfF_{l})$ and $(\bbfG_{l},\bbfH_{l})$ in the least squares problems given by
\begin{align} \label{eq:BatchLSState}
        \min_{\substack{\bbfE_{l} \in \R^{n \times n},  \bbfF_{l} \in \R^{n \times p},\\ l = 1,\dots,L}} \sum_{k=1}^{\Treproj-2} \sum_{i=1}^{\Nreproj} \left \Vert \StageDiscrepState^{(i)}(k,0) -  \sum_{l=k-L}^{k-1} \left(\bbfE_{k-l} \rbfz^{(i)}_l + \bbfF_{k-l} \bfu^{(i)}_l\right)  \right \Vert_2^2
\end{align}
and
\begin{align} \label{eq:BatchLSOut}
        \min_{\substack{\bbfG_{l} \in \R^{s \times n},  \bbfH_{l} \in \R^{s \times p},\\ l = 1,\dots,L}} \sum_{k=1}^{\Treproj-2} \sum_{i=1}^{\Nreproj} \left \Vert \StageDiscrepOut^{(i)}(k,0) -  \sum_{l=k-L}^{k-1} \left(\bbfG_{k-l} \rbfz^{(i)}_l + \bbfH_{k-l} \bfu^{(i)}_l\right)  \right \Vert_2^2\,,
\end{align}
where $\StageDiscrepState^{(i)}(k,0)$ and $\StageDiscrepOut^{(i)}(k,0)$ defined in \eqref{eq:StageDiscrepComp} represent the discrepancy between the current reduced state and output with the corresponding Markovian model.

To write the least-squares problems \eqref{eq:BatchLSState} and \eqref{eq:BatchLSOut} in matrix form, set 
$\StageDiscrepState(k,0) = [\StageDiscrepState^{(1)}(k,0),\dots,\StageDiscrepState^{(\Nreproj)}(k,0)] \in \R^{n \times \Nreproj}$ and $\StageDiscrepOut(k,0) = [\StageDiscrepOut^{(1)}(k,0),\dots,\StageDiscrepOut^{(\Nreproj)}(k,0)] \in \R^{s \times \Nreproj}$ for $k=1,\dots,\Treproj-2$. The right hand side matrix is then given by 
\begin{align} \label{eq:BatchRHS}
        \StageDiscrepState & = [\StageDiscrepState(1,0), \dots, \StageDiscrepState(\Treproj-2,0)] \in \R^{n \times \Nreproj (\Treproj-2)}\,, \\
        \StageDiscrepOut & = [\StageDiscrepOut(1,0), \dots, \StageDiscrepOut(\Treproj-2,0)] \in \R^{s \times \Nreproj (\Treproj-2)}. \notag
\end{align}
The matrices of unknowns are 
    \begin{align} \label{eq:BatchMatUnknowns}
        \bbfO^{\bfz} &=[\bbfE_{1} \,\,\, \bbfF_{1} \,\,\, \dots \,\,\, \bbfE_{L} \,\,\, \bbfF_{L}] \in \R^{n \times (n+p)L}\,, \\
        \bbfO^{\bfy} & = [\bbfG_{1} \,\,\, \bbfH_{1} \,\,\, \dots \,\,\, \bbfG_{L} \,\,\, \bbfH_{L}] \in \R^{s \times (n+p)L}. \notag
    \end{align}
    Finally, for $k=1,\dots,\Treproj-2$, define $\bbfD(k) \in \R^{(n+p)L \times \Nreproj}$ as
    \begin{align*}  
\bbfD^T(k) =  
\begin{bmatrix} 
(\rbfz_{k-1}^{(1)})^T & (\bfu_{k-1}^{(1)})^T & \dots & (\rbfz_{k-L}^{(1)})^T & (\bfu_{k-L}^{(1)})^T \\
\vdots & \vdots & \ddots & \vdots & \vdots \\
(\rbfz_{k-1}^{(\Nreproj)})^T & (\bfu_{k-1}^{(\Nreproj)})^T & \dots & (\rbfz_{k-L}^{(\Nreproj)})^T & (\bfu_{k-L}^{(\Nreproj)})^T
\end{bmatrix}
\end{align*}
and set the data matrix to
\begin{align} \label{eq:BatchDataMat}
        \bbfD = [\bbfD(1) \,\,\, \dots \,\,\, \bbfD(\Treproj-2)] \in \R^{(n+p)L \times (\Treproj-2)\Nreproj}.
\end{align}
The problems \eqref{eq:BatchLSState} and \eqref{eq:BatchLSOut} are then, in matrix form,
\begin{align} \label{eq:BatchLSStateMat}
    \min_{\bbfO^{\bfz} \in \R^{n \times (n+p)L}} \| \bbfD^T (\bbfO^{\bfz})^T - \StageDiscrepState^T\|_F^2
    \end{align}
and
\begin{align} \label{eq:BatchLSOutMat}
    \min_{\bbfO^{\bfy} \in \R^{n \times (n+p)L}} \| \bbfD^T (\bbfO^{\bfy})^T - \StageDiscrepOut^T\|_F^2,
    \end{align}
respectively.

The batch operator inference procedure is summarized in Algorithm \ref{alg:Batch}. Like the stagewise approach, batch operator inference recovers the non-Markovian operators if the full model \eqref{eq:FOM} is linear.

\begin{proposition} \label{prop:BatchLinear}
If \eqref{eq:FOM} is linear, $\Treproj = L+2$, $(\Treproj-2)\Nreproj \ge (n+p)L$ and the data matrix $\bbfD$ in \eqref{eq:BatchDataMat} is full rank, the least squares problems \eqref{eq:BatchLSState} and \eqref{eq:BatchLSOut} are uniquely minimized at $(\bbfE_{l},\bbfF_{l}) = (\bfE_{l},\bfF_{l})$ and $(\bbfG_{l},\bbfH_{l}) = (\bfG_{l},\bfH_{l})$ for $l=1,\dots,L$ defined in \eqref{eq:CloseOpStateTrue} and \eqref{eq:CloseOpOutTrue}, respectively, with objective value 0.
\end{proposition}
    
\begin{proof}
The proof is analogous to the proof of  Proposition~\ref{prop:StageLinear}.
\end{proof}

\begin{algorithm}[t]
\caption{Batch operator inference for learning non-Markovian correction terms}
\begin{algorithmic}[1]
  \FOR{$k = 1,\dots,\Treproj-2$}
    \STATE Compute $\StageDiscrepState(k,0), \StageDiscrepOut(k,0), \bbfD(k)$
  \ENDFOR
  \STATE Formulate $\StageDiscrepState$ and $\StageDiscrepOut$ in \eqref{eq:BatchRHS}
  \STATE Formulate the data matrix $\bbfD$ in \eqref{eq:BatchDataMat}
  \STATE Solve for $\bbfO^{\bfz}$ and $\bbfO^{\bfy}$ in \eqref{eq:BatchLSStateMat} and \eqref{eq:BatchLSOutMat}
  \STATE Return $(\bbfE_{l},\bbfF_{l})$ and $(\bbfG_{l},\bbfH_{l},)$ for $l = 1,\dots,L$
\end{algorithmic}
\label{alg:Batch}
\end{algorithm}

\section{Numerical experiments} \label{sec:NumExp}

We conduct numerical experiments with data from linear full models based on the  convection-diffusion equation in Section \ref{subsec:Convdiff} and nonlinear models corresponding to diffusion-reaction processes in Section \ref{subsec:react2d} and the Chafee-Infante equation in Section \ref{subsec:Chafee}. The number of the observed state components ranges from 20, 40, 60 to 80 percent of the total number of state components $N$, the dimension of the full system \eqref{eq:FOM}. The observed components correspond to equidistant points in the spatial domain for problems with one-dimensional spatial domains. In problems with two-dimensional spatial domains, the observed state components are selected equidistantly with respect to the indexing of the grid points. The Markovian operators are learned according to the methodology in \cite{paper:PeherstorferW2016,paper:Peherstorfer2019} which guarantees recovery of the operators in intrusive model reduction from data in our examples up to numerical errors. 

\subsection{Convection-diffusion equation} \label{subsec:Convdiff}

Consider a linear parabolic PDE described by the convection-diffusion equation on the time domain $t \in (0,0.5)$ and spatial domain $\Omega = (0,1) \times (0,0.25)$  with boundary $\partial \Omega$ given by
\begin{align} \label{eq:ConvDiffPDE}
\hspace{-1in}   \frac{\partial}{\partial t} x(\xi_1,\xi_2,t) &= \nabla \cdot (\nabla x(\xi_1,\xi_2,t)) - (1,1) \cdot \nabla x(\xi_1,\xi_2,t), \\
\hspace{-1in}   x(\xi_1,\xi_2,t) & = 0 \text{\,\, for \,\,} (\xi_1,\xi_2) \in \partial \Omega \backslash \cup_{i=1}^5 E_i, \notag \\ 
\hspace{-1in}  \nabla x(\xi_1,\xi_2,t) \cdot \mathbf{n} & = u_i(t) \text{\,\, for \,\,} (\xi_1,\xi_2) \in E_i, i = 1,\dots, 5, \notag \\
\hspace{-1in}  x(\xi_1,\xi_2,0) & = 0 \notag
\end{align}
for $(\xi_1,\xi_2,t) \in \Omega \times (0,0.5).$ The pieces of the boundary $\partial \Omega$ with prescribed Neumann conditions are $E_1 = \{(\xi_1,\xi_2):\xi_1 \in (0.8,1), \xi_2 = 0\} \cup \{(\xi_1,\xi_2): \xi_1 = 1, \xi_2 \in (0,0.25)\}$, $E_2 = \{(\xi_1,\xi_2): \xi_1 \in (0.6,0.8), \xi_2 = 0.25\}$, $E_3 = \{(\xi_1,\xi_2): \xi_1 \in (0.2,0.4), \xi_2 = 0.25\}$ , $E_4 = \{(\xi_1,\xi_2): \xi_1 \in (0,0.2), \xi_2 = 0\}$, and $E_5 = \{(\xi_1,\xi_2): \xi_1 \in (0.4,0.6), \xi_2 = 0\}$. The Neumann boundary condition at each of these edges is driven by an input $u_i(t)$ for $i=1,\dots,5$. We set the output to be the integral of the PDE solution along $E_5$, i.e. $y(t) = \int_{E_5} x(\xi_1,\xi_2,t) \, \partial \Omega$.

The PDE \eqref{eq:ConvDiffPDE} is temporally discretized using finite difference with time step size $\delta t = 10^{-5}$ and spatially discretized using square finite elements with width $\Delta \xi_1 = \Delta \xi_2 = 1/75$ and $N = 1121$ linear hat basis functions. This yields the high-dimensional system
\begin{align} \label{eq:ConvDiffFOMMatrix}
    \bfx_{k+1} &= \bfA_1 \bfx_k + \bfB \bfu_k, \qquad k = 0,\dots, K-1\,, \\
    \bfy_{k+1} &= \bfC \bfx_{k+1}. \notag
\end{align}
The basis $\bfV$ of dimension $\nr$ is obtained from snapshots of the observation trajectory (partially observed states) using the input  $\bfu^{\text{basis}}(t) = [\sin(2t), \sin(4t), \dots, \sin(10t)]^T \in \mathbb{R}^5$  at time $t$ and $K = 50001$ time steps. Together with the basis $\bfV$, the extended re-projection algorithm (Algorithm~\ref{alg:ClosureReproj}) is applied with $\Nreproj$ re-projection steps and $\Treproj$ time steps per re-projection step, where inputs to the full model  \eqref{eq:ConvDiffFOMMatrix} at each time step are realizations of a $p$-dimensional random vector with independent components that are uniformly distributed in  $[0,2]$. The test input trajectory is given by $\bfu^{\text{test}}(t) = [ e^t\sin(1.75t), e^t\sin(3.5t), \dots, e^t\sin(8.75t)]^T \in \mathbb{R}^5$. 

\subsubsection{Recovering non-Markovian operators}

Figure~\ref{fig:ConvDiff_StageOpNorm_dim10} plots the norms  $\|\bfE_{l}\|_2$ (left panel) and $\|\bfF_{l}\|_2$ (right panel) of the non-Markovian operators \eqref{eq:CloseOpStateTrue} obtained with intrusive model reduction for lag values $l=1,\dots,100$. The four curves in each panel correspond to 20\%, 40\%, 60\%, and 80\% observed state components, respectively. The norms of the non-Markovian operators decay with increasing lag $l$, which supports truncating the non-Markovian term. 

Set now the lag to $L = 100$ with  60\% of all state components observed. Consider stagewise inference (Algorithm~\ref{alg:Stagewise}) of the non-Markovian operators applied to trajectories sampled with the extended re-projection algorithm with $\Nreproj = 25$ re-projection steps $\Treproj = 102$ time steps per re-projection step. This means that Proposition~\ref{prop:StageLinear} applies because the full model \eqref{eq:ConvDiffFOMMatrix} is linear in the state variable and we have numerically ensured that the data matrix is full rank. Thus, stagewise inference recovers the very same non-Markovian operators \eqref{eq:CloseOpStateTrue} that are obtained with intrusive model reduction. Figure~\ref{fig:ConvDiff_OpNorm_StageVSBatchDetailed} compares the norms of the operators $\bfE_{l},\bfF_{l}$ from intrusive model reduction \eqref{eq:CloseOpStateTrue} with the norms of the inferred operators $\hbfE_{l},\hbfF_{l}$ from stagewise inference. Notice that the norms coincide which is in agreement with Proposition~\ref{prop:StageLinear}.

Consider now batch  inference with trajectories sampled with the extended re-projection algorithms with $\Nreproj = 12$ and $\Treproj \in \{250, 1000, 10000\}$ which correspond to larger training data sets than what is used with $\Nreproj = 25$ and $\Treproj = 102$ for stagewise inference. The data sets generated for increasing values of $\Treproj$ are nested. Since $\Treproj > L + 2$, Proposition~\ref{prop:BatchLinear} does not apply and therefore we do not expect that the operators inferred with batch  inference  coincide with the operators from intrusive model reduction in this setting. This is indicated by the results shown in Figure~\ref{fig:ConvDiff_OpNorm_StageVSBatchDetailed}. Note, however, that the operators obtained by batch inference have norms that closely approximate the norm of the intrusive non-Markovian operators for $l \leq 10$.

\begin{figure}
  \begin{subfigure}[b]{0.45\textwidth}
    \begin{center}
{{\Large\resizebox{1.15\columnwidth}{!}{\input{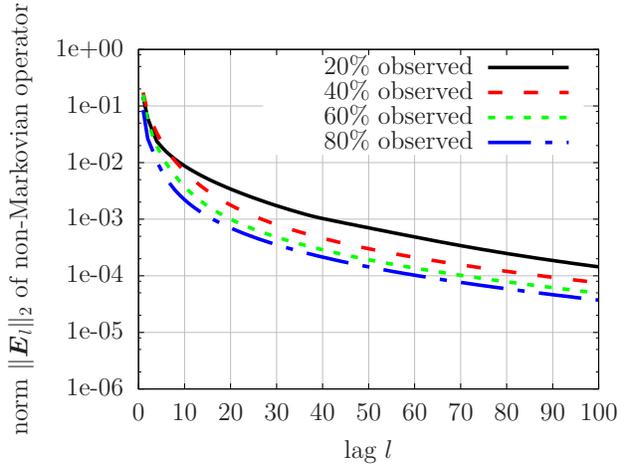}}}}
\end{center}
    \caption{non-Markovian state operator $\bfE_{l}$, dimension $n=10$.}
  \end{subfigure}
  \quad \quad \quad
  \begin{subfigure}[b]{0.45\textwidth}
    \begin{center}
{{\Large\resizebox{1.15\columnwidth}{!}{\input{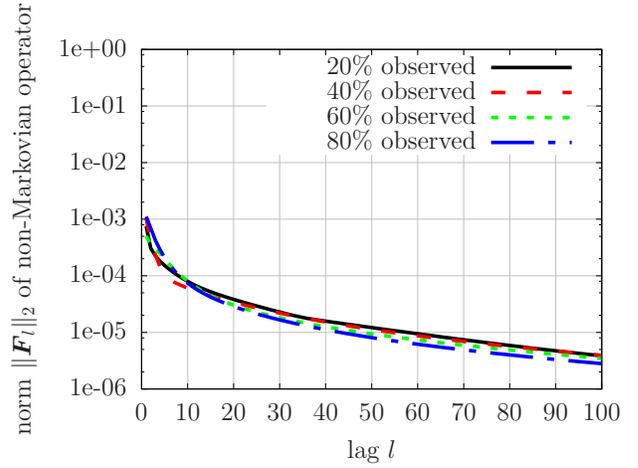}}}}
\end{center}
    \caption{non-Markovian input operator $\bfF_{l}$, dimension $n=10$.}
  \end{subfigure}
  \caption{Convection-diffusion equation (Section~\ref{subsec:Convdiff}). The norm of the non-Markovian operators decays with the lag $l$ and thus supports truncating the non-Markovian term.}
  \label{fig:ConvDiff_StageOpNorm_dim10}
\end{figure}

\begin{figure}
  \begin{subfigure}[b]{0.45\textwidth}
    \begin{center}
{{\Large\resizebox{1.15\columnwidth}{!}{\input{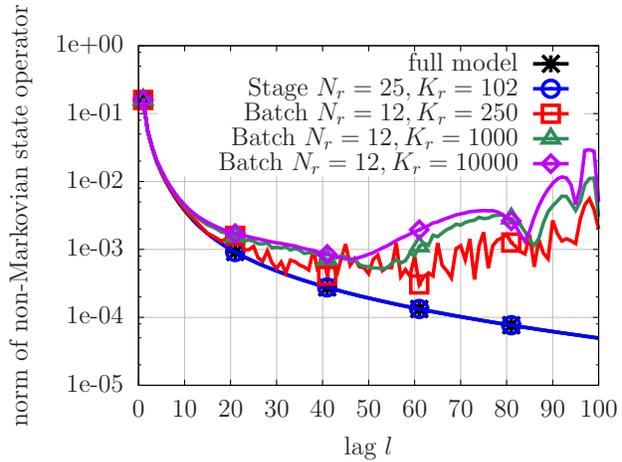}}}}
\end{center}
    \caption{Stage vs batch for $\bfE_{l}$, $n=10$, 60\% observed.}
  \end{subfigure}
  \quad \quad \quad
  \begin{subfigure}[b]{0.45\textwidth}
    \begin{center}
{{\Large\resizebox{1.15\columnwidth}{!}{\input{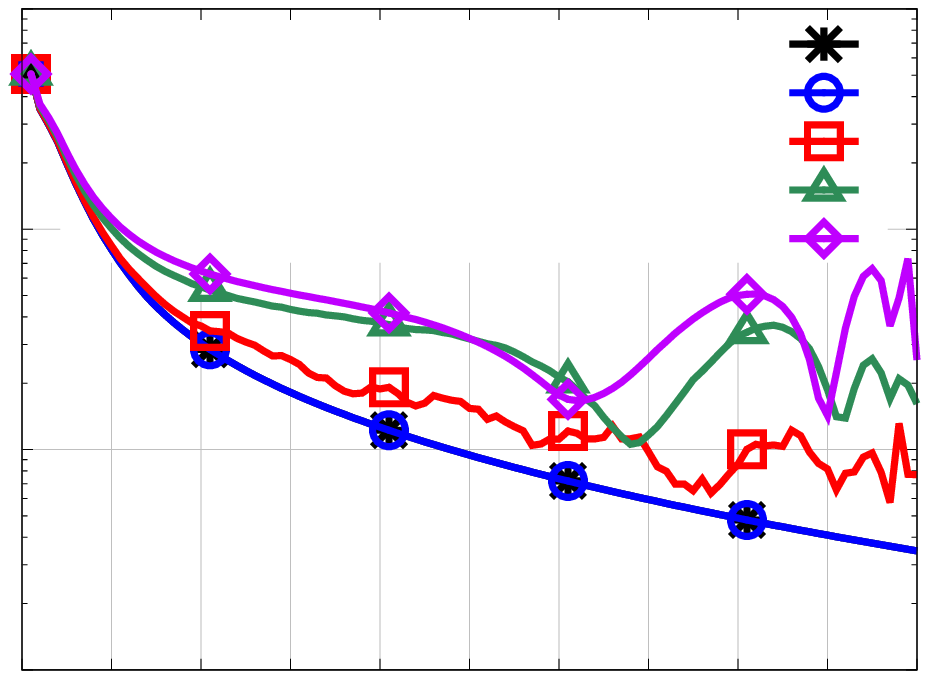}}}}
\end{center}
    \caption{Stage vs batch for $\bfF_{l}$, $n=10$, 60\% observed.}
  \end{subfigure}
  \caption{Convection-diffusion equation (Section~\ref{subsec:Convdiff}). Stagewise operator inference recovers the non-Markovian operators from intrusive model reduction (cf.~Proposition~\ref{prop:StageLinear}) in this example. The batch approach only gives approximations of the non-Markovian operators of intrusive model reduction because $\Treproj > L+2$  (cf.~Proposition~\ref{prop:BatchLinear}).}
  \label{fig:ConvDiff_OpNorm_StageVSBatchDetailed}
\end{figure}

\subsubsection{State error with number of time steps at most the lag of non-Markovian term}

Let $\bfZ_{\text{test}} = [\bfz_0^{\text{test}},\dots,\bfz_{K-1}^{\text{test}}] \in \mathbb{R}^{r \times K}$ be the trajectory of observations under the test input $\bfu^{\text{test}}(t)$ generated with the full model for  $K =  102$ time steps. Note that the number of rows of $\bfZ_{\text{test}}$ depends on the number of observed state components. Figure~\ref{fig:ConvDiff_ProjErr_Detailed} plots the relative projection error 
\begin{align}
\label{eq:ProjError}
       \errStateProj = \frac{\|\bfZ_{\text{test}} - \bfV \bfV^T \bfZ_{\text{test}}\|_F}{ \|\bfZ_{\text{test}}\|_F}
\end{align}
for trajectories $\bfZ_{\text{test}}$ with 20\% and 80\% observed state components with reduced dimensions $n = 4$ and $n = 10$. Additionally, Figure~\ref{fig:ConvDiff_ProjErr_Detailed} plots the relative error of the observations
\begin{equation}
	\errStateStage = \frac{\|\bfZ_{\text{test}} - \bfV \tbfZ_{\text{test}}^{\text{stage}}\|_F}{\|\bfZ_{\text{test}}\|_F}\,
	\label{eq:RelStateStageError}
\end{equation}
where $\tbfZ_{\text{test}}^{\text{stage}} \in \R^{n \times K}$ is the trajectory of observations computed with the learned reduced model with non-Markovian operators obtained with stagewise inference and test input $\bfu^{\text{test}}(t)$. Note that the observation error $\errStateStage$ depends on the number of observed state components, the dimension $n$ of the reduced space, and the lag $L$. Stagewise inference recovers the non-Markovian operators from intusive model reduction in this example (cf.~Proposition~\ref{prop:StageLinear}) and thus the relative observation error \eqref{eq:RelStateStageError} of the learned reduced model equals the projection error if the lag $L$ of the non-Markovian term satisfies $L \geq K - 2$. This corresponds to model \eqref{eq:ResSoln} without truncation. 

We now set $L = 100$. Figure \ref{fig:ConvDiff_ProjErr_Condensed} shows the absolute difference  $|\errStateStage - \errStateProj|$ between the relative stagewise observation error \eqref{eq:RelStateStageError} and the projection error \eqref{eq:ProjError} after $K = L + 2$ time steps for 20\%, 40\%, 60\%, and 80\% observed components. Define the analog of \eqref{eq:RelStateStageError} for batch operator inference as
\begin{equation}
	\errStateBatch = \frac{\|\bfZ_{\text{test}} - \bfV \tbfZ_{\text{test}}^{\text{batch}}\|_F}{\|\bfZ_{\text{test}}\|_F}\,
	\label{eq:RelStateBatchError}
\end{equation}
where $\tbfZ_{\text{test}}^{\text{batch}} \in \mathbb{R}^{n \times K}$ is the trajectory of observations computed with the reduced model whose non-Markovian operators are learned simultaneously. 
Also shown in Figure~\ref{fig:ConvDiff_ProjErr_Condensed} is the absolute difference  $|\errStateBatch - \errStateProj|$ between the relative batch observation error \eqref{eq:RelStateBatchError} and the projection error \eqref{eq:ProjError}. The results demonstrate that the reduced model with non-Markovian terms obtained with stagewise operator inference in this example achieves a state error \eqref{eq:RelStateStageError} that equals, up to numerical errors, the projection error \eqref{eq:ProjError}. Meanwhile, batch inference in this situation does not recover the  non-Markovian operators of intrusive model reduction and thus the corresponding difference between observation \eqref{eq:RelStateBatchError} and projection error \eqref{eq:ProjError} is higher than with stagewise inference in this example.

\begin{figure}
  \begin{subfigure}[b]{0.45\textwidth}
    \begin{center}
{{\Large\resizebox{1.15\columnwidth}{!}{\input{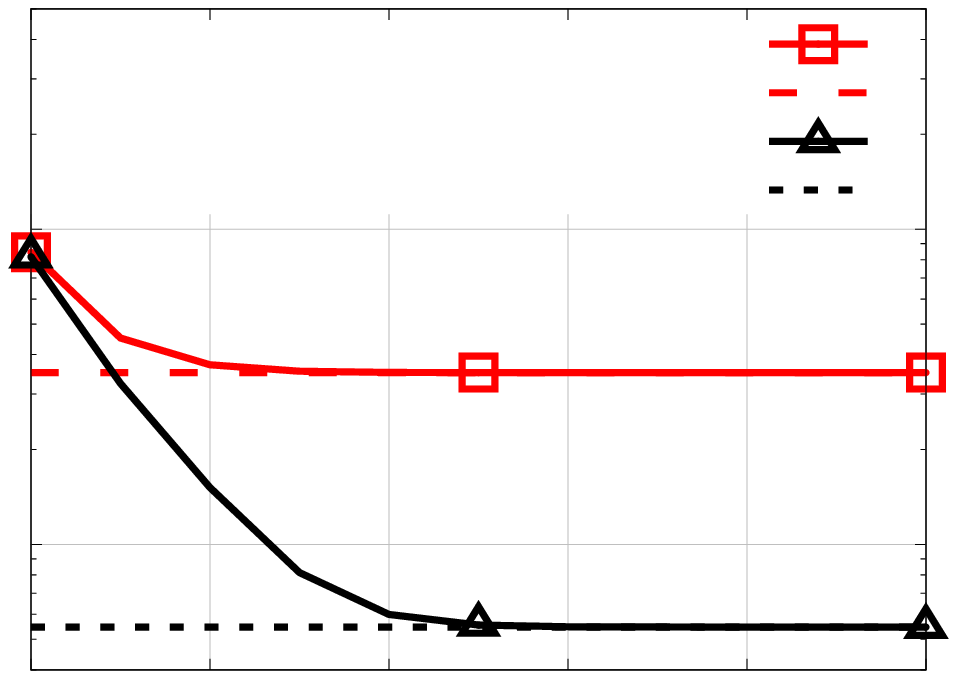}}}}
\end{center}
    \caption{Stagewise OpInf error, 20$\%$ observed components}
  \end{subfigure}
  \quad \quad \quad
  \begin{subfigure}[b]{0.45\textwidth}
    \begin{center}
{{\Large\resizebox{1.15\columnwidth}{!}{\input{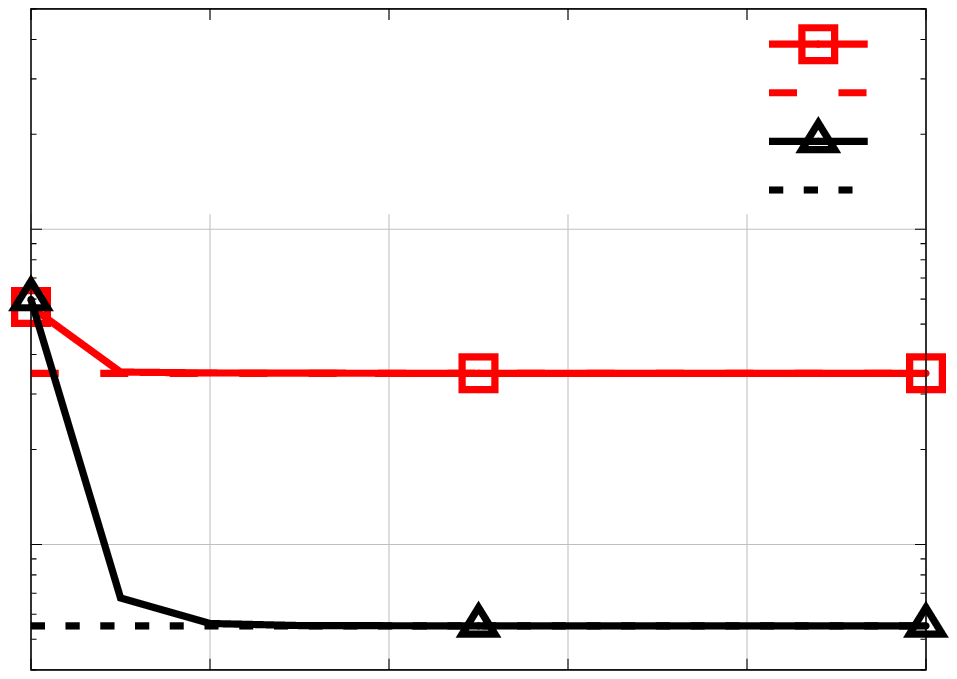}}}}
\end{center}
    \caption{Stagewise OpInf error, 80$\%$ observed components}
  \end{subfigure}
  \caption{Convection-diffusion equation (Section~\ref{subsec:Convdiff}). The projection error $\errStateProj$ \eqref{eq:ProjError} and the stagewise observation error $\errStateStage$ \eqref{eq:RelStateStageError} are equal at the maximal lag $L=100$ since $L = K-2$.
   }
  \label{fig:ConvDiff_ProjErr_Detailed}
\end{figure}

\begin{figure}
  \begin{subfigure}[b]{0.45\textwidth}
    \begin{center}
{{\Large\resizebox{1.15\columnwidth}{!}{\input{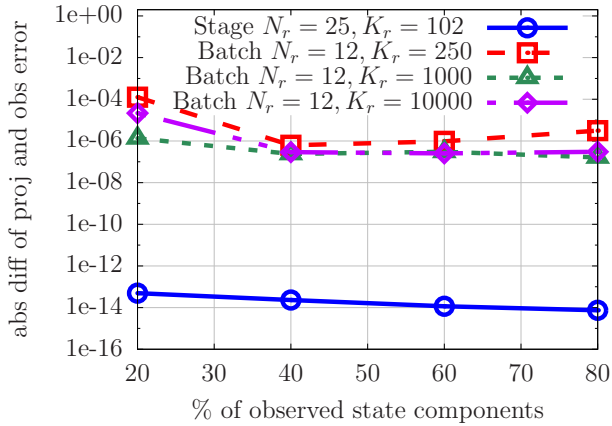}}}}
\end{center}
    \caption{Stage vs batch error, $n=8$.}
  \end{subfigure}
  \quad \quad \quad
  \begin{subfigure}[b]{0.45\textwidth}
    \begin{center}
{{\Large\resizebox{1.15\columnwidth}{!}{\input{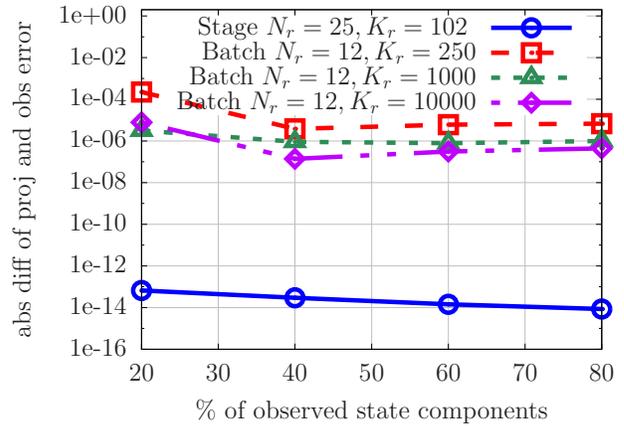}}}}
\end{center}
    \caption{Stage vs batch error, $n=10$.}
  \end{subfigure}
  \caption{Convection-diffusion equation (Section~\ref{subsec:Convdiff}). If the lag of the non-Markovian term satisfies $L=K-2$, the observation error corresponding to stagewise operator inference coincides with the projection error.  Since the batch inference does not recover the non-Markovian operators of intrusive model reduction in this example, the batch observation error differs from the projection error.}
  \label{fig:ConvDiff_ProjErr_Condensed}
\end{figure}

\subsubsection{State error with number of time steps larger than lag of non-Markovian term} \label{subsubsec:ConvDiffMoreStepsThanLag}

We now consider $K = 50001$ time steps and a lag $L \ll K$ so that the observation error of a reduced model obtained with stagewise  inference does not necessarily coincide with the projection error \eqref{eq:ProjError}. Figure \ref{fig:ConvDiff_StateErr_Detailed} shows the observation errors \eqref{eq:RelStateStageError}  and \eqref{eq:RelStateBatchError} for stagewise and batch inference for 20\% observed states with dimension $n = 10$ (left panel) and 80\% observed states with dimension $n = 4$ (right panel). Batch inference with $\Nreproj = 12$ and $\Treproj = 250$ leads to a model with a large observation error $\errStateBatch$ near $L = 20$ for 20\% observed states and dimension $n = 10$; this is consistent with the graphs of Figure \ref{fig:ConvDiff_OpNorm_StageVSBatchDetailed} that report that the norm of the learned operators can differ significantly from the norm of the intrusive operators. However, if the number of time steps $\Treproj$ sampled from the full model with re-projection is increased to $\Treproj = 1000$ and $\Treproj = 10000$, the non-Markovian models obtained with batch inference outperform stagewise inference in this example. While stagewise inference is confined to recovering the non-Markovian operators of intrusive model reduction, batch inference offers more flexibility that can lead to a non-Markovian model that achieves lower errors than the stagewise approach. The right panel of Figure~\ref{fig:ConvDiff_StateErr_Detailed} shows results for 80\% observed state components and dimension $n=4$ where batch inference leads to a non-Markovian model whose observation error does not improve as the lag $L$ is increased. In contrast, stagewise operator inference  yields a non-Markovian model that eventually attains a lower error with increasing lag $L$.

We now fix the lag to $L = 100$ and consider the relative observation errors shown in Figure~\ref{fig:ConvDiff_StateErr_Condensed} over the number of observed components. Here, batch inference was also applied to trajectories sampled with the extended re-projection algorithm with $\Nreproj = 25$ and $\Treproj = \{500,1000,10000\}$ which represent nested training data sets. The results demonstrate that batch inference of the operators seems to outperform stagewise learning if training is done on a sufficiently large data set. In terms of computational cost, stagewise uses $\Nreproj \times \Treproj = 25 \times 102$ data points whereas for batch  inference, up to $\Nreproj \times \Treproj = 25 \times 10000$ data points are used.

\begin{figure}
  \begin{subfigure}[b]{0.45\textwidth}
    \begin{center}
{{\Large\resizebox{1.15\columnwidth}{!}{\input{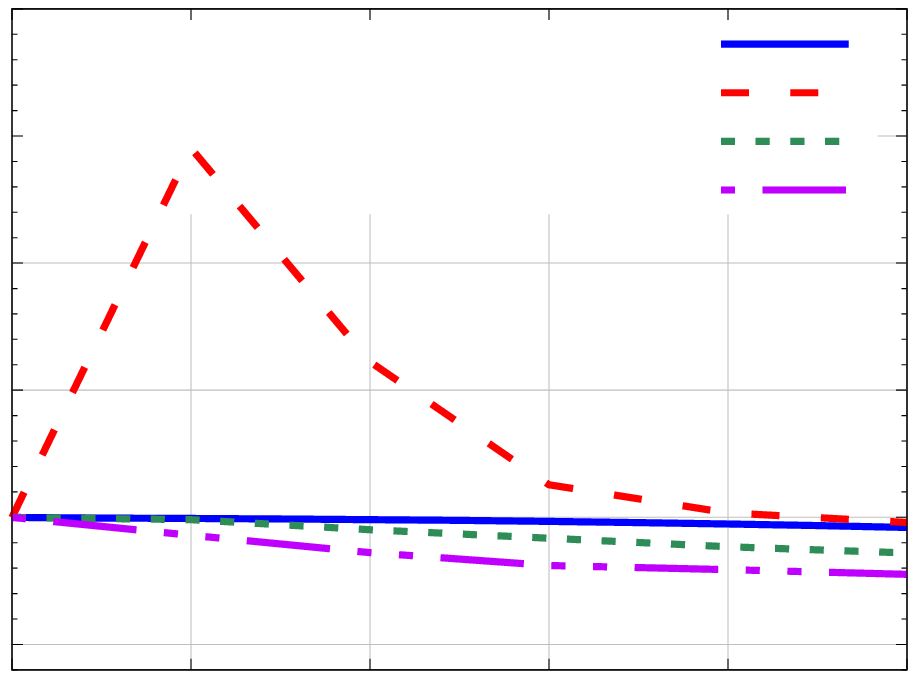}}}}
\end{center}
    \caption{$n = 10, 20\%$ observed components}
  \end{subfigure}
  \quad \quad \quad
  \begin{subfigure}[b]{0.45\textwidth}
    \begin{center}
{{\Large\resizebox{1.15\columnwidth}{!}{\input{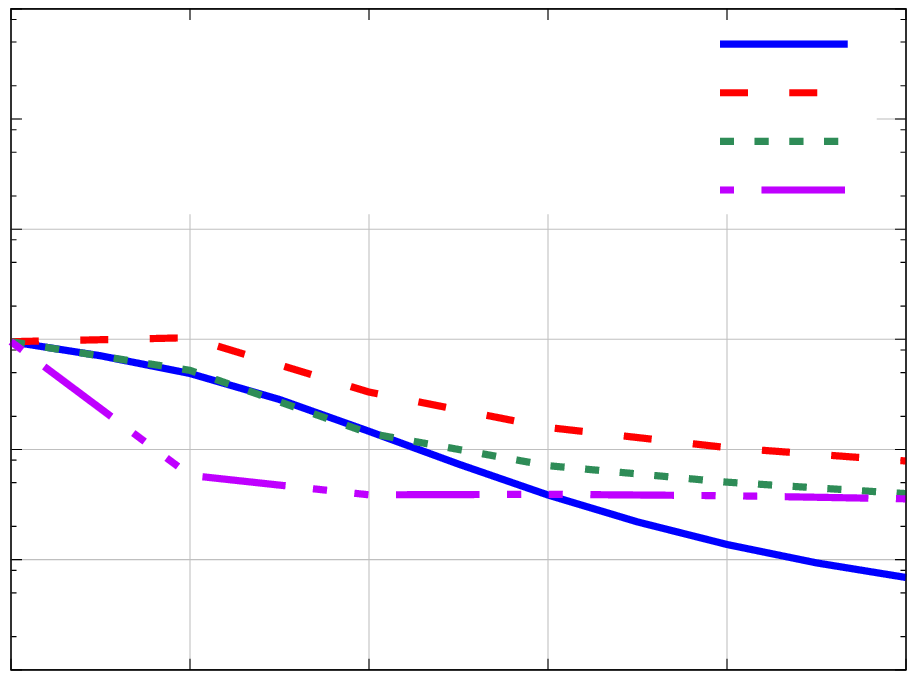}}}}
\end{center}
    \caption{$n = 4, 80\%$ observed components}
  \end{subfigure}
  \caption{Convection-diffusion equation (Section~\ref{subsec:Convdiff}). The non-Markovian model trained via batch inference may produce unstable dynamics (left) or result in a model whose observation error stagnates despite an increase in the lag $L$ (right).}
  \label{fig:ConvDiff_StateErr_Detailed}
\end{figure}

\begin{figure}
  \begin{subfigure}[b]{0.45\textwidth}
    \begin{center}
{{\Large\resizebox{1.15\columnwidth}{!}{\input{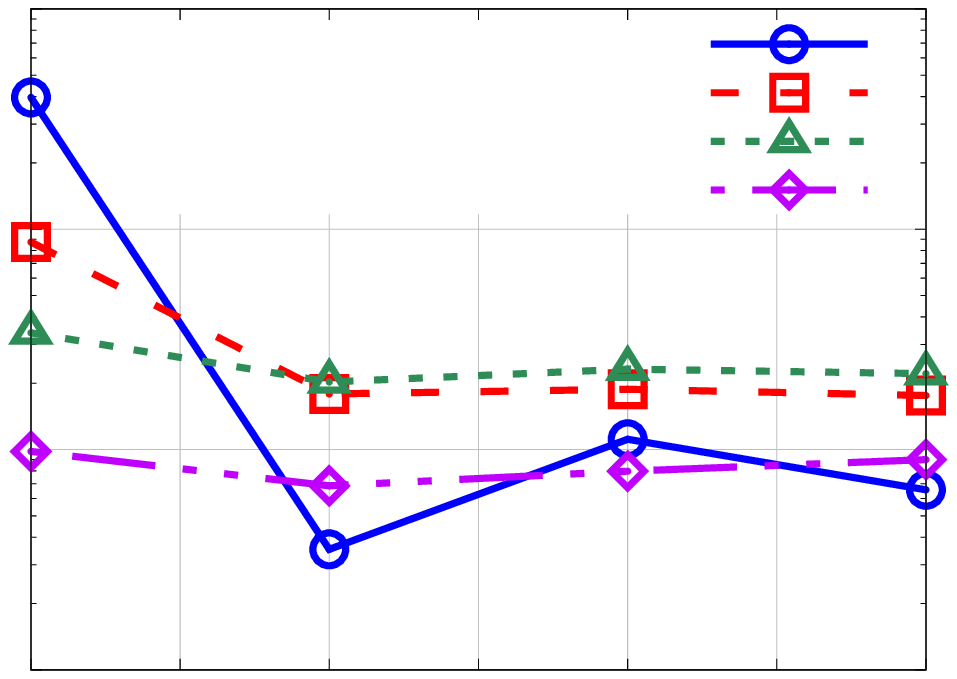}}}}
\end{center}
    \caption{Batch OpInf trained with $\Nreproj = 25, n = 10$}
  \end{subfigure}
  \quad \quad
  \begin{subfigure}[b]{0.45\textwidth}
    \begin{center}
{{\Large\resizebox{1.15\columnwidth}{!}{\input{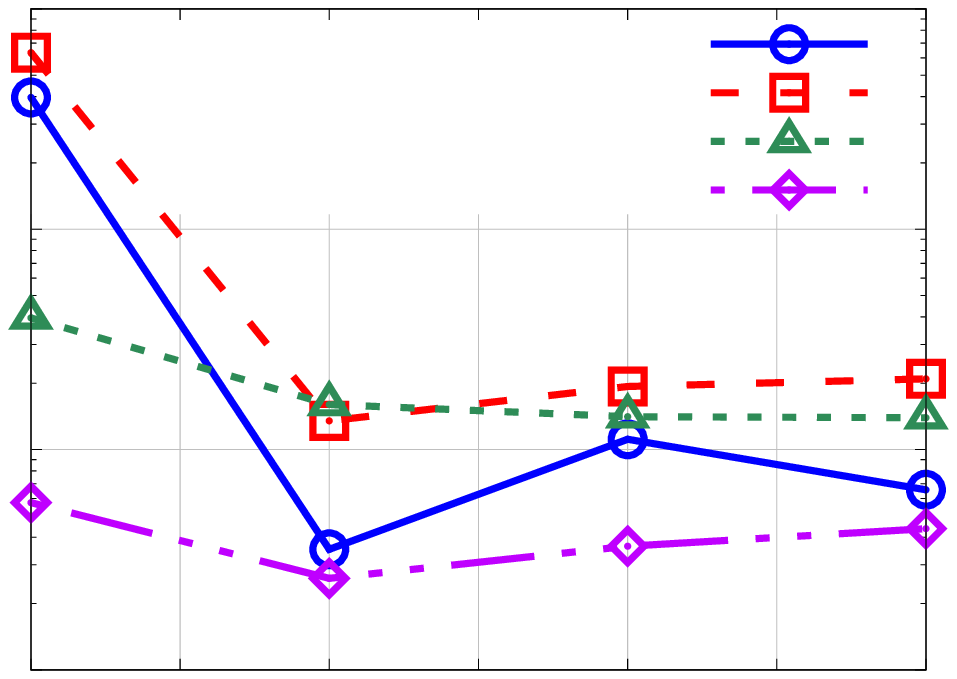}}}}
\end{center}
    \caption{Batch OpInf trained with $\Nreproj = 12, n = 10$}
  \end{subfigure}
  \caption{Convection-diffusion equation (Section~\ref{subsec:Convdiff}). Stagewise inference achieves a smaller observation error than batch inference in this example except when the latter is trained on data sets with sufficiently large $\Treproj$.}
  \label{fig:ConvDiff_StateErr_Condensed}
\end{figure}

\subsubsection{Output error}

Consider now the output trajectory $\bfY_{\text{test}}$ of the full model with the test input trajectory and let $\tbfY_{\text{test}}^{\text{Markovian}}$ be the output trajectory obtained with a Markovian reduced model learned via non-intrusive model reduction from partially observed state trajectories, i.e. the lag of the non-Markovian term is $L=0$. Denote by $\tbfY_{\text{test}}^{\text{stage}}$ the output trajectory computed with the non-Markovian reduced model learned via stagewise  inference. Note that $\tbfY_{\text{test}}^{\text{stage}}$ depends on the lag $L$, the dimension $n$ of the reduced space, and the number of observed state components. Figure~\ref{fig:ConvDiff_Output_Detailed} shows the output trajectory of the full model $\bfY_{\text{test}}$, the trajectory $\tbfY_{\text{test}}^{\text{Markovian}}$ of the Markovian reduced model ($L = 0$), and the trajectory $\tbfY_{\text{test}}^{\text{stage}}$ of the non-Markovian reduced model with stagewise operator inference and lag $L = 100$ for $n = 8$ and 60\% observed state components (left panel) and for $n=10$ and 20\% observed state components (right panel). The learned non-Markovian reduced model provides a more accurate approximation of the full model output than the traditional Markovian reduced model.

Consider now the relative output error
\begin{equation}
    \errOut (\tbfY) = \frac{\|\bfY_{\text{test}} - \tbfY\|_F}{\|\bfY_{\text{test}}\|_F}
    \label{eq:OutErr}
\end{equation}
where $\tbfY$ is a trajectory from a reduced model for the output. Figure \ref{fig:ConvDiff_Output_Condensed} shows the output error \eqref{eq:OutErr} of the trajectory computed with the Markovian reduced model $(\tbfY = \tbfY_{\text{test}}^{\text{Markovian}})$  and compares it to the  error \eqref{eq:OutErr} of the trajectories computed with the stagewise learned non-Markovian reduced model with lag $L = 100$ $(\tbfY = \tbfY_{\text{test}}^{\text{stage}})$ for 20\%, 40\%, 60\%, and 80\% observed state components and dimension $n = 8$ (left panel) and $n = 10$ (right panel). The non-Markovian reduced model achieves errors of orders of magnitude lower than the Markovian reduced model in this example. 

If instead 99\% of the state components of the full model were observed, the approximation quality of the learned Markovian reduced model for the output improves as depicted in Figure~\ref{fig:ConvDiff_Output_PO99} for $n=10$ (left panel) and $n=8$ (right panel). We also emphasize that, depending on the nature of the output, it is still possible for the learned Markovian model to provide a poor approximation despite increasing the rate of observed components to 99\%. This underscores the benefits of incorporating a non-Markovian term to the reduced model. 

\begin{figure}
    \begin{subfigure}[b]{0.45\textwidth}
    \begin{center}
{{\Large\resizebox{1.15\columnwidth}{!}{\input{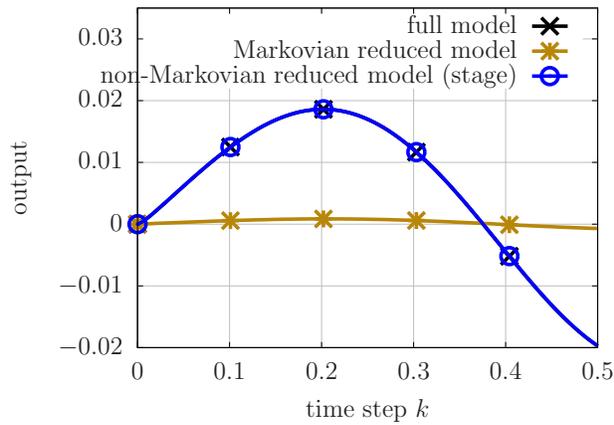}}}}
\end{center}
    \caption{$n=8, 60\%$ observed components}
  \end{subfigure}
  \quad \quad \quad
  \begin{subfigure}[b]{0.45\textwidth}
    \begin{center}
{{\Large\resizebox{1.15\columnwidth}{!}{\input{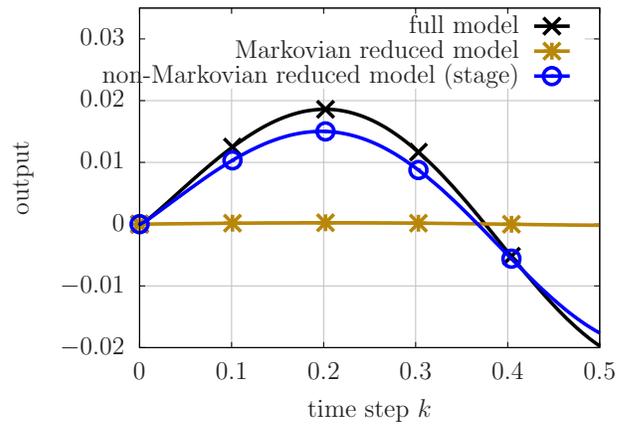}}}}
\end{center}
    \caption{$n=10, 20\%$ observed components}
  \end{subfigure}
  \caption{Convection-diffusion equation (Section~\ref{subsec:Convdiff}).  The correction terms in the non-Markovian reduced model improve the approximation offered by the Markovian reduced model.}
  \label{fig:ConvDiff_Output_Detailed}
\end{figure}

\begin{figure}
  \begin{subfigure}[b]{0.45\textwidth}
    \begin{center}
{{\Large\resizebox{1.15\columnwidth}{!}{\input{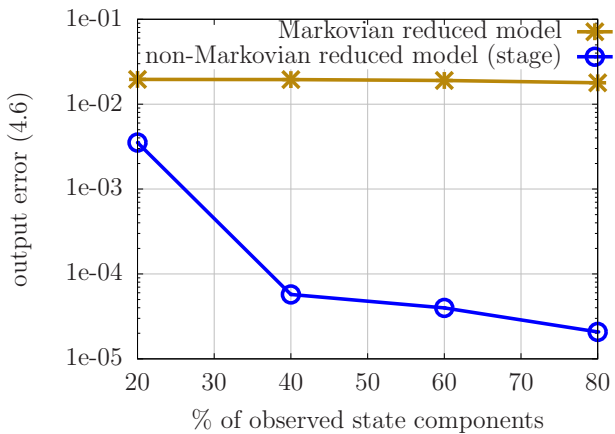}}}}
\end{center}
    \caption{output, $n=8$}
  \end{subfigure}
  \quad \quad \quad
  \begin{subfigure}[b]{0.45\textwidth}
    \begin{center}
{{\Large\resizebox{1.15\columnwidth}{!}{\input{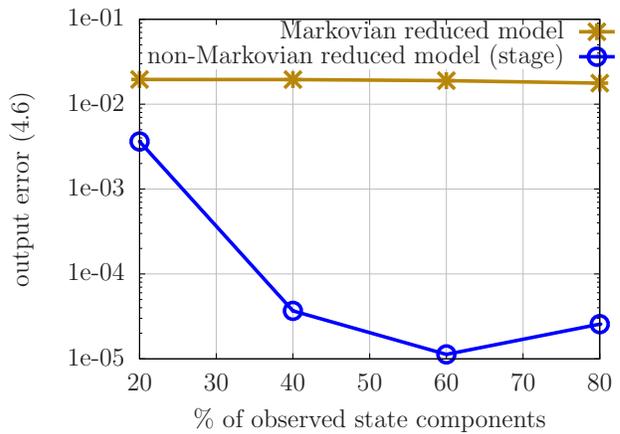}}}}
\end{center}
    \caption{output, $n=10$}
  \end{subfigure}
  \caption{Convection-diffusion equation (Section~\ref{subsec:Convdiff}). The non-Markovian reduced model for the output achieves errors which are orders of magnitude lower than the Markovian model for this example.}
  \label{fig:ConvDiff_Output_Condensed}
\end{figure}

\begin{figure}
\begin{subfigure}[b]{0.45\textwidth}
    \begin{center}
{{\Large\resizebox{1.15\columnwidth}{!}{\input{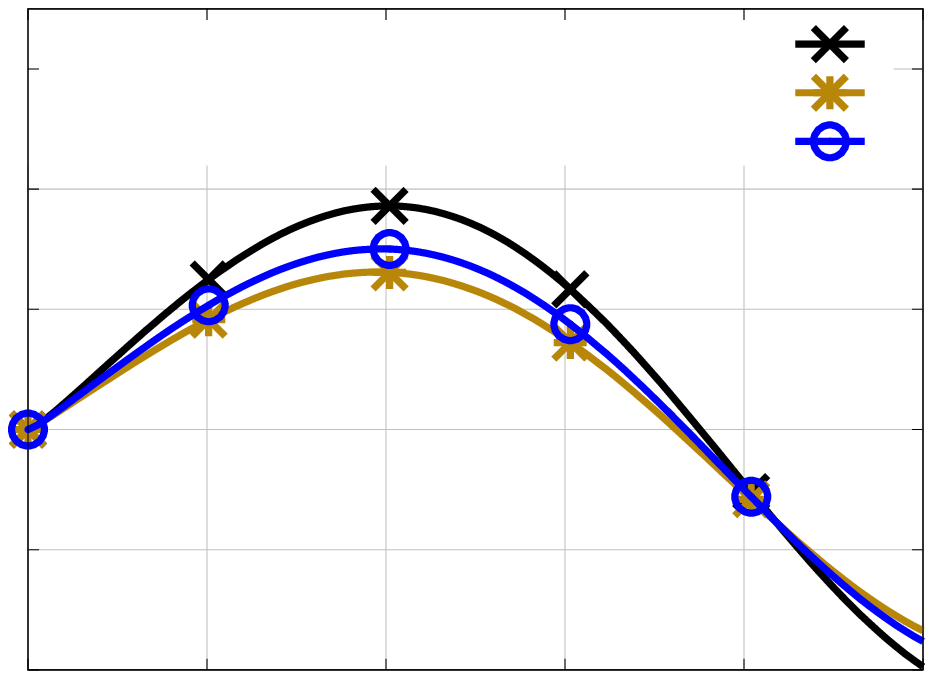}}}}
\end{center}
    \caption{$n=10$, 99\% observed components}
  \end{subfigure}
  \quad \quad \quad
\begin{subfigure}[b]{0.45\textwidth}
    \begin{center}
{{\Large\resizebox{1.15\columnwidth}{!}{\input{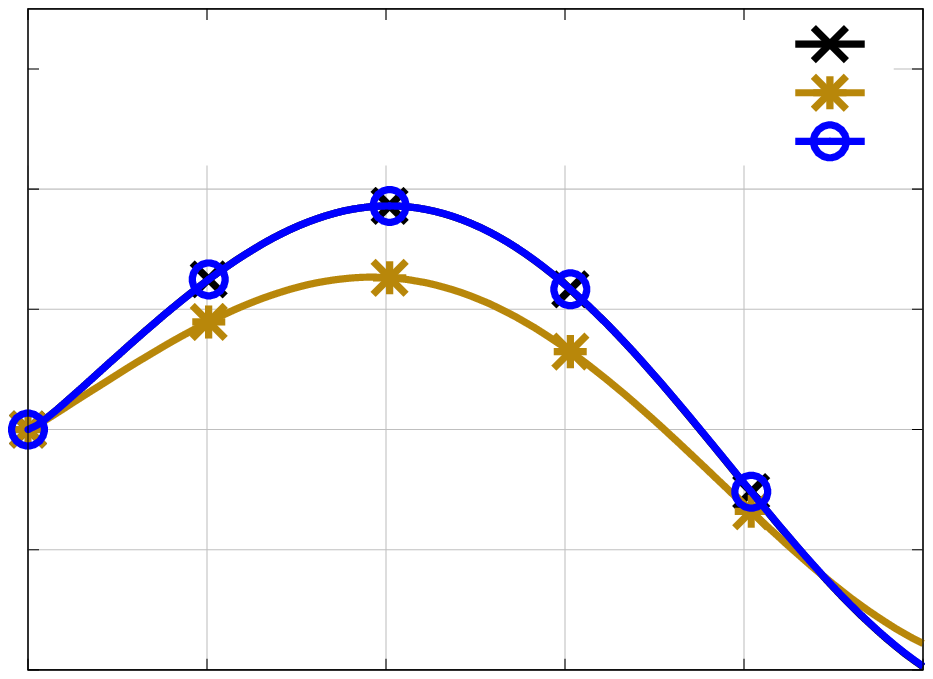}}}}
\end{center}
    \caption{$n=8$, 99\% observed components}
  \end{subfigure}
  \caption{Convection-diffusion equation (Section~\ref{subsec:Convdiff}). A Markovian reduced model needs up to 99\% of the components of the states to be observed to achieve comparable accuracy as a non-Markovian reduced model learned from only 20\% observed components in this example.}
  \label{fig:ConvDiff_Output_PO99}
\end{figure}

\subsection{Diffusion-reaction equation} \label{subsec:react2d}

We now learn a non-Markovian reduced model of a parametric nonlinear polynomial system arising from
the diffusion-reaction example in \cite{paper:Peherstorfer2019}. Set the spatial domain as $\Omega = (0,1) \times (0,1)$ with boundary $\partial \Omega$, the time domain as $t \in (0,100)$, and the parameter domain as $\mu \in \paramSpace= [1,1.5]$.  For $(\xi_1,\xi_2,t) \in \Omega \times (0,100)$ and $\mu \in \paramSpace$, the PDE is described by 
\begin{align} \label{eq:react2d}
        \frac{\partial}{\partial t}x(\xi_1,\xi_2,t ; \mu) & = - \Delta x(\xi_1,\xi_2,t;\mu) + (\sin(2 \pi \xi_1) \sin(2 \pi \xi_2)/10)u(t) + g(x(\xi_1,\xi_2,t;\mu);\mu),\\
        \nabla x(\xi_1,\xi_2,t;\mu) \cdot \mathbf{n} & = 0 \quad \text{for } (\xi_1,\xi_2) \in \partial \Omega, \notag \\ 
        x(\xi_1,\xi_2,0; \mu) & = 0 \notag
\end{align}
where $u(t)$ is the input and  $g$ is defined by $g(x;\mu) = -(0.1 \sin(\mu)+2)e^{-2.7\mu^2 }(1 + 1.8\mu x  + 1.62 \mu^2 x^2)$. The PDE  \eqref{eq:react2d} is spatially discretized on a grid with 64 equidistant points in $\xi_1$ and $\xi_2$ via the finite difference method. It is then temporally discretized with forward Euler using time step size of $\delta t = 10^{-2}$ to obtain the paremetric full model
    \begin{align} \label{eq:React2DFOMMatrix}
        \bfx_{k+1}(\mu) &= \bfA_1(\mu) \bfx_k(\mu)  + \bfA_2(\mu) \bfx_k^2(\mu) + \bfA_3(\mu) \bfx_k^3(\mu) + \bfB u_k , \quad k = 0,\dots,K-1,
    \end{align}
    with $\bfx_k(\mu) \in \R^{N}, N=64^2$ and $u_k \in \R$.

The inputs for basis generation, training, and testing are independently sampled at each time step from a uniform distribution on $[-3,996]$. The parameters for basis generation and training are $\mu_i^{\text{train}} = 1 + (i-1)/18$ for $i = 1, \dots, 10$ while the test parameters are  $\mu_i^{\text{test}} = 1 + (i-1)/16$ for $i = 1, \dots, 9$. The Markovian and non-Markovian operators are interpolated in the testing phase as described in Section~\ref{subsec:TradMOR} and \cite{paper:DegrooteVW2009}. For each training parameter $\mu_1^{\text{train}}, \dots, \mu_{10}^{\text{train}}$, we generate 20 trajectories with the extended re-projection algorithm with a total of $\Nreproj$ re-projection steps and $\Treproj$ time steps per re-projection step; each trajectory corresponds to different realizations of the input random variable. For fixed $\Nreproj$, a fair comparison across different data sets is aimed for by nesting the data sets with respect to increasing $\Treproj$. The lag $L = 40$ and $\Nreproj, \Treproj$ are always chosen so that the resulting data matrices for the least squares problems are numerically full rank.

\subsubsection{State error}

Consider the projection error averaged over $N_{\mu}$ parameters
\begin{align} \label{eq:ProjErrorParam}
    e^{\text{proj}}_{\mu} = \frac{1}{N_{\mu}}\sum_{i = 1}^{N_{\mu}} \frac{\|\bfZ(\mu_i) - \bfV\bfV^T\bfZ(\mu_i)\|_F}{\|\bfZ(\mu_i)\|_F}\,
\end{align}
where $\bfZ(\mu_i)$ is an observation trajectory for fixed parameter value $\mu_i$. Analogously, define the observation error as
\begin{align} \label{eq:StateErrParam}
        e^{\text{obs}}_{\mu}(\tbfZ)=\frac{1}{N_{\mu}}\sum_{i=1}^{N_{\mu}} \frac{\|\bfZ(\mu_i) - \bfV \tbfZ(\mu_i)\|_F}{\|\bfZ(\mu_i)\|_F}
\end{align}    
where $\tbfZ(\mu_i)$ is an observation trajectory obtained with a learned reduced model for parameter $\mu_i$. The projection error $e^{\text{proj}}_{\mu}$ depends on the dimension $\nr$ of the reduced space $\Vcal$, the number of observed state components $r$, and whether the average is taken over  the training parameters $\{\mu_i^{\text{train}}\}_{i = 1}^{10}$ or the test parameters $\{\mu_i^{\text{test}}\}_{i = 1}^9$. The observation error \eqref{eq:StateErrParam} additionally depends on all parameters that the reduced model depends on such as the lag $L$ and the data set used to train the reduced model.

Figure~\ref{fig:React2D_StateErr_StageVsBatch_ReprojMore} shows the projection error \eqref{eq:ProjErrorParam} and the observation error \eqref{eq:StateErrParam} over the test parameters and inputs for dimensions $n = 8$ (left panel) and $n = 10$ (right panel). Stagewise  inference is applied to data sets derived from re-projection with $\Nreproj = 100$ re-projection steps and $\Treproj = 42$ time steps per re-projection step; batch  inference is applied to data sets with $\Treproj = 42, 60, 80$ time steps per re-projection step. Recall that these data sets are nested. The results reported in Figure~\ref{fig:React2D_StateErr_StageVsBatch_ReprojMore} show that a non-Markovian reduced model learned with batch inference can achieve a lower error if the data set is sufficiently large, which is in agreement with the results reported for the linear full model in Section~\ref{subsubsec:ConvDiffMoreStepsThanLag}. Notice that if $\Treproj = 42$, both modes of operator inference are applied to the same training data. Missing markers indicate that models led numerically to NaN (Not a Number) after a finite number of time steps, e.g., for the non-Markovian model obtained with batch inference from 20\% observed states and $n = 8$ dimensions. Unlike in the linear case, there is no guarantee here that either mode of operator inference is able to recover the non-Markovian operators from intrusive model reduction because the full model is nonlinear polynomial and we  consider linear non-Markovian terms only (cf.~Section~\ref{subsec:Nonlinear}). 

Figure \ref{fig:React2D_StateErr_StageVsBatch_ReprojLess_dim10} shows the projection and observation error for dimension $n=10$ with re-projection parameters $\Nreproj = 12,\Treproj \in \{100,500,2500\}$ for batch inference. The left panel shows the errors over the training parameters and inputs while the right panel shows the errors over the test parameters and inputs. The results show that batch inference achieves an up to one order of magnitude lower test error than stagewise  inference in this experiment. This is because batch inference is trained on longer trajectories (larger $\Treproj$ for fixed $\Nreproj$) unlike stagewise inference which is constrained to utilize data only up to $\Treproj = L+2$ time steps if the non-Markovian term has lag $L$. Notice also that batch  inference starts to overfit for $\Treproj = 2500$, which we think is because the full model has reached steady state by 2500 time steps and so the additional training data provided by increasing the number of time steps $\Treproj$ per re-projection step is skewing the least squares problem towards the steady state behavior of the full model present in the training data set. 

\begin{figure}
  \begin{subfigure}[b]{0.45\textwidth}
    \begin{center}
{{\Large\resizebox{1.15\columnwidth}{!}{\input{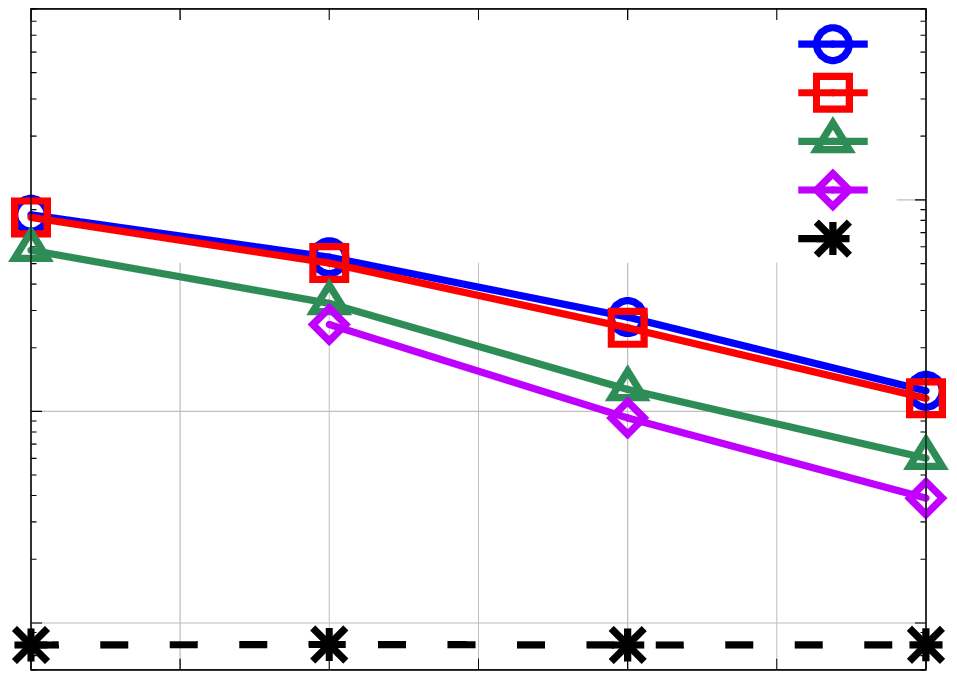}}}}
\end{center}
    \caption{Stage vs batch error, $ n=8$}
  \end{subfigure}
  \quad \quad \quad
  \begin{subfigure}[b]{0.45\textwidth}
    \begin{center}
{{\Large\resizebox{1.15\columnwidth}{!}{\input{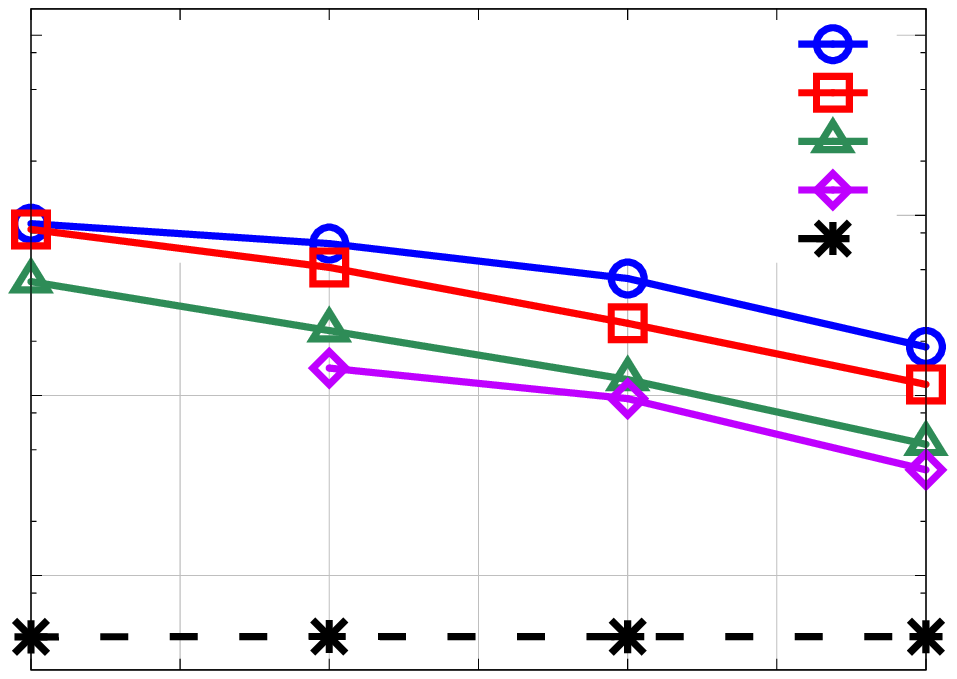}}}}
\end{center}
    \caption{Stage vs batch error, $n=10$}
  \end{subfigure}
  \caption{Diffusion-reaction equation (Section~\ref{subsec:react2d}). For nonlinear polynomial full models, neither mode of operator inference necessarily recovers the non-Markovian operators of intrusive model reduction. In this example, for a fixed number of re-projection steps, batch inference can achieve a lower error if the trajectories in the training data are sufficiently long.}
  \label{fig:React2D_StateErr_StageVsBatch_ReprojMore}
\end{figure}

\begin{figure}
  \begin{subfigure}[b]{0.45\textwidth}
    \begin{center}
{{\Large\resizebox{1.15\columnwidth}{!}{\input{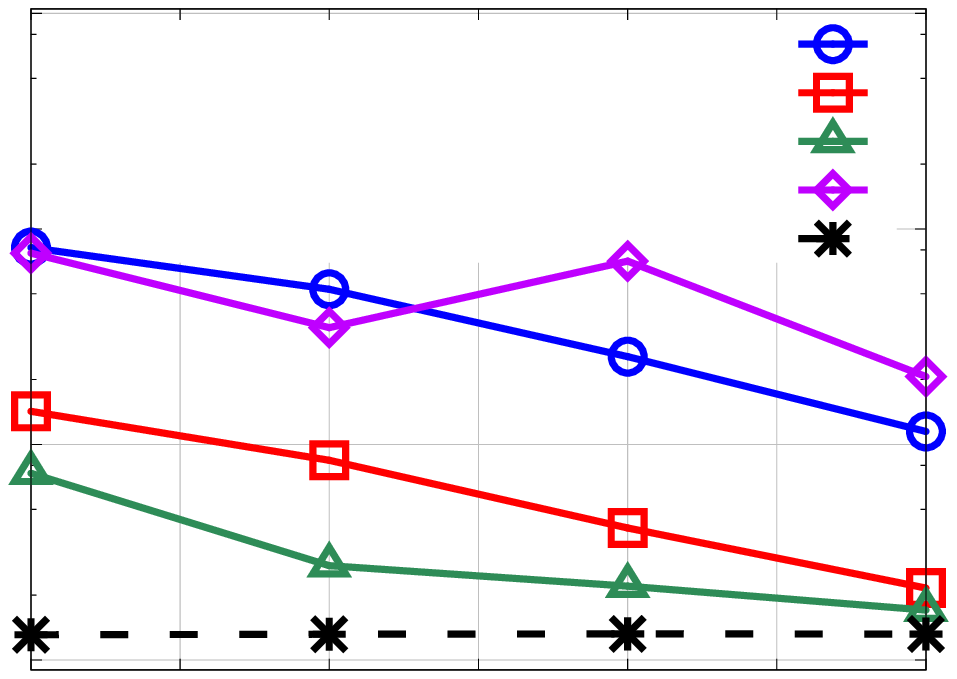}}}}
\end{center}
    \caption{Train error, $n =10$}
  \end{subfigure}
  \quad \quad \quad
  \begin{subfigure}[b]{0.45\textwidth}
    \begin{center}
{{\Large\resizebox{1.15\columnwidth}{!}{\input{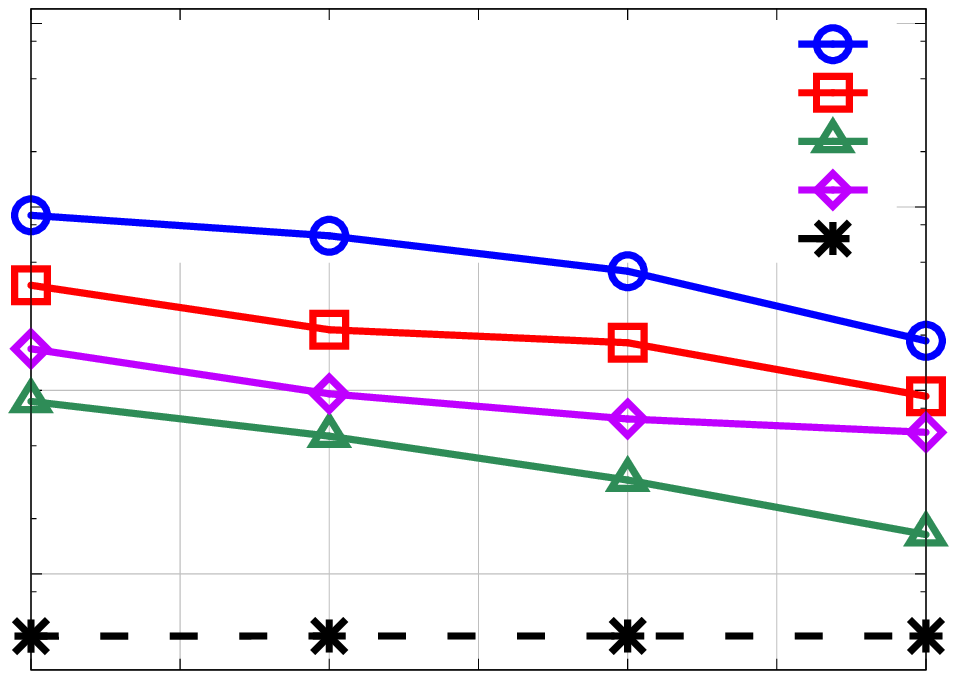}}}}
\end{center}
    \caption{Test error, $n = 10$}
  \end{subfigure}
  \caption{Diffusion-reaction equation (Section~\ref{subsec:react2d}). Batch inference is flexible in that it can be applied to training data with long trajectories (large $\Treproj$) to learn a non-Markovian reduced model with lag $L$ in contrast to stagewise inference which only utilizes data up to $\Treproj = L+2$ time steps. In this example, this results in a non-Markovian model learned with batch inference that achieves lower errors than its stagewise counterpart.}
  \label{fig:React2D_StateErr_StageVsBatch_ReprojLess_dim10}
\end{figure}

\subsection{Chafee-Infante equation}
\label{subsec:Chafee}

Consider the Chafee-Infante equation on the spatial domain $\Omega = (0,1)$ and time domain $t \in (0,4)$. For $(\xi,t) \in \Omega \times (0,4)$, if $u(t)$ is the input, the PDE satisfies
\begin{align} \label{eq:chafeePDE}
        \frac{\partial}{\partial t} x(\xi,t) & = \frac{\partial^2}{\partial \xi^2} x(\xi,t) - x^3(\xi,t) + x(\xi,t), \\
        \frac{\partial }{\partial \xi} x(1,t) & = 0, \notag \\
        x(0,t) & = u(t),  \notag \\
        x(\xi,0) & = 0. \notag
\end{align}
The quantity of interest is modeled as $y(t) = x(1,t)$ which is the solution at the boundary point $\xi = 1$. An approximation to the solution of \eqref{eq:chafeePDE} is sought at 128 equidistant spatial nodes in $(0,1)$ via the finite difference method. Forward Euler is then employed to discretize the PDE temporally with $\delta t = 10^{-5}$. This gives the full model
\begin{align} \label{eq:chafeeFOM}
        \bfx_{k+1} & = \bfA_1 \bfx_k +  \bfA_3 \bfx_k^3 + \bfB u_k, \qquad k = 0,\dots,K-1, \\
        y_{k+1} & = \bfC \bfx_{k+1} \notag
\end{align}
with $\bfx_k \in \R^{N}, N = 128$, and $u_k, y_k \in \R$.

The basis matrix $\bfV$ for the observations is computed with inputs sampled from a uniform distribution on $[0,10]$. The training data set is generated as follows. First, trajectories $i = 1, \dots, 15$  are obtained with re-projection using the inputs  $u_{k,i}^{\text{train}} = 10 \theta_k^{(i)} (\cos(50 \pi k \delta t \gamma_k^{(i)} ) + 1)$ for the $i$-th trajectory, where $\theta_k^{(i)}, \gamma_k^{(i)}$ are realizations of a random variable with uniform distribution on $[0, 1]$. Second, training trajectories $i = 16, \dots, 30$ are obtained with inputs that are realizations of a uniform distribution on $[0, 10]$. The lag is set to $L = 60$ and the non-Markovian operators are then learned as a batch with $\Nreproj = 600$ (20 re-projection steps per input) and  $\Treproj = 100$, and in a stagewise manner on the subset generated by the $\Treproj = 62$ time steps. The test input is $u^{\text{test}}(t) = 5 (\sin (\pi t) + 1)$.

\subsubsection{Output error}

Consider the output error  defined analogously as \eqref{eq:OutErr} for the non-Markovian reduced model obtained with batch and stagewise inference; see Figure~\ref{fig:Chafee_Output_Condensed}. The results indicate that a Markovian reduced model fails to make predictions about the full model output when one only observes partial components of the state. In contrast, the proposed non-Markovian reduced models learned with stagewise or batch inference achieve more than one order of magnitude improvement compared to the Markovian reduced model.

Figure \ref{fig:Chafee_Output_Detailed} compares the full model outputs over time with the approximations given by the Markovian and the non-Markovian reduced models learned from 40\%, 60\%, and 80\% observed state components. In agreement with the results shown in Figure~\ref{fig:ChafeeProbStatement}, the Markovian reduced model is unable to capture the oscillatory behavior of the output of the full model. In contrast, the non-Markovian reduced models capture the oscillatory output behavior even if only 40\% of the state components are observed and provide a close approximation as this rate is increased to 80\% of the state components. In the training data, each re-projection step is succeeded by at most $\Treproj = 100$ time steps with the full model. This corresponds to $t = 0.001$ and is plotted as a dashed vertical line in each panel of Figure \ref{fig:Chafee_Output_Detailed}. This training time length is not long enough to cover the oscillatory nature of the output. Yet, the resulting approximation with the non-Markovian reduced model provides a reasonable approximation far outside of the training regime, which emphasizes that the dynamics of the underlying system are learned rather than mere interpolations between training data samples.

\begin{figure}
  \begin{subfigure}[b]{0.45\textwidth}
    \begin{center}
{{\Large\resizebox{1.15\columnwidth}{!}{\input{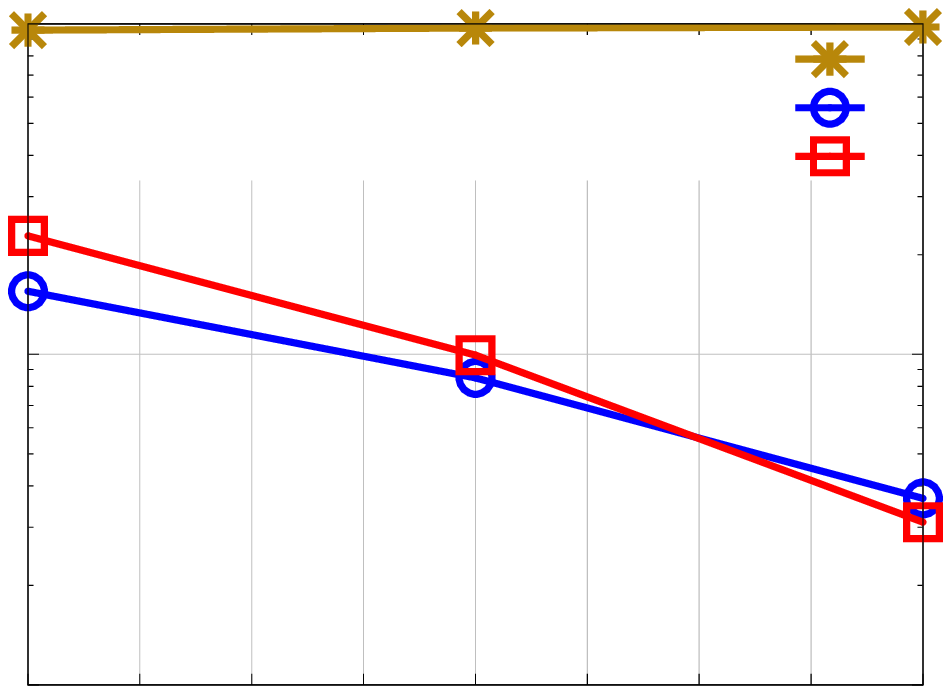}}}}
\end{center}
    \caption{Stage vs batch error, $n=8$}
  \end{subfigure}
  \quad \quad \quad
  \begin{subfigure}[b]{0.45\textwidth}
    \begin{center}
{{\Large\resizebox{1.15\columnwidth}{!}{\input{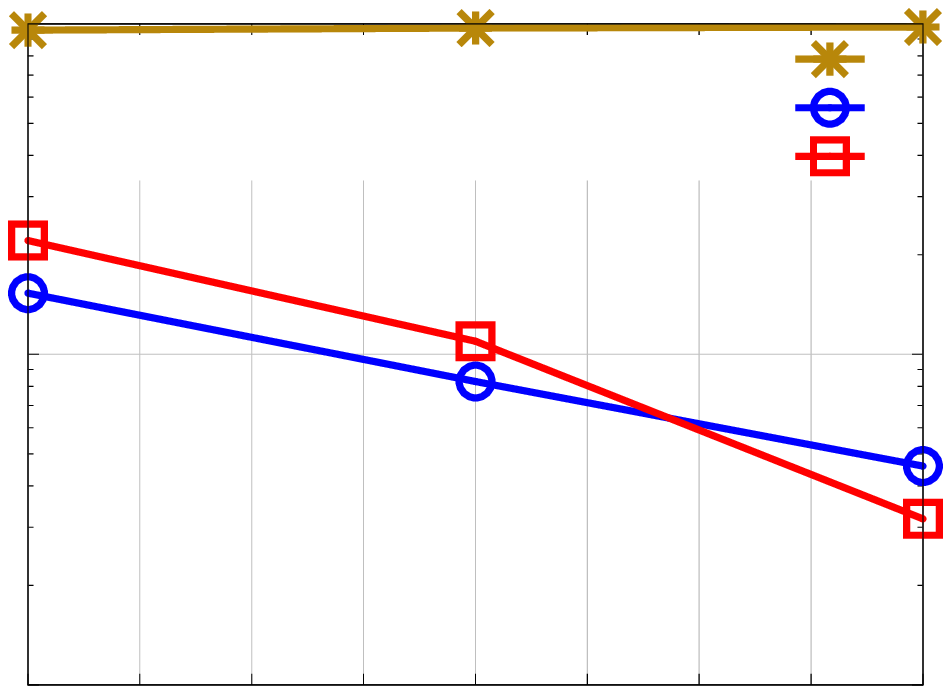}}}}
\end{center}
    \caption{Stage vs batch error, $n=10$}
  \end{subfigure}
  \caption{Chafee-Infante equation (Section~\ref{subsec:Chafee}). The non-Markovian reduced models achieve an error of more than one order of magnitude lower than its Markovian counterpart.}
  \label{fig:Chafee_Output_Condensed}
\end{figure}

\begin{figure}
  \begin{subfigure}[b]{0.45\textwidth}
    \begin{center}
{{\Large\resizebox{1.15\columnwidth}{!}{\input{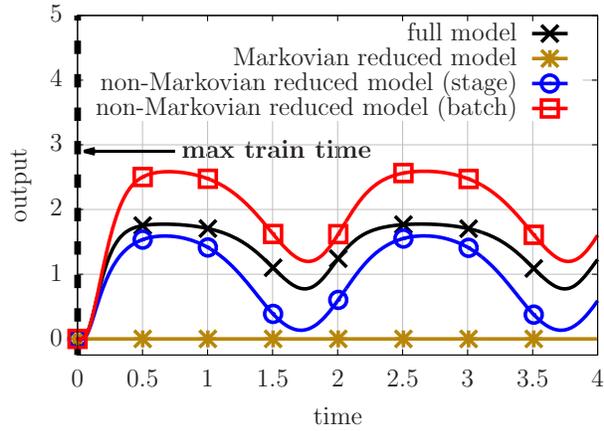}}}}
\end{center}
    \caption{40\% observed state components}
  \end{subfigure}
  \hspace{3em}
  \begin{subfigure}[b]{0.45\textwidth}
   \begin{center}
{{\Large\resizebox{1.15\columnwidth}{!}{\input{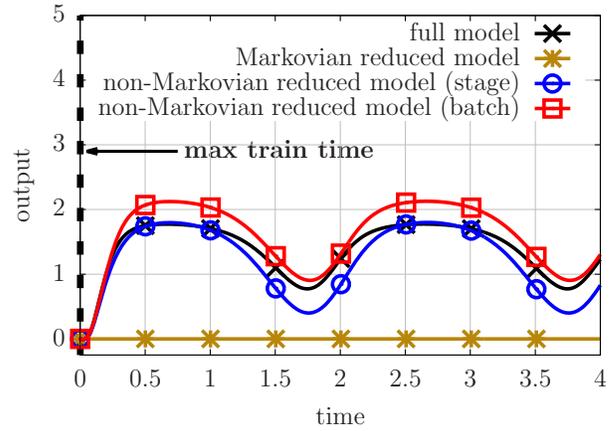}}}}
\end{center}
    \caption{60\% observed state components}
  \end{subfigure}
\begin{center}
  \begin{subfigure}[b]{0.45\textwidth}
{{\Large\resizebox{1.15\columnwidth}{!}{\input{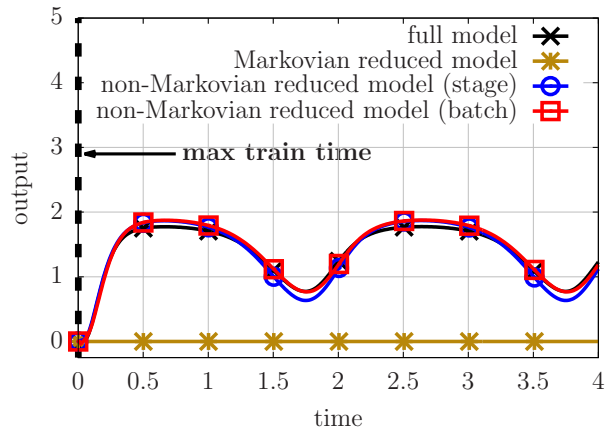}}}}
    \caption{80\% observed state components}
  \end{subfigure}
\end{center}
  \caption{Chafee-Infante equation (Section~\ref{subsec:Chafee}). The non-Markovian reduced model for the output captures the oscillatory nature of the full model output even though the models are learned from much shorter training trajectories than the prediction end time. Thus, the learned non-Markovian model provide reasonable approximations far beyond the training regime in this example.}
  \label{fig:Chafee_Output_Detailed}
\end{figure}

\section*{Acknowledgments}
This work was partially supported by US Department of Energy, Office of Advanced Scientific Computing Research, Applied Mathematics Program (Program Manager Dr. Steven Lee), DOE Award DESC0019334, and by the National Science Foundation under Grant No.~1901091 and under Grant No.~1761068.

\bibliographystyle{abbrv}
\bibliography{references.bib}

\begin{appendices}
\section{Error analysis of the non-Markovian reduced model} \label{appendix}

We build on the analysis in Section~\ref{subsubsec:LinearProposedClosure} and demonstrate numerically that  Markovian reduced models for linear autonomous full systems can achieve lower errors than the proposed non-Markovian reduced models \eqref{eq:ClosureLinear}. To motivate the numerical experiments that follow, consider a full model with dimension $N = 2$ and with observed state of dimension $r = 1$ and reduced dimension $n=1$. A sufficient condition for which the proposed reduced model with non-Markovian term \eqref{eq:ClosureLinear} yields a lower error than a Markovian reduced model is when $\bfA_1$ is symmetric positive definite. To see this, observe that  for positive integers $l$ ($l \in \mathbb{Z}^+$), $\tbfA_1, \bfE_l \in \R$ and that $\tbfA_1 >0$, $$\bfE_l = (\bfQ^T \bfA_1 \Qperp)^2 ((\Qperp)^T \bfA_1 \Qperp)^{l-1} > 0.$$ Provided that $\tbfz_0 = \tbfz_0^{(0)}= \tbfz_0^{(L)}$, for fixed $k \in \mathbb{Z}^+$, if $\tbfz_k^{(0)}$ and $\tbfz_k^{(L)}$ are expressed in terms of the initial condition $\tbfz_0$, algebraic calculations show that 
$$\|\tbfz_k - \tbfz_k^{(L)}\|_2 \le \|\tbfz_k - \tbfz_k^{(0)}\|_2$$
since $\tbfA_1, \bfE_l$ are positive for all $l \in \mathbb{Z}^+.$ Therefore, since 
\begin{align*}
        \|\bfz_k - \bfV \tbfz_k^{(0)}\|_2 & = \|\bfz_k - \bfV \tbfz_k\|_2 + \|\bfV(\tbfz_k - \tbfz_k^{(0)})\|_2, \\
        \|\bfz_k - \bfV \tbfz_k^{(L)}\|_2 & = \|\bfz_k - \bfV \tbfz_k\|_2 + \|\bfV(\tbfz_k - \tbfz_k^{(L)})\|_2,
\end{align*}
we conclude that $$\|\bfz_k - \bfV \tbfz_k^{(L)}\|_2 \le \|\bfz_k - \bfV \tbfz_k^{(0)}\|_2,$$
i.e., the reduced model with non-Markovian term achieves a lower error than its Markovian counterpart.

However, the symmetric positive definiteness of the matrix $\bfA_1$ is insufficient when $N > 2, n > 1$. To see this, consider the following two examples with lag  $L=1$. A numerical implementation is available in Python\footnote{\href{https://github.com/wayneisaacuy/OpInfPartialObs}{https://github.com/wayneisaacuy/OpInfPartialObs}} which reproduces Figure~\ref{fig:Appendix} below. We set $N = 10, n=2$ and consider 30\% observed state components for the first example while for the second, we use $N = 50, n = 40$ and consider 95\% observed state components. In both cases, the initial condition $\tbfz_0$ is chosen such that its components are realizations of independent standard normal random variables. The initial condition for the full system is then $\bfx_0 = \bfQ \tbfz_0$ so that $\bfx_0$ satisfies $(\Qperp)^T \bfx_0 = \boldsymbol{0}_{N-n}$. 

The symmetric positive definite matrix $\bfA_1$ is constructed as follows. Its eigenvalues are sampled from a uniform distribution on $(0,1)$ to ensure that the system is stable. Its orthonormal eigenvectors are then chosen to be the eigenvectors of $(\bfR + \bfR^T)/2$ where $\bfR^{N \times N}$ is a matrix whose entries are independently sampled from a uniform distribution on $(0,10)$. The components with indices 1,6,10 of the full state are observed in the first example with the initial condition and basis and system matrices given by
    \begin{align*}
        \bfx_0 & = \begin{bmatrix} -0.5960 & 0 & 0 & 0 & 0 & 1.0333 &  0 & 0 & 0 & 0.8346
    \end{bmatrix}^T, \\
        \bfV & = \begin{bmatrix} 
                -0.9889 & 0.0294\\  
                0.0767 & -0.7374\\  
                -0.1269 & -0.6748
                \end{bmatrix},\\
        \basisPerp & = \begin{bmatrix} -0.1453\\ 
                                        -0.6710\\ 
                                        0.7270\\
                        \end{bmatrix},\\
        \bfA_1 & = 
    \begin{bmatrix} 0.3603 & 0.0184 & -0.2192 & 0.0435 & -0.1624 & -0.0602 & 0.0758 & -0.0872 & 0.0634 & -0.0252\\  0.0184 & 0.2907 & -0.1049 & 0.1334 & 0.0087 & 0.0951 & -0.0594 & -0.0602 & -0.0717 &  0.1366\\  -0.2192 & -0.1049 & 0.2978 & -0.1695 & 0.0887 & 0.0648 & -0.0924 & 0.0624 & -0.0213 &  0.0079\\  0.0435 & 0.1334 & -0.1695 & 0.3700 & 0.0529 & -0.0074 & 0.1284 & 0.0196 & -0.0115 & 0.0273\\  -0.1624 & 0.0087 & 0.0887 & 0.0529 & 0.4582 & 0.0913 & 0.1194 & -0.0375 & 0.0449 & 0.1615 \\  -0.0602 & 0.0951 & 0.0648 & -0.0074 & 0.0913 & 0.4311 & -0.0781 & -0.0263 & 0.2070 &  0.1714\\ 0.0758 & -0.0594 & -0.0924 & 0.1284 & 0.1194 & -0.0781 & 0.3804 & 0.0296 & 0.1548 &  -0.1197\\  -0.0872 & -0.0602 & 0.0624 & 0.0196 & -0.0375 & -0.0263 & 0.0296 & 0.3470 & 0.1123 & -0.1761\\  0.0634 & -0.0717 & -0.0213 & -0.0115 & 0.0449 & 0.2070 & 0.1548 & 0.1123 & 0.5707 &  -0.1059\\  -0.0252 & 0.1366 & 0.0079 & 0.0273 & 0.1615 & 0.1714 & -0.1197 & -0.1761 & -0.1059 & 0.3255
    \end{bmatrix}.                
    \end{align*}
The details of the second example are provided in the repository\footnotemark[1].

Figure~\ref{fig:Appendix} shows the difference in the relative error
$$\frac{1}{\|\bfz_k\|_2} (\|\bfz_k - \bfV \tbfz_k^{(0)}\|_2 - \|\bfz_k - \bfV \tbfz_k^{(L)}\|_2)$$
against the time step $k$. At certain time instances, the Markovian reduced model has a smaller error (negative values on the $y$-axis) than the model with non-Markovian term of lag $L=1$. Thus, the conclusion we derived for $N = 2,n=1$ does not generalize and these examples show that it is possible that the Markovian model gives a more accurate approximation than the truncated non-Markovian model even if the matrix $\bfA_1$ is symmetric positive definite. A more rigorous analysis is warranted but is beyond the scope of this work.

\end{appendices}

\begin{figure}
  \begin{subfigure}[b]{0.45\textwidth}
    \begin{center}
{{\Large\resizebox{1.15\columnwidth}{!}{\input{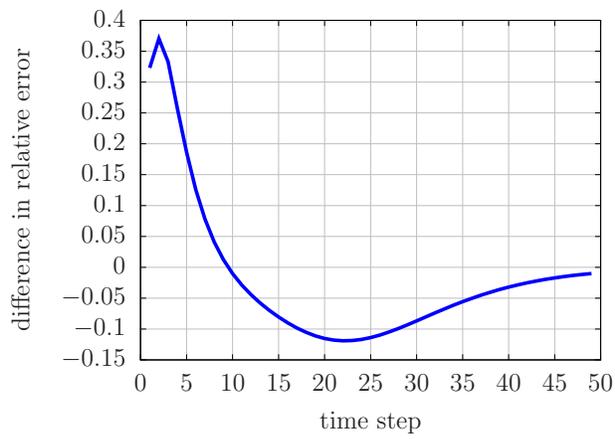}}}}
\end{center}
    \caption{$N = 10,n=2$, 30\% observed state components}
  \end{subfigure}
  \quad \quad \quad
  \begin{subfigure}[b]{0.45\textwidth}
    \begin{center}
{{\Large\resizebox{1.15\columnwidth}{!}{\input{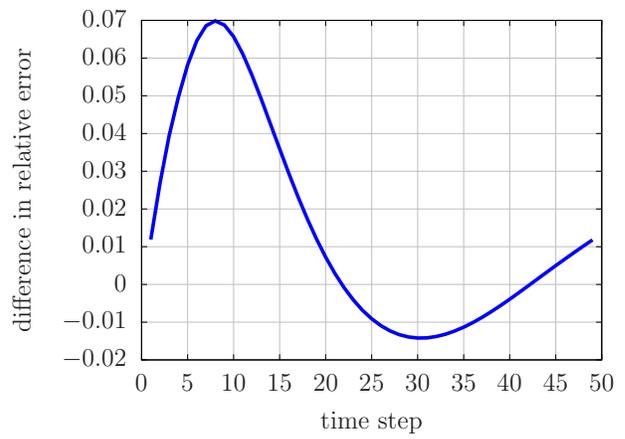}}}}
\end{center}
    \caption{$N=50,n=40$, 95\% observed state components}
  \end{subfigure}
  \caption{The Markovian model yields a more accurate approximation of the observed state dynamics at certain time points than the model with truncated non-Markovian term in this example.}
  \label{fig:Appendix}
\end{figure}

\end{document}

%% file: figures/Chafee_ProblemStatement_PO100.tex
\begingroup
  \makeatletter
  \providecommand\color[2][]{%
    \GenericError{(gnuplot) \space\space\space\@spaces}{%
      Package color not loaded in conjunction with
      terminal option `colourtext'%
    }{See the gnuplot documentation for explanation.%
    }{Either use 'blacktext' in gnuplot or load the package
      color.sty in LaTeX.}%
    \renewcommand\color[2][]{}%
  }%
  \providecommand\includegraphics[2][]{%
    \GenericError{(gnuplot) \space\space\space\@spaces}{%
      Package graphicx or graphics not loaded%
    }{See the gnuplot documentation for explanation.%
    }{The gnuplot epslatex terminal needs graphicx.sty or graphics.sty.}%
    \renewcommand\includegraphics[2][]{}%
  }%
  \providecommand\rotatebox[2]{#2}%
  \@ifundefined{ifGPcolor}{%
    \newif\ifGPcolor
    \GPcolortrue
  }{}%
  \@ifundefined{ifGPblacktext}{%
    \newif\ifGPblacktext
    \GPblacktexttrue
  }{}%
  \let\gplgaddtomacro\g@addto@macro
  \gdef\gplbacktext{}%
  \gdef\gplfronttext{}%
  \makeatother
  \ifGPblacktext
    \def\colorrgb#1{}%
    \def\colorgray#1{}%
  \else
    \ifGPcolor
      \def\colorrgb#1{\color[rgb]{#1}}%
      \def\colorgray#1{\color[gray]{#1}}%
      \expandafter\def\csname LTw\endcsname{\color{white}}%
      \expandafter\def\csname LTb\endcsname{\color{black}}%
      \expandafter\def\csname LTa\endcsname{\color{black}}%
      \expandafter\def\csname LT0\endcsname{\color[rgb]{1,0,0}}%
      \expandafter\def\csname LT1\endcsname{\color[rgb]{0,1,0}}%
      \expandafter\def\csname LT2\endcsname{\color[rgb]{0,0,1}}%
      \expandafter\def\csname LT3\endcsname{\color[rgb]{1,0,1}}%
      \expandafter\def\csname LT4\endcsname{\color[rgb]{0,1,1}}%
      \expandafter\def\csname LT5\endcsname{\color[rgb]{1,1,0}}%
      \expandafter\def\csname LT6\endcsname{\color[rgb]{0,0,0}}%
      \expandafter\def\csname LT7\endcsname{\color[rgb]{1,0.3,0}}%
      \expandafter\def\csname LT8\endcsname{\color[rgb]{0.5,0.5,0.5}}%
    \else
      \def\colorrgb#1{\color{black}}%
      \def\colorgray#1{\color[gray]{#1}}%
      \expandafter\def\csname LTw\endcsname{\color{white}}%
      \expandafter\def\csname LTb\endcsname{\color{black}}%
      \expandafter\def\csname LTa\endcsname{\color{black}}%
      \expandafter\def\csname LT0\endcsname{\color{black}}%
      \expandafter\def\csname LT1\endcsname{\color{black}}%
      \expandafter\def\csname LT2\endcsname{\color{black}}%
      \expandafter\def\csname LT3\endcsname{\color{black}}%
      \expandafter\def\csname LT4\endcsname{\color{black}}%
      \expandafter\def\csname LT5\endcsname{\color{black}}%
      \expandafter\def\csname LT6\endcsname{\color{black}}%
      \expandafter\def\csname LT7\endcsname{\color{black}}%
      \expandafter\def\csname LT8\endcsname{\color{black}}%
    \fi
  \fi
    \setlength{\unitlength}{0.0500bp}%
    \ifx\gptboxheight\undefined%
      \newlength{\gptboxheight}%
      \newlength{\gptboxwidth}%
      \newsavebox{\gptboxtext}%
    \fi%
    \setlength{\fboxrule}{0.5pt}%
    \setlength{\fboxsep}{1pt}%
\begin{picture}(7200.00,5040.00)%
    \gplgaddtomacro\gplbacktext{%
      \csname LTb\endcsname%
      \put(1036,1189){\makebox(0,0)[r]{\strut{}$0$}}%
      \csname LTb\endcsname%
      \put(1036,1775){\makebox(0,0)[r]{\strut{}$0.5$}}%
      \csname LTb\endcsname%
      \put(1036,2360){\makebox(0,0)[r]{\strut{}$1$}}%
      \csname LTb\endcsname%
      \put(1036,2946){\makebox(0,0)[r]{\strut{}$1.5$}}%
      \csname LTb\endcsname%
      \put(1036,3532){\makebox(0,0)[r]{\strut{}$2$}}%
      \csname LTb\endcsname%
      \put(1036,4117){\makebox(0,0)[r]{\strut{}$2.5$}}%
      \csname LTb\endcsname%
      \put(1036,4703){\makebox(0,0)[r]{\strut{}$3$}}%
      \csname LTb\endcsname%
      \put(1204,616){\makebox(0,0){\strut{}$0$}}%
      \csname LTb\endcsname%
      \put(1890,616){\makebox(0,0){\strut{}$0.5$}}%
      \csname LTb\endcsname%
      \put(2577,616){\makebox(0,0){\strut{}$1$}}%
      \csname LTb\endcsname%
      \put(3263,616){\makebox(0,0){\strut{}$1.5$}}%
      \csname LTb\endcsname%
      \put(3950,616){\makebox(0,0){\strut{}$2$}}%
      \csname LTb\endcsname%
      \put(4636,616){\makebox(0,0){\strut{}$2.5$}}%
      \csname LTb\endcsname%
      \put(5322,616){\makebox(0,0){\strut{}$3$}}%
      \csname LTb\endcsname%
      \put(6009,616){\makebox(0,0){\strut{}$3.5$}}%
      \csname LTb\endcsname%
      \put(6695,616){\makebox(0,0){\strut{}$4$}}%
    }%
    \gplgaddtomacro\gplfronttext{%
      \csname LTb\endcsname%
      \put(224,2799){\rotatebox{-270}{\makebox(0,0){\strut{}output}}}%
      \put(3949,196){\makebox(0,0){\strut{}time}}%
      \csname LTb\endcsname%
      \put(5288,4500){\makebox(0,0)[r]{\strut{}full model}}%
      \csname LTb\endcsname%
      \put(5288,4220){\makebox(0,0)[r]{\strut{}learned model}}%
    }%
    \gplbacktext
    \put(0,0){\includegraphics{Chafee_ProblemStatement_PO100}}%
    \gplfronttext
  \end{picture}%
\endgroup

%% file: figures/Chafee_ProblemStatement.tex
\begingroup
  \makeatletter
  \providecommand\color[2][]{%
    \GenericError{(gnuplot) \space\space\space\@spaces}{%
      Package color not loaded in conjunction with
      terminal option `colourtext'%
    }{See the gnuplot documentation for explanation.%
    }{Either use 'blacktext' in gnuplot or load the package
      color.sty in LaTeX.}%
    \renewcommand\color[2][]{}%
  }%
  \providecommand\includegraphics[2][]{%
    \GenericError{(gnuplot) \space\space\space\@spaces}{%
      Package graphicx or graphics not loaded%
    }{See the gnuplot documentation for explanation.%
    }{The gnuplot epslatex terminal needs graphicx.sty or graphics.sty.}%
    \renewcommand\includegraphics[2][]{}%
  }%
  \providecommand\rotatebox[2]{#2}%
  \@ifundefined{ifGPcolor}{%
    \newif\ifGPcolor
    \GPcolortrue
  }{}%
  \@ifundefined{ifGPblacktext}{%
    \newif\ifGPblacktext
    \GPblacktexttrue
  }{}%
  \let\gplgaddtomacro\g@addto@macro
  \gdef\gplbacktext{}%
  \gdef\gplfronttext{}%
  \makeatother
  \ifGPblacktext
    \def\colorrgb#1{}%
    \def\colorgray#1{}%
  \else
    \ifGPcolor
      \def\colorrgb#1{\color[rgb]{#1}}%
      \def\colorgray#1{\color[gray]{#1}}%
      \expandafter\def\csname LTw\endcsname{\color{white}}%
      \expandafter\def\csname LTb\endcsname{\color{black}}%
      \expandafter\def\csname LTa\endcsname{\color{black}}%
      \expandafter\def\csname LT0\endcsname{\color[rgb]{1,0,0}}%
      \expandafter\def\csname LT1\endcsname{\color[rgb]{0,1,0}}%
      \expandafter\def\csname LT2\endcsname{\color[rgb]{0,0,1}}%
      \expandafter\def\csname LT3\endcsname{\color[rgb]{1,0,1}}%
      \expandafter\def\csname LT4\endcsname{\color[rgb]{0,1,1}}%
      \expandafter\def\csname LT5\endcsname{\color[rgb]{1,1,0}}%
      \expandafter\def\csname LT6\endcsname{\color[rgb]{0,0,0}}%
      \expandafter\def\csname LT7\endcsname{\color[rgb]{1,0.3,0}}%
      \expandafter\def\csname LT8\endcsname{\color[rgb]{0.5,0.5,0.5}}%
    \else
      \def\colorrgb#1{\color{black}}%
      \def\colorgray#1{\color[gray]{#1}}%
      \expandafter\def\csname LTw\endcsname{\color{white}}%
      \expandafter\def\csname LTb\endcsname{\color{black}}%
      \expandafter\def\csname LTa\endcsname{\color{black}}%
      \expandafter\def\csname LT0\endcsname{\color{black}}%
      \expandafter\def\csname LT1\endcsname{\color{black}}%
      \expandafter\def\csname LT2\endcsname{\color{black}}%
      \expandafter\def\csname LT3\endcsname{\color{black}}%
      \expandafter\def\csname LT4\endcsname{\color{black}}%
      \expandafter\def\csname LT5\endcsname{\color{black}}%
      \expandafter\def\csname LT6\endcsname{\color{black}}%
      \expandafter\def\csname LT7\endcsname{\color{black}}%
      \expandafter\def\csname LT8\endcsname{\color{black}}%
    \fi
  \fi
    \setlength{\unitlength}{0.0500bp}%
    \ifx\gptboxheight\undefined%
      \newlength{\gptboxheight}%
      \newlength{\gptboxwidth}%
      \newsavebox{\gptboxtext}%
    \fi%
    \setlength{\fboxrule}{0.5pt}%
    \setlength{\fboxsep}{1pt}%
\begin{picture}(7200.00,5040.00)%
    \gplgaddtomacro\gplbacktext{%
      \csname LTb\endcsname%
      \put(1036,1189){\makebox(0,0)[r]{\strut{}$0$}}%
      \csname LTb\endcsname%
      \put(1036,1775){\makebox(0,0)[r]{\strut{}$0.5$}}%
      \csname LTb\endcsname%
      \put(1036,2360){\makebox(0,0)[r]{\strut{}$1$}}%
      \csname LTb\endcsname%
      \put(1036,2946){\makebox(0,0)[r]{\strut{}$1.5$}}%
      \csname LTb\endcsname%
      \put(1036,3532){\makebox(0,0)[r]{\strut{}$2$}}%
      \csname LTb\endcsname%
      \put(1036,4117){\makebox(0,0)[r]{\strut{}$2.5$}}%
      \csname LTb\endcsname%
      \put(1036,4703){\makebox(0,0)[r]{\strut{}$3$}}%
      \csname LTb\endcsname%
      \put(1204,616){\makebox(0,0){\strut{}$0$}}%
      \csname LTb\endcsname%
      \put(1890,616){\makebox(0,0){\strut{}$0.5$}}%
      \csname LTb\endcsname%
      \put(2577,616){\makebox(0,0){\strut{}$1$}}%
      \csname LTb\endcsname%
      \put(3263,616){\makebox(0,0){\strut{}$1.5$}}%
      \csname LTb\endcsname%
      \put(3950,616){\makebox(0,0){\strut{}$2$}}%
      \csname LTb\endcsname%
      \put(4636,616){\makebox(0,0){\strut{}$2.5$}}%
      \csname LTb\endcsname%
      \put(5322,616){\makebox(0,0){\strut{}$3$}}%
      \csname LTb\endcsname%
      \put(6009,616){\makebox(0,0){\strut{}$3.5$}}%
      \csname LTb\endcsname%
      \put(6695,616){\makebox(0,0){\strut{}$4$}}%
    }%
    \gplgaddtomacro\gplfronttext{%
      \csname LTb\endcsname%
      \put(224,2799){\rotatebox{-270}{\makebox(0,0){\strut{}output}}}%
      \put(3949,196){\makebox(0,0){\strut{}time}}%
      \csname LTb\endcsname%
      \put(5288,4500){\makebox(0,0)[r]{\strut{}full model}}%
      \csname LTb\endcsname%
      \put(5288,4220){\makebox(0,0)[r]{\strut{}learned model}}%
    }%
    \gplbacktext
    \put(0,0){\includegraphics{Chafee_ProblemStatement}}%
    \gplfronttext
  \end{picture}%
\endgroup

%% file: figures/ConvDiff_OpNormDecay_Stage_dim10State.tex
\begingroup
  \makeatletter
  \providecommand\color[2][]{%
    \GenericError{(gnuplot) \space\space\space\@spaces}{%
      Package color not loaded in conjunction with
      terminal option `colourtext'%
    }{See the gnuplot documentation for explanation.%
    }{Either use 'blacktext' in gnuplot or load the package
      color.sty in LaTeX.}%
    \renewcommand\color[2][]{}%
  }%
  \providecommand\includegraphics[2][]{%
    \GenericError{(gnuplot) \space\space\space\@spaces}{%
      Package graphicx or graphics not loaded%
    }{See the gnuplot documentation for explanation.%
    }{The gnuplot epslatex terminal needs graphicx.sty or graphics.sty.}%
    \renewcommand\includegraphics[2][]{}%
  }%
  \providecommand\rotatebox[2]{#2}%
  \@ifundefined{ifGPcolor}{%
    \newif\ifGPcolor
    \GPcolortrue
  }{}%
  \@ifundefined{ifGPblacktext}{%
    \newif\ifGPblacktext
    \GPblacktexttrue
  }{}%
  \let\gplgaddtomacro\g@addto@macro
  \gdef\gplbacktext{}%
  \gdef\gplfronttext{}%
  \makeatother
  \ifGPblacktext
    \def\colorrgb#1{}%
    \def\colorgray#1{}%
  \else
    \ifGPcolor
      \def\colorrgb#1{\color[rgb]{#1}}%
      \def\colorgray#1{\color[gray]{#1}}%
      \expandafter\def\csname LTw\endcsname{\color{white}}%
      \expandafter\def\csname LTb\endcsname{\color{black}}%
      \expandafter\def\csname LTa\endcsname{\color{black}}%
      \expandafter\def\csname LT0\endcsname{\color[rgb]{1,0,0}}%
      \expandafter\def\csname LT1\endcsname{\color[rgb]{0,1,0}}%
      \expandafter\def\csname LT2\endcsname{\color[rgb]{0,0,1}}%
      \expandafter\def\csname LT3\endcsname{\color[rgb]{1,0,1}}%
      \expandafter\def\csname LT4\endcsname{\color[rgb]{0,1,1}}%
      \expandafter\def\csname LT5\endcsname{\color[rgb]{1,1,0}}%
      \expandafter\def\csname LT6\endcsname{\color[rgb]{0,0,0}}%
      \expandafter\def\csname LT7\endcsname{\color[rgb]{1,0.3,0}}%
      \expandafter\def\csname LT8\endcsname{\color[rgb]{0.5,0.5,0.5}}%
    \else
      \def\colorrgb#1{\color{black}}%
      \def\colorgray#1{\color[gray]{#1}}%
      \expandafter\def\csname LTw\endcsname{\color{white}}%
      \expandafter\def\csname LTb\endcsname{\color{black}}%
      \expandafter\def\csname LTa\endcsname{\color{black}}%
      \expandafter\def\csname LT0\endcsname{\color{black}}%
      \expandafter\def\csname LT1\endcsname{\color{black}}%
      \expandafter\def\csname LT2\endcsname{\color{black}}%
      \expandafter\def\csname LT3\endcsname{\color{black}}%
      \expandafter\def\csname LT4\endcsname{\color{black}}%
      \expandafter\def\csname LT5\endcsname{\color{black}}%
      \expandafter\def\csname LT6\endcsname{\color{black}}%
      \expandafter\def\csname LT7\endcsname{\color{black}}%
      \expandafter\def\csname LT8\endcsname{\color{black}}%
    \fi
  \fi
    \setlength{\unitlength}{0.0500bp}%
    \ifx\gptboxheight\undefined%
      \newlength{\gptboxheight}%
      \newlength{\gptboxwidth}%
      \newsavebox{\gptboxtext}%
    \fi%
    \setlength{\fboxrule}{0.5pt}%
    \setlength{\fboxsep}{1pt}%
\begin{picture}(7200.00,5040.00)%
    \gplgaddtomacro\gplbacktext{%
      \csname LTb\endcsname%
      \put(1372,896){\makebox(0,0)[r]{\strut{}1e-06}}%
      \csname LTb\endcsname%
      \put(1372,1530){\makebox(0,0)[r]{\strut{}1e-05}}%
      \csname LTb\endcsname%
      \put(1372,2165){\makebox(0,0)[r]{\strut{}1e-04}}%
      \csname LTb\endcsname%
      \put(1372,2800){\makebox(0,0)[r]{\strut{}1e-03}}%
      \csname LTb\endcsname%
      \put(1372,3434){\makebox(0,0)[r]{\strut{}1e-02}}%
      \csname LTb\endcsname%
      \put(1372,4069){\makebox(0,0)[r]{\strut{}1e-01}}%
      \csname LTb\endcsname%
      \put(1372,4703){\makebox(0,0)[r]{\strut{}1e+00}}%
      \csname LTb\endcsname%
      \put(1540,616){\makebox(0,0){\strut{}$0$}}%
      \csname LTb\endcsname%
      \put(2056,616){\makebox(0,0){\strut{}$10$}}%
      \csname LTb\endcsname%
      \put(2571,616){\makebox(0,0){\strut{}$20$}}%
      \csname LTb\endcsname%
      \put(3087,616){\makebox(0,0){\strut{}$30$}}%
      \csname LTb\endcsname%
      \put(3602,616){\makebox(0,0){\strut{}$40$}}%
      \csname LTb\endcsname%
      \put(4118,616){\makebox(0,0){\strut{}$50$}}%
      \csname LTb\endcsname%
      \put(4633,616){\makebox(0,0){\strut{}$60$}}%
      \csname LTb\endcsname%
      \put(5149,616){\makebox(0,0){\strut{}$70$}}%
      \csname LTb\endcsname%
      \put(5664,616){\makebox(0,0){\strut{}$80$}}%
      \csname LTb\endcsname%
      \put(6180,616){\makebox(0,0){\strut{}$90$}}%
      \csname LTb\endcsname%
      \put(6695,616){\makebox(0,0){\strut{}$100$}}%
    }%
    \gplgaddtomacro\gplfronttext{%
      \csname LTb\endcsname%
      \put(224,2799){\rotatebox{-270}{\makebox(0,0){\strut{}norm $\|\bfE_{l}\|_2$ of non-Markovian operator}}}%
      \put(4117,196){\makebox(0,0){\strut{}lag $l$}}%
      \csname LTb\endcsname%
      \put(5288,4500){\makebox(0,0)[r]{\strut{}20\% observed}}%
      \csname LTb\endcsname%
      \put(5288,4220){\makebox(0,0)[r]{\strut{}40\% observed}}%
      \csname LTb\endcsname%
      \put(5288,3940){\makebox(0,0)[r]{\strut{}60\% observed}}%
      \csname LTb\endcsname%
      \put(5288,3660){\makebox(0,0)[r]{\strut{}80\% observed}}%
    }%
    \gplbacktext
    \put(0,0){\includegraphics{ConvDiff_OpNormDecay_Stage_dim10State}}%
    \gplfronttext
  \end{picture}%
\endgroup

%% file: figures/ConvDiff_OpNormDecay_Stage_dim10Out.tex
\begingroup
  \makeatletter
  \providecommand\color[2][]{%
    \GenericError{(gnuplot) \space\space\space\@spaces}{%
      Package color not loaded in conjunction with
      terminal option `colourtext'%
    }{See the gnuplot documentation for explanation.%
    }{Either use 'blacktext' in gnuplot or load the package
      color.sty in LaTeX.}%
    \renewcommand\color[2][]{}%
  }%
  \providecommand\includegraphics[2][]{%
    \GenericError{(gnuplot) \space\space\space\@spaces}{%
      Package graphicx or graphics not loaded%
    }{See the gnuplot documentation for explanation.%
    }{The gnuplot epslatex terminal needs graphicx.sty or graphics.sty.}%
    \renewcommand\includegraphics[2][]{}%
  }%
  \providecommand\rotatebox[2]{#2}%
  \@ifundefined{ifGPcolor}{%
    \newif\ifGPcolor
    \GPcolortrue
  }{}%
  \@ifundefined{ifGPblacktext}{%
    \newif\ifGPblacktext
    \GPblacktexttrue
  }{}%
  \let\gplgaddtomacro\g@addto@macro
  \gdef\gplbacktext{}%
  \gdef\gplfronttext{}%
  \makeatother
  \ifGPblacktext
    \def\colorrgb#1{}%
    \def\colorgray#1{}%
  \else
    \ifGPcolor
      \def\colorrgb#1{\color[rgb]{#1}}%
      \def\colorgray#1{\color[gray]{#1}}%
      \expandafter\def\csname LTw\endcsname{\color{white}}%
      \expandafter\def\csname LTb\endcsname{\color{black}}%
      \expandafter\def\csname LTa\endcsname{\color{black}}%
      \expandafter\def\csname LT0\endcsname{\color[rgb]{1,0,0}}%
      \expandafter\def\csname LT1\endcsname{\color[rgb]{0,1,0}}%
      \expandafter\def\csname LT2\endcsname{\color[rgb]{0,0,1}}%
      \expandafter\def\csname LT3\endcsname{\color[rgb]{1,0,1}}%
      \expandafter\def\csname LT4\endcsname{\color[rgb]{0,1,1}}%
      \expandafter\def\csname LT5\endcsname{\color[rgb]{1,1,0}}%
      \expandafter\def\csname LT6\endcsname{\color[rgb]{0,0,0}}%
      \expandafter\def\csname LT7\endcsname{\color[rgb]{1,0.3,0}}%
      \expandafter\def\csname LT8\endcsname{\color[rgb]{0.5,0.5,0.5}}%
    \else
      \def\colorrgb#1{\color{black}}%
      \def\colorgray#1{\color[gray]{#1}}%
      \expandafter\def\csname LTw\endcsname{\color{white}}%
      \expandafter\def\csname LTb\endcsname{\color{black}}%
      \expandafter\def\csname LTa\endcsname{\color{black}}%
      \expandafter\def\csname LT0\endcsname{\color{black}}%
      \expandafter\def\csname LT1\endcsname{\color{black}}%
      \expandafter\def\csname LT2\endcsname{\color{black}}%
      \expandafter\def\csname LT3\endcsname{\color{black}}%
      \expandafter\def\csname LT4\endcsname{\color{black}}%
      \expandafter\def\csname LT5\endcsname{\color{black}}%
      \expandafter\def\csname LT6\endcsname{\color{black}}%
      \expandafter\def\csname LT7\endcsname{\color{black}}%
      \expandafter\def\csname LT8\endcsname{\color{black}}%
    \fi
  \fi
    \setlength{\unitlength}{0.0500bp}%
    \ifx\gptboxheight\undefined%
      \newlength{\gptboxheight}%
      \newlength{\gptboxwidth}%
      \newsavebox{\gptboxtext}%
    \fi%
    \setlength{\fboxrule}{0.5pt}%
    \setlength{\fboxsep}{1pt}%
\begin{picture}(7200.00,5040.00)%
    \gplgaddtomacro\gplbacktext{%
      \csname LTb\endcsname%
      \put(1372,896){\makebox(0,0)[r]{\strut{}1e-06}}%
      \csname LTb\endcsname%
      \put(1372,1530){\makebox(0,0)[r]{\strut{}1e-05}}%
      \csname LTb\endcsname%
      \put(1372,2165){\makebox(0,0)[r]{\strut{}1e-04}}%
      \csname LTb\endcsname%
      \put(1372,2800){\makebox(0,0)[r]{\strut{}1e-03}}%
      \csname LTb\endcsname%
      \put(1372,3434){\makebox(0,0)[r]{\strut{}1e-02}}%
      \csname LTb\endcsname%
      \put(1372,4069){\makebox(0,0)[r]{\strut{}1e-01}}%
      \csname LTb\endcsname%
      \put(1372,4703){\makebox(0,0)[r]{\strut{}1e+00}}%
      \csname LTb\endcsname%
      \put(1540,616){\makebox(0,0){\strut{}$0$}}%
      \csname LTb\endcsname%
      \put(2056,616){\makebox(0,0){\strut{}$10$}}%
      \csname LTb\endcsname%
      \put(2571,616){\makebox(0,0){\strut{}$20$}}%
      \csname LTb\endcsname%
      \put(3087,616){\makebox(0,0){\strut{}$30$}}%
      \csname LTb\endcsname%
      \put(3602,616){\makebox(0,0){\strut{}$40$}}%
      \csname LTb\endcsname%
      \put(4118,616){\makebox(0,0){\strut{}$50$}}%
      \csname LTb\endcsname%
      \put(4633,616){\makebox(0,0){\strut{}$60$}}%
      \csname LTb\endcsname%
      \put(5149,616){\makebox(0,0){\strut{}$70$}}%
      \csname LTb\endcsname%
      \put(5664,616){\makebox(0,0){\strut{}$80$}}%
      \csname LTb\endcsname%
      \put(6180,616){\makebox(0,0){\strut{}$90$}}%
      \csname LTb\endcsname%
      \put(6695,616){\makebox(0,0){\strut{}$100$}}%
    }%
    \gplgaddtomacro\gplfronttext{%
      \csname LTb\endcsname%
      \put(224,2799){\rotatebox{-270}{\makebox(0,0){\strut{}norm $\|\bfF_{l}\|_2$ of non-Markovian operator}}}%
      \put(4117,196){\makebox(0,0){\strut{}lag $l$}}%
      \csname LTb\endcsname%
      \put(5288,4500){\makebox(0,0)[r]{\strut{}20\% observed}}%
      \csname LTb\endcsname%
      \put(5288,4220){\makebox(0,0)[r]{\strut{}40\% observed}}%
      \csname LTb\endcsname%
      \put(5288,3940){\makebox(0,0)[r]{\strut{}60\% observed}}%
      \csname LTb\endcsname%
      \put(5288,3660){\makebox(0,0)[r]{\strut{}80\% observed}}%
    }%
    \gplbacktext
    \put(0,0){\includegraphics{ConvDiff_OpNormDecay_Stage_dim10Out}}%
    \gplfronttext
  \end{picture}%
\endgroup

%% file: figures/ConvDiff_OpNorm_StageVsBatchDetailed_State.tex
\begingroup
  \makeatletter
  \providecommand\color[2][]{%
    \GenericError{(gnuplot) \space\space\space\@spaces}{%
      Package color not loaded in conjunction with
      terminal option `colourtext'%
    }{See the gnuplot documentation for explanation.%
    }{Either use 'blacktext' in gnuplot or load the package
      color.sty in LaTeX.}%
    \renewcommand\color[2][]{}%
  }%
  \providecommand\includegraphics[2][]{%
    \GenericError{(gnuplot) \space\space\space\@spaces}{%
      Package graphicx or graphics not loaded%
    }{See the gnuplot documentation for explanation.%
    }{The gnuplot epslatex terminal needs graphicx.sty or graphics.sty.}%
    \renewcommand\includegraphics[2][]{}%
  }%
  \providecommand\rotatebox[2]{#2}%
  \@ifundefined{ifGPcolor}{%
    \newif\ifGPcolor
    \GPcolortrue
  }{}%
  \@ifundefined{ifGPblacktext}{%
    \newif\ifGPblacktext
    \GPblacktexttrue
  }{}%
  \let\gplgaddtomacro\g@addto@macro
  \gdef\gplbacktext{}%
  \gdef\gplfronttext{}%
  \makeatother
  \ifGPblacktext
    \def\colorrgb#1{}%
    \def\colorgray#1{}%
  \else
    \ifGPcolor
      \def\colorrgb#1{\color[rgb]{#1}}%
      \def\colorgray#1{\color[gray]{#1}}%
      \expandafter\def\csname LTw\endcsname{\color{white}}%
      \expandafter\def\csname LTb\endcsname{\color{black}}%
      \expandafter\def\csname LTa\endcsname{\color{black}}%
      \expandafter\def\csname LT0\endcsname{\color[rgb]{1,0,0}}%
      \expandafter\def\csname LT1\endcsname{\color[rgb]{0,1,0}}%
      \expandafter\def\csname LT2\endcsname{\color[rgb]{0,0,1}}%
      \expandafter\def\csname LT3\endcsname{\color[rgb]{1,0,1}}%
      \expandafter\def\csname LT4\endcsname{\color[rgb]{0,1,1}}%
      \expandafter\def\csname LT5\endcsname{\color[rgb]{1,1,0}}%
      \expandafter\def\csname LT6\endcsname{\color[rgb]{0,0,0}}%
      \expandafter\def\csname LT7\endcsname{\color[rgb]{1,0.3,0}}%
      \expandafter\def\csname LT8\endcsname{\color[rgb]{0.5,0.5,0.5}}%
    \else
      \def\colorrgb#1{\color{black}}%
      \def\colorgray#1{\color[gray]{#1}}%
      \expandafter\def\csname LTw\endcsname{\color{white}}%
      \expandafter\def\csname LTb\endcsname{\color{black}}%
      \expandafter\def\csname LTa\endcsname{\color{black}}%
      \expandafter\def\csname LT0\endcsname{\color{black}}%
      \expandafter\def\csname LT1\endcsname{\color{black}}%
      \expandafter\def\csname LT2\endcsname{\color{black}}%
      \expandafter\def\csname LT3\endcsname{\color{black}}%
      \expandafter\def\csname LT4\endcsname{\color{black}}%
      \expandafter\def\csname LT5\endcsname{\color{black}}%
      \expandafter\def\csname LT6\endcsname{\color{black}}%
      \expandafter\def\csname LT7\endcsname{\color{black}}%
      \expandafter\def\csname LT8\endcsname{\color{black}}%
    \fi
  \fi
    \setlength{\unitlength}{0.0500bp}%
    \ifx\gptboxheight\undefined%
      \newlength{\gptboxheight}%
      \newlength{\gptboxwidth}%
      \newsavebox{\gptboxtext}%
    \fi%
    \setlength{\fboxrule}{0.5pt}%
    \setlength{\fboxsep}{1pt}%
\begin{picture}(7200.00,5040.00)%
    \gplgaddtomacro\gplbacktext{%
      \csname LTb\endcsname%
      \put(1372,896){\makebox(0,0)[r]{\strut{}1e-05}}%
      \csname LTb\endcsname%
      \put(1372,1657){\makebox(0,0)[r]{\strut{}1e-04}}%
      \csname LTb\endcsname%
      \put(1372,2419){\makebox(0,0)[r]{\strut{}1e-03}}%
      \csname LTb\endcsname%
      \put(1372,3180){\makebox(0,0)[r]{\strut{}1e-02}}%
      \csname LTb\endcsname%
      \put(1372,3942){\makebox(0,0)[r]{\strut{}1e-01}}%
      \csname LTb\endcsname%
      \put(1372,4703){\makebox(0,0)[r]{\strut{}1e+00}}%
      \csname LTb\endcsname%
      \put(1540,616){\makebox(0,0){\strut{}$0$}}%
      \csname LTb\endcsname%
      \put(2056,616){\makebox(0,0){\strut{}$10$}}%
      \csname LTb\endcsname%
      \put(2571,616){\makebox(0,0){\strut{}$20$}}%
      \csname LTb\endcsname%
      \put(3087,616){\makebox(0,0){\strut{}$30$}}%
      \csname LTb\endcsname%
      \put(3602,616){\makebox(0,0){\strut{}$40$}}%
      \csname LTb\endcsname%
      \put(4118,616){\makebox(0,0){\strut{}$50$}}%
      \csname LTb\endcsname%
      \put(4633,616){\makebox(0,0){\strut{}$60$}}%
      \csname LTb\endcsname%
      \put(5149,616){\makebox(0,0){\strut{}$70$}}%
      \csname LTb\endcsname%
      \put(5664,616){\makebox(0,0){\strut{}$80$}}%
      \csname LTb\endcsname%
      \put(6180,616){\makebox(0,0){\strut{}$90$}}%
      \csname LTb\endcsname%
      \put(6695,616){\makebox(0,0){\strut{}$100$}}%
    }%
    \gplgaddtomacro\gplfronttext{%
      \csname LTb\endcsname%
      \put(224,2799){\rotatebox{-270}{\makebox(0,0){\strut{}norm of non-Markovian state operator}}}%
      \put(4117,196){\makebox(0,0){\strut{}lag $l$}}%
      \csname LTb\endcsname%
      \put(5792,4500){\makebox(0,0)[r]{\strut{}full model}}%
      \csname LTb\endcsname%
      \put(5792,4220){\makebox(0,0)[r]{\strut{}Stage $\Nreproj=25, \Treproj=102$}}%
      \csname LTb\endcsname%
      \put(5792,3940){\makebox(0,0)[r]{\strut{}Batch $\Nreproj=12, \Treproj=250$}}%
      \csname LTb\endcsname%
      \put(5792,3660){\makebox(0,0)[r]{\strut{}Batch $\Nreproj=12, \Treproj=1000$}}%
      \csname LTb\endcsname%
      \put(5792,3380){\makebox(0,0)[r]{\strut{}Batch $\Nreproj=12, \Treproj=10000$}}%
    }%
    \gplbacktext
    \put(0,0){\includegraphics{ConvDiff_OpNorm_StageVsBatchDetailed_State}}%
    \gplfronttext
  \end{picture}%
\endgroup

%% file: figures/ConvDiff_OpNorm_StageVsBatchDetailed_Out.tex
\begingroup
  \makeatletter
  \providecommand\color[2][]{%
    \GenericError{(gnuplot) \space\space\space\@spaces}{%
      Package color not loaded in conjunction with
      terminal option `colourtext'%
    }{See the gnuplot documentation for explanation.%
    }{Either use 'blacktext' in gnuplot or load the package
      color.sty in LaTeX.}%
    \renewcommand\color[2][]{}%
  }%
  \providecommand\includegraphics[2][]{%
    \GenericError{(gnuplot) \space\space\space\@spaces}{%
      Package graphicx or graphics not loaded%
    }{See the gnuplot documentation for explanation.%
    }{The gnuplot epslatex terminal needs graphicx.sty or graphics.sty.}%
    \renewcommand\includegraphics[2][]{}%
  }%
  \providecommand\rotatebox[2]{#2}%
  \@ifundefined{ifGPcolor}{%
    \newif\ifGPcolor
    \GPcolortrue
  }{}%
  \@ifundefined{ifGPblacktext}{%
    \newif\ifGPblacktext
    \GPblacktexttrue
  }{}%
  \let\gplgaddtomacro\g@addto@macro
  \gdef\gplbacktext{}%
  \gdef\gplfronttext{}%
  \makeatother
  \ifGPblacktext
    \def\colorrgb#1{}%
    \def\colorgray#1{}%
  \else
    \ifGPcolor
      \def\colorrgb#1{\color[rgb]{#1}}%
      \def\colorgray#1{\color[gray]{#1}}%
      \expandafter\def\csname LTw\endcsname{\color{white}}%
      \expandafter\def\csname LTb\endcsname{\color{black}}%
      \expandafter\def\csname LTa\endcsname{\color{black}}%
      \expandafter\def\csname LT0\endcsname{\color[rgb]{1,0,0}}%
      \expandafter\def\csname LT1\endcsname{\color[rgb]{0,1,0}}%
      \expandafter\def\csname LT2\endcsname{\color[rgb]{0,0,1}}%
      \expandafter\def\csname LT3\endcsname{\color[rgb]{1,0,1}}%
      \expandafter\def\csname LT4\endcsname{\color[rgb]{0,1,1}}%
      \expandafter\def\csname LT5\endcsname{\color[rgb]{1,1,0}}%
      \expandafter\def\csname LT6\endcsname{\color[rgb]{0,0,0}}%
      \expandafter\def\csname LT7\endcsname{\color[rgb]{1,0.3,0}}%
      \expandafter\def\csname LT8\endcsname{\color[rgb]{0.5,0.5,0.5}}%
    \else
      \def\colorrgb#1{\color{black}}%
      \def\colorgray#1{\color[gray]{#1}}%
      \expandafter\def\csname LTw\endcsname{\color{white}}%
      \expandafter\def\csname LTb\endcsname{\color{black}}%
      \expandafter\def\csname LTa\endcsname{\color{black}}%
      \expandafter\def\csname LT0\endcsname{\color{black}}%
      \expandafter\def\csname LT1\endcsname{\color{black}}%
      \expandafter\def\csname LT2\endcsname{\color{black}}%
      \expandafter\def\csname LT3\endcsname{\color{black}}%
      \expandafter\def\csname LT4\endcsname{\color{black}}%
      \expandafter\def\csname LT5\endcsname{\color{black}}%
      \expandafter\def\csname LT6\endcsname{\color{black}}%
      \expandafter\def\csname LT7\endcsname{\color{black}}%
      \expandafter\def\csname LT8\endcsname{\color{black}}%
    \fi
  \fi
    \setlength{\unitlength}{0.0500bp}%
    \ifx\gptboxheight\undefined%
      \newlength{\gptboxheight}%
      \newlength{\gptboxwidth}%
      \newsavebox{\gptboxtext}%
    \fi%
    \setlength{\fboxrule}{0.5pt}%
    \setlength{\fboxsep}{1pt}%
\begin{picture}(7200.00,5040.00)%
    \gplgaddtomacro\gplbacktext{%
      \csname LTb\endcsname%
      \put(1372,896){\makebox(0,0)[r]{\strut{}1e-06}}%
      \csname LTb\endcsname%
      \put(1372,2165){\makebox(0,0)[r]{\strut{}1e-05}}%
      \csname LTb\endcsname%
      \put(1372,3434){\makebox(0,0)[r]{\strut{}1e-04}}%
      \csname LTb\endcsname%
      \put(1372,4703){\makebox(0,0)[r]{\strut{}1e-03}}%
      \csname LTb\endcsname%
      \put(1540,616){\makebox(0,0){\strut{}$0$}}%
      \csname LTb\endcsname%
      \put(2056,616){\makebox(0,0){\strut{}$10$}}%
      \csname LTb\endcsname%
      \put(2571,616){\makebox(0,0){\strut{}$20$}}%
      \csname LTb\endcsname%
      \put(3087,616){\makebox(0,0){\strut{}$30$}}%
      \csname LTb\endcsname%
      \put(3602,616){\makebox(0,0){\strut{}$40$}}%
      \csname LTb\endcsname%
      \put(4118,616){\makebox(0,0){\strut{}$50$}}%
      \csname LTb\endcsname%
      \put(4633,616){\makebox(0,0){\strut{}$60$}}%
      \csname LTb\endcsname%
      \put(5149,616){\makebox(0,0){\strut{}$70$}}%
      \csname LTb\endcsname%
      \put(5664,616){\makebox(0,0){\strut{}$80$}}%
      \csname LTb\endcsname%
      \put(6180,616){\makebox(0,0){\strut{}$90$}}%
      \csname LTb\endcsname%
      \put(6695,616){\makebox(0,0){\strut{}$100$}}%
    }%
    \gplgaddtomacro\gplfronttext{%
      \csname LTb\endcsname%
      \put(224,2799){\rotatebox{-270}{\makebox(0,0){\strut{}norm of non-Markovian input operator}}}%
      \put(4117,196){\makebox(0,0){\strut{}lag $l$}}%
      \csname LTb\endcsname%
      \put(5792,4500){\makebox(0,0)[r]{\strut{}full model}}%
      \csname LTb\endcsname%
      \put(5792,4220){\makebox(0,0)[r]{\strut{}Stage $\Nreproj=25, \Treproj=102$}}%
      \csname LTb\endcsname%
      \put(5792,3940){\makebox(0,0)[r]{\strut{}Batch $\Nreproj=12, \Treproj=250$}}%
      \csname LTb\endcsname%
      \put(5792,3660){\makebox(0,0)[r]{\strut{}Batch $\Nreproj=12, \Treproj=1000$}}%
      \csname LTb\endcsname%
      \put(5792,3380){\makebox(0,0)[r]{\strut{}Batch $\Nreproj=12, \Treproj=10000$}}%
    }%
    \gplbacktext
    \put(0,0){\includegraphics{ConvDiff_OpNorm_StageVsBatchDetailed_Out}}%
    \gplfronttext
  \end{picture}%
\endgroup

%% file: figures/ConvDiff_ProjErr_PO20.tex
\begingroup
  \makeatletter
  \providecommand\color[2][]{%
    \GenericError{(gnuplot) \space\space\space\@spaces}{%
      Package color not loaded in conjunction with
      terminal option `colourtext'%
    }{See the gnuplot documentation for explanation.%
    }{Either use 'blacktext' in gnuplot or load the package
      color.sty in LaTeX.}%
    \renewcommand\color[2][]{}%
  }%
  \providecommand\includegraphics[2][]{%
    \GenericError{(gnuplot) \space\space\space\@spaces}{%
      Package graphicx or graphics not loaded%
    }{See the gnuplot documentation for explanation.%
    }{The gnuplot epslatex terminal needs graphicx.sty or graphics.sty.}%
    \renewcommand\includegraphics[2][]{}%
  }%
  \providecommand\rotatebox[2]{#2}%
  \@ifundefined{ifGPcolor}{%
    \newif\ifGPcolor
    \GPcolortrue
  }{}%
  \@ifundefined{ifGPblacktext}{%
    \newif\ifGPblacktext
    \GPblacktexttrue
  }{}%
  \let\gplgaddtomacro\g@addto@macro
  \gdef\gplbacktext{}%
  \gdef\gplfronttext{}%
  \makeatother
  \ifGPblacktext
    \def\colorrgb#1{}%
    \def\colorgray#1{}%
  \else
    \ifGPcolor
      \def\colorrgb#1{\color[rgb]{#1}}%
      \def\colorgray#1{\color[gray]{#1}}%
      \expandafter\def\csname LTw\endcsname{\color{white}}%
      \expandafter\def\csname LTb\endcsname{\color{black}}%
      \expandafter\def\csname LTa\endcsname{\color{black}}%
      \expandafter\def\csname LT0\endcsname{\color[rgb]{1,0,0}}%
      \expandafter\def\csname LT1\endcsname{\color[rgb]{0,1,0}}%
      \expandafter\def\csname LT2\endcsname{\color[rgb]{0,0,1}}%
      \expandafter\def\csname LT3\endcsname{\color[rgb]{1,0,1}}%
      \expandafter\def\csname LT4\endcsname{\color[rgb]{0,1,1}}%
      \expandafter\def\csname LT5\endcsname{\color[rgb]{1,1,0}}%
      \expandafter\def\csname LT6\endcsname{\color[rgb]{0,0,0}}%
      \expandafter\def\csname LT7\endcsname{\color[rgb]{1,0.3,0}}%
      \expandafter\def\csname LT8\endcsname{\color[rgb]{0.5,0.5,0.5}}%
    \else
      \def\colorrgb#1{\color{black}}%
      \def\colorgray#1{\color[gray]{#1}}%
      \expandafter\def\csname LTw\endcsname{\color{white}}%
      \expandafter\def\csname LTb\endcsname{\color{black}}%
      \expandafter\def\csname LTa\endcsname{\color{black}}%
      \expandafter\def\csname LT0\endcsname{\color{black}}%
      \expandafter\def\csname LT1\endcsname{\color{black}}%
      \expandafter\def\csname LT2\endcsname{\color{black}}%
      \expandafter\def\csname LT3\endcsname{\color{black}}%
      \expandafter\def\csname LT4\endcsname{\color{black}}%
      \expandafter\def\csname LT5\endcsname{\color{black}}%
      \expandafter\def\csname LT6\endcsname{\color{black}}%
      \expandafter\def\csname LT7\endcsname{\color{black}}%
      \expandafter\def\csname LT8\endcsname{\color{black}}%
    \fi
  \fi
    \setlength{\unitlength}{0.0500bp}%
    \ifx\gptboxheight\undefined%
      \newlength{\gptboxheight}%
      \newlength{\gptboxwidth}%
      \newsavebox{\gptboxtext}%
    \fi%
    \setlength{\fboxrule}{0.5pt}%
    \setlength{\fboxsep}{1pt}%
\begin{picture}(7200.00,5040.00)%
    \gplgaddtomacro\gplbacktext{%
      \csname LTb\endcsname%
      \put(1372,1618){\makebox(0,0)[r]{\strut{}1e-01}}%
      \csname LTb\endcsname%
      \put(1372,3434){\makebox(0,0)[r]{\strut{}1e+00}}%
      \csname LTb\endcsname%
      \put(1540,616){\makebox(0,0){\strut{}$0$}}%
      \csname LTb\endcsname%
      \put(2571,616){\makebox(0,0){\strut{}$20$}}%
      \csname LTb\endcsname%
      \put(3602,616){\makebox(0,0){\strut{}$40$}}%
      \csname LTb\endcsname%
      \put(4633,616){\makebox(0,0){\strut{}$60$}}%
      \csname LTb\endcsname%
      \put(5664,616){\makebox(0,0){\strut{}$80$}}%
      \csname LTb\endcsname%
      \put(6695,616){\makebox(0,0){\strut{}$100$}}%
    }%
    \gplgaddtomacro\gplfronttext{%
      \csname LTb\endcsname%
      \put(224,2799){\rotatebox{-270}{\makebox(0,0){\strut{}projection versus observation error}}}%
      \put(4117,196){\makebox(0,0){\strut{}lag $L$ of non-Markovian term}}%
      \csname LTb\endcsname%
      \put(5624,4500){\makebox(0,0)[r]{\strut{}$n=4$, observation error}}%
      \csname LTb\endcsname%
      \put(5624,4220){\makebox(0,0)[r]{\strut{}$n=4$, projection error}}%
      \csname LTb\endcsname%
      \put(5624,3940){\makebox(0,0)[r]{\strut{}$n=10$, observation error}}%
      \csname LTb\endcsname%
      \put(5624,3660){\makebox(0,0)[r]{\strut{}$n=10$, projection error}}%
    }%
    \gplbacktext
    \put(0,0){\includegraphics{ConvDiff_ProjErr_PO20}}%
    \gplfronttext
  \end{picture}%
\endgroup

%% file: figures/ConvDiff_ProjErr_PO80.tex
\begingroup
  \makeatletter
  \providecommand\color[2][]{%
    \GenericError{(gnuplot) \space\space\space\@spaces}{%
      Package color not loaded in conjunction with
      terminal option `colourtext'%
    }{See the gnuplot documentation for explanation.%
    }{Either use 'blacktext' in gnuplot or load the package
      color.sty in LaTeX.}%
    \renewcommand\color[2][]{}%
  }%
  \providecommand\includegraphics[2][]{%
    \GenericError{(gnuplot) \space\space\space\@spaces}{%
      Package graphicx or graphics not loaded%
    }{See the gnuplot documentation for explanation.%
    }{The gnuplot epslatex terminal needs graphicx.sty or graphics.sty.}%
    \renewcommand\includegraphics[2][]{}%
  }%
  \providecommand\rotatebox[2]{#2}%
  \@ifundefined{ifGPcolor}{%
    \newif\ifGPcolor
    \GPcolortrue
  }{}%
  \@ifundefined{ifGPblacktext}{%
    \newif\ifGPblacktext
    \GPblacktexttrue
  }{}%
  \let\gplgaddtomacro\g@addto@macro
  \gdef\gplbacktext{}%
  \gdef\gplfronttext{}%
  \makeatother
  \ifGPblacktext
    \def\colorrgb#1{}%
    \def\colorgray#1{}%
  \else
    \ifGPcolor
      \def\colorrgb#1{\color[rgb]{#1}}%
      \def\colorgray#1{\color[gray]{#1}}%
      \expandafter\def\csname LTw\endcsname{\color{white}}%
      \expandafter\def\csname LTb\endcsname{\color{black}}%
      \expandafter\def\csname LTa\endcsname{\color{black}}%
      \expandafter\def\csname LT0\endcsname{\color[rgb]{1,0,0}}%
      \expandafter\def\csname LT1\endcsname{\color[rgb]{0,1,0}}%
      \expandafter\def\csname LT2\endcsname{\color[rgb]{0,0,1}}%
      \expandafter\def\csname LT3\endcsname{\color[rgb]{1,0,1}}%
      \expandafter\def\csname LT4\endcsname{\color[rgb]{0,1,1}}%
      \expandafter\def\csname LT5\endcsname{\color[rgb]{1,1,0}}%
      \expandafter\def\csname LT6\endcsname{\color[rgb]{0,0,0}}%
      \expandafter\def\csname LT7\endcsname{\color[rgb]{1,0.3,0}}%
      \expandafter\def\csname LT8\endcsname{\color[rgb]{0.5,0.5,0.5}}%
    \else
      \def\colorrgb#1{\color{black}}%
      \def\colorgray#1{\color[gray]{#1}}%
      \expandafter\def\csname LTw\endcsname{\color{white}}%
      \expandafter\def\csname LTb\endcsname{\color{black}}%
      \expandafter\def\csname LTa\endcsname{\color{black}}%
      \expandafter\def\csname LT0\endcsname{\color{black}}%
      \expandafter\def\csname LT1\endcsname{\color{black}}%
      \expandafter\def\csname LT2\endcsname{\color{black}}%
      \expandafter\def\csname LT3\endcsname{\color{black}}%
      \expandafter\def\csname LT4\endcsname{\color{black}}%
      \expandafter\def\csname LT5\endcsname{\color{black}}%
      \expandafter\def\csname LT6\endcsname{\color{black}}%
      \expandafter\def\csname LT7\endcsname{\color{black}}%
      \expandafter\def\csname LT8\endcsname{\color{black}}%
    \fi
  \fi
    \setlength{\unitlength}{0.0500bp}%
    \ifx\gptboxheight\undefined%
      \newlength{\gptboxheight}%
      \newlength{\gptboxwidth}%
      \newsavebox{\gptboxtext}%
    \fi%
    \setlength{\fboxrule}{0.5pt}%
    \setlength{\fboxsep}{1pt}%
\begin{picture}(7200.00,5040.00)%
    \gplgaddtomacro\gplbacktext{%
      \csname LTb\endcsname%
      \put(1372,1618){\makebox(0,0)[r]{\strut{}1e-01}}%
      \csname LTb\endcsname%
      \put(1372,3434){\makebox(0,0)[r]{\strut{}1e+00}}%
      \csname LTb\endcsname%
      \put(1540,616){\makebox(0,0){\strut{}$0$}}%
      \csname LTb\endcsname%
      \put(2571,616){\makebox(0,0){\strut{}$20$}}%
      \csname LTb\endcsname%
      \put(3602,616){\makebox(0,0){\strut{}$40$}}%
      \csname LTb\endcsname%
      \put(4633,616){\makebox(0,0){\strut{}$60$}}%
      \csname LTb\endcsname%
      \put(5664,616){\makebox(0,0){\strut{}$80$}}%
      \csname LTb\endcsname%
      \put(6695,616){\makebox(0,0){\strut{}$100$}}%
    }%
    \gplgaddtomacro\gplfronttext{%
      \csname LTb\endcsname%
      \put(224,2799){\rotatebox{-270}{\makebox(0,0){\strut{}projection versus observation error}}}%
      \put(4117,196){\makebox(0,0){\strut{}lag $L$ of non-Markovian term}}%
      \csname LTb\endcsname%
      \put(5624,4500){\makebox(0,0)[r]{\strut{}$n=4$, observation error}}%
      \csname LTb\endcsname%
      \put(5624,4220){\makebox(0,0)[r]{\strut{}$n=4$, projection error}}%
      \csname LTb\endcsname%
      \put(5624,3940){\makebox(0,0)[r]{\strut{}$n=10$, observation error}}%
      \csname LTb\endcsname%
      \put(5624,3660){\makebox(0,0)[r]{\strut{}$n=10$, projection error}}%
    }%
    \gplbacktext
    \put(0,0){\includegraphics{ConvDiff_ProjErr_PO80}}%
    \gplfronttext
  \end{picture}%
\endgroup

%% file: figures/ConvDiff_ProjErrStageVSBatch_dim8.tex
\begingroup
  \makeatletter
  \providecommand\color[2][]{%
    \GenericError{(gnuplot) \space\space\space\@spaces}{%
      Package color not loaded in conjunction with
      terminal option `colourtext'%
    }{See the gnuplot documentation for explanation.%
    }{Either use 'blacktext' in gnuplot or load the package
      color.sty in LaTeX.}%
    \renewcommand\color[2][]{}%
  }%
  \providecommand\includegraphics[2][]{%
    \GenericError{(gnuplot) \space\space\space\@spaces}{%
      Package graphicx or graphics not loaded%
    }{See the gnuplot documentation for explanation.%
    }{The gnuplot epslatex terminal needs graphicx.sty or graphics.sty.}%
    \renewcommand\includegraphics[2][]{}%
  }%
  \providecommand\rotatebox[2]{#2}%
  \@ifundefined{ifGPcolor}{%
    \newif\ifGPcolor
    \GPcolortrue
  }{}%
  \@ifundefined{ifGPblacktext}{%
    \newif\ifGPblacktext
    \GPblacktexttrue
  }{}%
  \let\gplgaddtomacro\g@addto@macro
  \gdef\gplbacktext{}%
  \gdef\gplfronttext{}%
  \makeatother
  \ifGPblacktext
    \def\colorrgb#1{}%
    \def\colorgray#1{}%
  \else
    \ifGPcolor
      \def\colorrgb#1{\color[rgb]{#1}}%
      \def\colorgray#1{\color[gray]{#1}}%
      \expandafter\def\csname LTw\endcsname{\color{white}}%
      \expandafter\def\csname LTb\endcsname{\color{black}}%
      \expandafter\def\csname LTa\endcsname{\color{black}}%
      \expandafter\def\csname LT0\endcsname{\color[rgb]{1,0,0}}%
      \expandafter\def\csname LT1\endcsname{\color[rgb]{0,1,0}}%
      \expandafter\def\csname LT2\endcsname{\color[rgb]{0,0,1}}%
      \expandafter\def\csname LT3\endcsname{\color[rgb]{1,0,1}}%
      \expandafter\def\csname LT4\endcsname{\color[rgb]{0,1,1}}%
      \expandafter\def\csname LT5\endcsname{\color[rgb]{1,1,0}}%
      \expandafter\def\csname LT6\endcsname{\color[rgb]{0,0,0}}%
      \expandafter\def\csname LT7\endcsname{\color[rgb]{1,0.3,0}}%
      \expandafter\def\csname LT8\endcsname{\color[rgb]{0.5,0.5,0.5}}%
    \else
      \def\colorrgb#1{\color{black}}%
      \def\colorgray#1{\color[gray]{#1}}%
      \expandafter\def\csname LTw\endcsname{\color{white}}%
      \expandafter\def\csname LTb\endcsname{\color{black}}%
      \expandafter\def\csname LTa\endcsname{\color{black}}%
      \expandafter\def\csname LT0\endcsname{\color{black}}%
      \expandafter\def\csname LT1\endcsname{\color{black}}%
      \expandafter\def\csname LT2\endcsname{\color{black}}%
      \expandafter\def\csname LT3\endcsname{\color{black}}%
      \expandafter\def\csname LT4\endcsname{\color{black}}%
      \expandafter\def\csname LT5\endcsname{\color{black}}%
      \expandafter\def\csname LT6\endcsname{\color{black}}%
      \expandafter\def\csname LT7\endcsname{\color{black}}%
      \expandafter\def\csname LT8\endcsname{\color{black}}%
    \fi
  \fi
    \setlength{\unitlength}{0.0500bp}%
    \ifx\gptboxheight\undefined%
      \newlength{\gptboxheight}%
      \newlength{\gptboxwidth}%
      \newsavebox{\gptboxtext}%
    \fi%
    \setlength{\fboxrule}{0.5pt}%
    \setlength{\fboxsep}{1pt}%
\begin{picture}(7200.00,5040.00)%
    \gplgaddtomacro\gplbacktext{%
      \csname LTb\endcsname%
      \put(1372,896){\makebox(0,0)[r]{\strut{}1e-16}}%
      \csname LTb\endcsname%
      \put(1372,1363){\makebox(0,0)[r]{\strut{}1e-14}}%
      \csname LTb\endcsname%
      \put(1372,1830){\makebox(0,0)[r]{\strut{}1e-12}}%
      \csname LTb\endcsname%
      \put(1372,2297){\makebox(0,0)[r]{\strut{}1e-10}}%
      \csname LTb\endcsname%
      \put(1372,2764){\makebox(0,0)[r]{\strut{}1e-08}}%
      \csname LTb\endcsname%
      \put(1372,3231){\makebox(0,0)[r]{\strut{}1e-06}}%
      \csname LTb\endcsname%
      \put(1372,3699){\makebox(0,0)[r]{\strut{}1e-04}}%
      \csname LTb\endcsname%
      \put(1372,4166){\makebox(0,0)[r]{\strut{}1e-02}}%
      \csname LTb\endcsname%
      \put(1372,4633){\makebox(0,0)[r]{\strut{}1e+00}}%
      \csname LTb\endcsname%
      \put(1540,616){\makebox(0,0){\strut{}$20$}}%
      \csname LTb\endcsname%
      \put(2399,616){\makebox(0,0){\strut{}$30$}}%
      \csname LTb\endcsname%
      \put(3258,616){\makebox(0,0){\strut{}$40$}}%
      \csname LTb\endcsname%
      \put(4118,616){\makebox(0,0){\strut{}$50$}}%
      \csname LTb\endcsname%
      \put(4977,616){\makebox(0,0){\strut{}$60$}}%
      \csname LTb\endcsname%
      \put(5836,616){\makebox(0,0){\strut{}$70$}}%
      \csname LTb\endcsname%
      \put(6695,616){\makebox(0,0){\strut{}$80$}}%
    }%
    \gplgaddtomacro\gplfronttext{%
      \csname LTb\endcsname%
      \put(224,2799){\rotatebox{-270}{\makebox(0,0){\strut{}abs diff of proj and obs error}}}%
      \put(4117,196){\makebox(0,0){\strut{}$\%$ of observed state components}}%
      \csname LTb\endcsname%
      \put(5288,4500){\makebox(0,0)[r]{\strut{}Stage $\Nreproj = 25,\Treproj=102$ }}%
      \csname LTb\endcsname%
      \put(5288,4220){\makebox(0,0)[r]{\strut{}Batch $\Nreproj=12,\Treproj=250$}}%
      \csname LTb\endcsname%
      \put(5288,3940){\makebox(0,0)[r]{\strut{}Batch $\Nreproj=12,\Treproj=1000$}}%
      \csname LTb\endcsname%
      \put(5288,3660){\makebox(0,0)[r]{\strut{}Batch $\Nreproj=12,\Treproj=10000$}}%
    }%
    \gplbacktext
    \put(0,0){\includegraphics{ConvDiff_ProjErrStageVSBatch_dim8}}%
    \gplfronttext
  \end{picture}%
\endgroup

%% file: figures/ConvDiff_ProjErrStageVSBatch_dim10.tex
\begingroup
  \makeatletter
  \providecommand\color[2][]{%
    \GenericError{(gnuplot) \space\space\space\@spaces}{%
      Package color not loaded in conjunction with
      terminal option `colourtext'%
    }{See the gnuplot documentation for explanation.%
    }{Either use 'blacktext' in gnuplot or load the package
      color.sty in LaTeX.}%
    \renewcommand\color[2][]{}%
  }%
  \providecommand\includegraphics[2][]{%
    \GenericError{(gnuplot) \space\space\space\@spaces}{%
      Package graphicx or graphics not loaded%
    }{See the gnuplot documentation for explanation.%
    }{The gnuplot epslatex terminal needs graphicx.sty or graphics.sty.}%
    \renewcommand\includegraphics[2][]{}%
  }%
  \providecommand\rotatebox[2]{#2}%
  \@ifundefined{ifGPcolor}{%
    \newif\ifGPcolor
    \GPcolortrue
  }{}%
  \@ifundefined{ifGPblacktext}{%
    \newif\ifGPblacktext
    \GPblacktexttrue
  }{}%
  \let\gplgaddtomacro\g@addto@macro
  \gdef\gplbacktext{}%
  \gdef\gplfronttext{}%
  \makeatother
  \ifGPblacktext
    \def\colorrgb#1{}%
    \def\colorgray#1{}%
  \else
    \ifGPcolor
      \def\colorrgb#1{\color[rgb]{#1}}%
      \def\colorgray#1{\color[gray]{#1}}%
      \expandafter\def\csname LTw\endcsname{\color{white}}%
      \expandafter\def\csname LTb\endcsname{\color{black}}%
      \expandafter\def\csname LTa\endcsname{\color{black}}%
      \expandafter\def\csname LT0\endcsname{\color[rgb]{1,0,0}}%
      \expandafter\def\csname LT1\endcsname{\color[rgb]{0,1,0}}%
      \expandafter\def\csname LT2\endcsname{\color[rgb]{0,0,1}}%
      \expandafter\def\csname LT3\endcsname{\color[rgb]{1,0,1}}%
      \expandafter\def\csname LT4\endcsname{\color[rgb]{0,1,1}}%
      \expandafter\def\csname LT5\endcsname{\color[rgb]{1,1,0}}%
      \expandafter\def\csname LT6\endcsname{\color[rgb]{0,0,0}}%
      \expandafter\def\csname LT7\endcsname{\color[rgb]{1,0.3,0}}%
      \expandafter\def\csname LT8\endcsname{\color[rgb]{0.5,0.5,0.5}}%
    \else
      \def\colorrgb#1{\color{black}}%
      \def\colorgray#1{\color[gray]{#1}}%
      \expandafter\def\csname LTw\endcsname{\color{white}}%
      \expandafter\def\csname LTb\endcsname{\color{black}}%
      \expandafter\def\csname LTa\endcsname{\color{black}}%
      \expandafter\def\csname LT0\endcsname{\color{black}}%
      \expandafter\def\csname LT1\endcsname{\color{black}}%
      \expandafter\def\csname LT2\endcsname{\color{black}}%
      \expandafter\def\csname LT3\endcsname{\color{black}}%
      \expandafter\def\csname LT4\endcsname{\color{black}}%
      \expandafter\def\csname LT5\endcsname{\color{black}}%
      \expandafter\def\csname LT6\endcsname{\color{black}}%
      \expandafter\def\csname LT7\endcsname{\color{black}}%
      \expandafter\def\csname LT8\endcsname{\color{black}}%
    \fi
  \fi
    \setlength{\unitlength}{0.0500bp}%
    \ifx\gptboxheight\undefined%
      \newlength{\gptboxheight}%
      \newlength{\gptboxwidth}%
      \newsavebox{\gptboxtext}%
    \fi%
    \setlength{\fboxrule}{0.5pt}%
    \setlength{\fboxsep}{1pt}%
\begin{picture}(7200.00,5040.00)%
    \gplgaddtomacro\gplbacktext{%
      \csname LTb\endcsname%
      \put(1372,896){\makebox(0,0)[r]{\strut{}1e-16}}%
      \csname LTb\endcsname%
      \put(1372,1363){\makebox(0,0)[r]{\strut{}1e-14}}%
      \csname LTb\endcsname%
      \put(1372,1830){\makebox(0,0)[r]{\strut{}1e-12}}%
      \csname LTb\endcsname%
      \put(1372,2297){\makebox(0,0)[r]{\strut{}1e-10}}%
      \csname LTb\endcsname%
      \put(1372,2764){\makebox(0,0)[r]{\strut{}1e-08}}%
      \csname LTb\endcsname%
      \put(1372,3231){\makebox(0,0)[r]{\strut{}1e-06}}%
      \csname LTb\endcsname%
      \put(1372,3699){\makebox(0,0)[r]{\strut{}1e-04}}%
      \csname LTb\endcsname%
      \put(1372,4166){\makebox(0,0)[r]{\strut{}1e-02}}%
      \csname LTb\endcsname%
      \put(1372,4633){\makebox(0,0)[r]{\strut{}1e+00}}%
      \csname LTb\endcsname%
      \put(1540,616){\makebox(0,0){\strut{}$20$}}%
      \csname LTb\endcsname%
      \put(2399,616){\makebox(0,0){\strut{}$30$}}%
      \csname LTb\endcsname%
      \put(3258,616){\makebox(0,0){\strut{}$40$}}%
      \csname LTb\endcsname%
      \put(4118,616){\makebox(0,0){\strut{}$50$}}%
      \csname LTb\endcsname%
      \put(4977,616){\makebox(0,0){\strut{}$60$}}%
      \csname LTb\endcsname%
      \put(5836,616){\makebox(0,0){\strut{}$70$}}%
      \csname LTb\endcsname%
      \put(6695,616){\makebox(0,0){\strut{}$80$}}%
    }%
    \gplgaddtomacro\gplfronttext{%
      \csname LTb\endcsname%
      \put(224,2799){\rotatebox{-270}{\makebox(0,0){\strut{}abs diff of proj and obs error}}}%
      \put(4117,196){\makebox(0,0){\strut{}$\%$ of observed state components}}%
      \csname LTb\endcsname%
      \put(5288,4500){\makebox(0,0)[r]{\strut{}Stage $\Nreproj = 25,\Treproj=102$ }}%
      \csname LTb\endcsname%
      \put(5288,4220){\makebox(0,0)[r]{\strut{}Batch $\Nreproj=12,\Treproj=250$}}%
      \csname LTb\endcsname%
      \put(5288,3940){\makebox(0,0)[r]{\strut{}Batch $\Nreproj=12,\Treproj=1000$}}%
      \csname LTb\endcsname%
      \put(5288,3660){\makebox(0,0)[r]{\strut{}Batch $\Nreproj=12,\Treproj=10000$}}%
    }%
    \gplbacktext
    \put(0,0){\includegraphics{ConvDiff_ProjErrStageVSBatch_dim10}}%
    \gplfronttext
  \end{picture}%
\endgroup

%% file: figures/ConvDiff_StateErrDetailed_PO20dim10.tex
\begingroup
  \makeatletter
  \providecommand\color[2][]{%
    \GenericError{(gnuplot) \space\space\space\@spaces}{%
      Package color not loaded in conjunction with
      terminal option `colourtext'%
    }{See the gnuplot documentation for explanation.%
    }{Either use 'blacktext' in gnuplot or load the package
      color.sty in LaTeX.}%
    \renewcommand\color[2][]{}%
  }%
  \providecommand\includegraphics[2][]{%
    \GenericError{(gnuplot) \space\space\space\@spaces}{%
      Package graphicx or graphics not loaded%
    }{See the gnuplot documentation for explanation.%
    }{The gnuplot epslatex terminal needs graphicx.sty or graphics.sty.}%
    \renewcommand\includegraphics[2][]{}%
  }%
  \providecommand\rotatebox[2]{#2}%
  \@ifundefined{ifGPcolor}{%
    \newif\ifGPcolor
    \GPcolortrue
  }{}%
  \@ifundefined{ifGPblacktext}{%
    \newif\ifGPblacktext
    \GPblacktexttrue
  }{}%
  \let\gplgaddtomacro\g@addto@macro
  \gdef\gplbacktext{}%
  \gdef\gplfronttext{}%
  \makeatother
  \ifGPblacktext
    \def\colorrgb#1{}%
    \def\colorgray#1{}%
  \else
    \ifGPcolor
      \def\colorrgb#1{\color[rgb]{#1}}%
      \def\colorgray#1{\color[gray]{#1}}%
      \expandafter\def\csname LTw\endcsname{\color{white}}%
      \expandafter\def\csname LTb\endcsname{\color{black}}%
      \expandafter\def\csname LTa\endcsname{\color{black}}%
      \expandafter\def\csname LT0\endcsname{\color[rgb]{1,0,0}}%
      \expandafter\def\csname LT1\endcsname{\color[rgb]{0,1,0}}%
      \expandafter\def\csname LT2\endcsname{\color[rgb]{0,0,1}}%
      \expandafter\def\csname LT3\endcsname{\color[rgb]{1,0,1}}%
      \expandafter\def\csname LT4\endcsname{\color[rgb]{0,1,1}}%
      \expandafter\def\csname LT5\endcsname{\color[rgb]{1,1,0}}%
      \expandafter\def\csname LT6\endcsname{\color[rgb]{0,0,0}}%
      \expandafter\def\csname LT7\endcsname{\color[rgb]{1,0.3,0}}%
      \expandafter\def\csname LT8\endcsname{\color[rgb]{0.5,0.5,0.5}}%
    \else
      \def\colorrgb#1{\color{black}}%
      \def\colorgray#1{\color[gray]{#1}}%
      \expandafter\def\csname LTw\endcsname{\color{white}}%
      \expandafter\def\csname LTb\endcsname{\color{black}}%
      \expandafter\def\csname LTa\endcsname{\color{black}}%
      \expandafter\def\csname LT0\endcsname{\color{black}}%
      \expandafter\def\csname LT1\endcsname{\color{black}}%
      \expandafter\def\csname LT2\endcsname{\color{black}}%
      \expandafter\def\csname LT3\endcsname{\color{black}}%
      \expandafter\def\csname LT4\endcsname{\color{black}}%
      \expandafter\def\csname LT5\endcsname{\color{black}}%
      \expandafter\def\csname LT6\endcsname{\color{black}}%
      \expandafter\def\csname LT7\endcsname{\color{black}}%
      \expandafter\def\csname LT8\endcsname{\color{black}}%
    \fi
  \fi
    \setlength{\unitlength}{0.0500bp}%
    \ifx\gptboxheight\undefined%
      \newlength{\gptboxheight}%
      \newlength{\gptboxwidth}%
      \newsavebox{\gptboxtext}%
    \fi%
    \setlength{\fboxrule}{0.5pt}%
    \setlength{\fboxsep}{1pt}%
\begin{picture}(7200.00,5040.00)%
    \gplgaddtomacro\gplbacktext{%
      \csname LTb\endcsname%
      \put(1372,1042){\makebox(0,0)[r]{\strut{}1e-05}}%
      \csname LTb\endcsname%
      \put(1372,1775){\makebox(0,0)[r]{\strut{}1e+00}}%
      \csname LTb\endcsname%
      \put(1372,2507){\makebox(0,0)[r]{\strut{}1e+05}}%
      \csname LTb\endcsname%
      \put(1372,3239){\makebox(0,0)[r]{\strut{}1e+10}}%
      \csname LTb\endcsname%
      \put(1372,3971){\makebox(0,0)[r]{\strut{}1e+15}}%
      \csname LTb\endcsname%
      \put(1372,4703){\makebox(0,0)[r]{\strut{}1e+20}}%
      \csname LTb\endcsname%
      \put(1540,616){\makebox(0,0){\strut{}$0$}}%
      \csname LTb\endcsname%
      \put(2571,616){\makebox(0,0){\strut{}$20$}}%
      \csname LTb\endcsname%
      \put(3602,616){\makebox(0,0){\strut{}$40$}}%
      \csname LTb\endcsname%
      \put(4633,616){\makebox(0,0){\strut{}$60$}}%
      \csname LTb\endcsname%
      \put(5664,616){\makebox(0,0){\strut{}$80$}}%
      \csname LTb\endcsname%
      \put(6695,616){\makebox(0,0){\strut{}$100$}}%
    }%
    \gplgaddtomacro\gplfronttext{%
      \csname LTb\endcsname%
      \put(224,2799){\rotatebox{-270}{\makebox(0,0){\strut{}observation error}}}%
      \put(4117,196){\makebox(0,0){\strut{}lag $L$ of non-Markovian term}}%
      \csname LTb\endcsname%
      \put(5456,4500){\makebox(0,0)[r]{\strut{}Stage $\Nreproj = 25, \Treproj=102$}}%
      \csname LTb\endcsname%
      \put(5456,4220){\makebox(0,0)[r]{\strut{}Batch $\Nreproj=12,\Treproj=250$}}%
      \csname LTb\endcsname%
      \put(5456,3940){\makebox(0,0)[r]{\strut{}Batch $\Nreproj=12,\Treproj=1000$}}%
      \csname LTb\endcsname%
      \put(5456,3660){\makebox(0,0)[r]{\strut{}Batch $\Nreproj=12,\Treproj=10000$}}%
    }%
    \gplbacktext
    \put(0,0){\includegraphics{ConvDiff_StateErrDetailed_PO20dim10}}%
    \gplfronttext
  \end{picture}%
\endgroup

%% file: figures/ConvDiff_StateErrDetailed_PO80dim4.tex
\begingroup
  \makeatletter
  \providecommand\color[2][]{%
    \GenericError{(gnuplot) \space\space\space\@spaces}{%
      Package color not loaded in conjunction with
      terminal option `colourtext'%
    }{See the gnuplot documentation for explanation.%
    }{Either use 'blacktext' in gnuplot or load the package
      color.sty in LaTeX.}%
    \renewcommand\color[2][]{}%
  }%
  \providecommand\includegraphics[2][]{%
    \GenericError{(gnuplot) \space\space\space\@spaces}{%
      Package graphicx or graphics not loaded%
    }{See the gnuplot documentation for explanation.%
    }{The gnuplot epslatex terminal needs graphicx.sty or graphics.sty.}%
    \renewcommand\includegraphics[2][]{}%
  }%
  \providecommand\rotatebox[2]{#2}%
  \@ifundefined{ifGPcolor}{%
    \newif\ifGPcolor
    \GPcolortrue
  }{}%
  \@ifundefined{ifGPblacktext}{%
    \newif\ifGPblacktext
    \GPblacktexttrue
  }{}%
  \let\gplgaddtomacro\g@addto@macro
  \gdef\gplbacktext{}%
  \gdef\gplfronttext{}%
  \makeatother
  \ifGPblacktext
    \def\colorrgb#1{}%
    \def\colorgray#1{}%
  \else
    \ifGPcolor
      \def\colorrgb#1{\color[rgb]{#1}}%
      \def\colorgray#1{\color[gray]{#1}}%
      \expandafter\def\csname LTw\endcsname{\color{white}}%
      \expandafter\def\csname LTb\endcsname{\color{black}}%
      \expandafter\def\csname LTa\endcsname{\color{black}}%
      \expandafter\def\csname LT0\endcsname{\color[rgb]{1,0,0}}%
      \expandafter\def\csname LT1\endcsname{\color[rgb]{0,1,0}}%
      \expandafter\def\csname LT2\endcsname{\color[rgb]{0,0,1}}%
      \expandafter\def\csname LT3\endcsname{\color[rgb]{1,0,1}}%
      \expandafter\def\csname LT4\endcsname{\color[rgb]{0,1,1}}%
      \expandafter\def\csname LT5\endcsname{\color[rgb]{1,1,0}}%
      \expandafter\def\csname LT6\endcsname{\color[rgb]{0,0,0}}%
      \expandafter\def\csname LT7\endcsname{\color[rgb]{1,0.3,0}}%
      \expandafter\def\csname LT8\endcsname{\color[rgb]{0.5,0.5,0.5}}%
    \else
      \def\colorrgb#1{\color{black}}%
      \def\colorgray#1{\color[gray]{#1}}%
      \expandafter\def\csname LTw\endcsname{\color{white}}%
      \expandafter\def\csname LTb\endcsname{\color{black}}%
      \expandafter\def\csname LTa\endcsname{\color{black}}%
      \expandafter\def\csname LT0\endcsname{\color{black}}%
      \expandafter\def\csname LT1\endcsname{\color{black}}%
      \expandafter\def\csname LT2\endcsname{\color{black}}%
      \expandafter\def\csname LT3\endcsname{\color{black}}%
      \expandafter\def\csname LT4\endcsname{\color{black}}%
      \expandafter\def\csname LT5\endcsname{\color{black}}%
      \expandafter\def\csname LT6\endcsname{\color{black}}%
      \expandafter\def\csname LT7\endcsname{\color{black}}%
      \expandafter\def\csname LT8\endcsname{\color{black}}%
    \fi
  \fi
    \setlength{\unitlength}{0.0500bp}%
    \ifx\gptboxheight\undefined%
      \newlength{\gptboxheight}%
      \newlength{\gptboxwidth}%
      \newsavebox{\gptboxtext}%
    \fi%
    \setlength{\fboxrule}{0.5pt}%
    \setlength{\fboxsep}{1pt}%
\begin{picture}(7200.00,5040.00)%
    \gplgaddtomacro\gplbacktext{%
      \csname LTb\endcsname%
      \put(1372,896){\makebox(0,0)[r]{\strut{}1e-03}}%
      \csname LTb\endcsname%
      \put(1372,1531){\makebox(0,0)[r]{\strut{}1e-02}}%
      \csname LTb\endcsname%
      \put(1372,2165){\makebox(0,0)[r]{\strut{}1e-01}}%
      \csname LTb\endcsname%
      \put(1372,2800){\makebox(0,0)[r]{\strut{}1e+00}}%
      \csname LTb\endcsname%
      \put(1372,3434){\makebox(0,0)[r]{\strut{}1e+01}}%
      \csname LTb\endcsname%
      \put(1372,4069){\makebox(0,0)[r]{\strut{}1e+02}}%
      \csname LTb\endcsname%
      \put(1372,4703){\makebox(0,0)[r]{\strut{}1e+03}}%
      \csname LTb\endcsname%
      \put(1540,616){\makebox(0,0){\strut{}$0$}}%
      \csname LTb\endcsname%
      \put(2571,616){\makebox(0,0){\strut{}$20$}}%
      \csname LTb\endcsname%
      \put(3602,616){\makebox(0,0){\strut{}$40$}}%
      \csname LTb\endcsname%
      \put(4633,616){\makebox(0,0){\strut{}$60$}}%
      \csname LTb\endcsname%
      \put(5664,616){\makebox(0,0){\strut{}$80$}}%
      \csname LTb\endcsname%
      \put(6695,616){\makebox(0,0){\strut{}$100$}}%
    }%
    \gplgaddtomacro\gplfronttext{%
      \csname LTb\endcsname%
      \put(224,2799){\rotatebox{-270}{\makebox(0,0){\strut{}observation error}}}%
      \put(4117,196){\makebox(0,0){\strut{}lag $L$ of non-Markovian term}}%
      \csname LTb\endcsname%
      \put(5456,4500){\makebox(0,0)[r]{\strut{}Stage $\Nreproj = 25,\Treproj=102$}}%
      \csname LTb\endcsname%
      \put(5456,4220){\makebox(0,0)[r]{\strut{}Batch $\Nreproj=12,\Treproj=250$}}%
      \csname LTb\endcsname%
      \put(5456,3940){\makebox(0,0)[r]{\strut{}Batch $\Nreproj=12,\Treproj=1000$}}%
      \csname LTb\endcsname%
      \put(5456,3660){\makebox(0,0)[r]{\strut{}Batch $\Nreproj=12,\Treproj=10000$}}%
    }%
    \gplbacktext
    \put(0,0){\includegraphics{ConvDiff_StateErrDetailed_PO80dim4}}%
    \gplfronttext
  \end{picture}%
\endgroup

%% file: figures/ConvDiff_StateErrCondensed_SchemeA_dim10.tex
\begingroup
  \makeatletter
  \providecommand\color[2][]{%
    \GenericError{(gnuplot) \space\space\space\@spaces}{%
      Package color not loaded in conjunction with
      terminal option `colourtext'%
    }{See the gnuplot documentation for explanation.%
    }{Either use 'blacktext' in gnuplot or load the package
      color.sty in LaTeX.}%
    \renewcommand\color[2][]{}%
  }%
  \providecommand\includegraphics[2][]{%
    \GenericError{(gnuplot) \space\space\space\@spaces}{%
      Package graphicx or graphics not loaded%
    }{See the gnuplot documentation for explanation.%
    }{The gnuplot epslatex terminal needs graphicx.sty or graphics.sty.}%
    \renewcommand\includegraphics[2][]{}%
  }%
  \providecommand\rotatebox[2]{#2}%
  \@ifundefined{ifGPcolor}{%
    \newif\ifGPcolor
    \GPcolortrue
  }{}%
  \@ifundefined{ifGPblacktext}{%
    \newif\ifGPblacktext
    \GPblacktexttrue
  }{}%
  \let\gplgaddtomacro\g@addto@macro
  \gdef\gplbacktext{}%
  \gdef\gplfronttext{}%
  \makeatother
  \ifGPblacktext
    \def\colorrgb#1{}%
    \def\colorgray#1{}%
  \else
    \ifGPcolor
      \def\colorrgb#1{\color[rgb]{#1}}%
      \def\colorgray#1{\color[gray]{#1}}%
      \expandafter\def\csname LTw\endcsname{\color{white}}%
      \expandafter\def\csname LTb\endcsname{\color{black}}%
      \expandafter\def\csname LTa\endcsname{\color{black}}%
      \expandafter\def\csname LT0\endcsname{\color[rgb]{1,0,0}}%
      \expandafter\def\csname LT1\endcsname{\color[rgb]{0,1,0}}%
      \expandafter\def\csname LT2\endcsname{\color[rgb]{0,0,1}}%
      \expandafter\def\csname LT3\endcsname{\color[rgb]{1,0,1}}%
      \expandafter\def\csname LT4\endcsname{\color[rgb]{0,1,1}}%
      \expandafter\def\csname LT5\endcsname{\color[rgb]{1,1,0}}%
      \expandafter\def\csname LT6\endcsname{\color[rgb]{0,0,0}}%
      \expandafter\def\csname LT7\endcsname{\color[rgb]{1,0.3,0}}%
      \expandafter\def\csname LT8\endcsname{\color[rgb]{0.5,0.5,0.5}}%
    \else
      \def\colorrgb#1{\color{black}}%
      \def\colorgray#1{\color[gray]{#1}}%
      \expandafter\def\csname LTw\endcsname{\color{white}}%
      \expandafter\def\csname LTb\endcsname{\color{black}}%
      \expandafter\def\csname LTa\endcsname{\color{black}}%
      \expandafter\def\csname LT0\endcsname{\color{black}}%
      \expandafter\def\csname LT1\endcsname{\color{black}}%
      \expandafter\def\csname LT2\endcsname{\color{black}}%
      \expandafter\def\csname LT3\endcsname{\color{black}}%
      \expandafter\def\csname LT4\endcsname{\color{black}}%
      \expandafter\def\csname LT5\endcsname{\color{black}}%
      \expandafter\def\csname LT6\endcsname{\color{black}}%
      \expandafter\def\csname LT7\endcsname{\color{black}}%
      \expandafter\def\csname LT8\endcsname{\color{black}}%
    \fi
  \fi
    \setlength{\unitlength}{0.0500bp}%
    \ifx\gptboxheight\undefined%
      \newlength{\gptboxheight}%
      \newlength{\gptboxwidth}%
      \newsavebox{\gptboxtext}%
    \fi%
    \setlength{\fboxrule}{0.5pt}%
    \setlength{\fboxsep}{1pt}%
\begin{picture}(7200.00,5040.00)%
    \gplgaddtomacro\gplbacktext{%
      \csname LTb\endcsname%
      \put(1372,896){\makebox(0,0)[r]{\strut{}1e-03}}%
      \csname LTb\endcsname%
      \put(1372,2165){\makebox(0,0)[r]{\strut{}1e-02}}%
      \csname LTb\endcsname%
      \put(1372,3434){\makebox(0,0)[r]{\strut{}1e-01}}%
      \csname LTb\endcsname%
      \put(1372,4703){\makebox(0,0)[r]{\strut{}1e+00}}%
      \csname LTb\endcsname%
      \put(1540,616){\makebox(0,0){\strut{}$20$}}%
      \csname LTb\endcsname%
      \put(2399,616){\makebox(0,0){\strut{}$30$}}%
      \csname LTb\endcsname%
      \put(3258,616){\makebox(0,0){\strut{}$40$}}%
      \csname LTb\endcsname%
      \put(4118,616){\makebox(0,0){\strut{}$50$}}%
      \csname LTb\endcsname%
      \put(4977,616){\makebox(0,0){\strut{}$60$}}%
      \csname LTb\endcsname%
      \put(5836,616){\makebox(0,0){\strut{}$70$}}%
      \csname LTb\endcsname%
      \put(6695,616){\makebox(0,0){\strut{}$80$}}%
    }%
    \gplgaddtomacro\gplfronttext{%
      \csname LTb\endcsname%
      \put(224,2799){\rotatebox{-270}{\makebox(0,0){\strut{}observation error}}}%
      \put(4117,196){\makebox(0,0){\strut{}$\%$ of observed state components}}%
      \csname LTb\endcsname%
      \put(5288,4500){\makebox(0,0)[r]{\strut{}Stage $\Nreproj=25,\Treproj=102$}}%
      \csname LTb\endcsname%
      \put(5288,4220){\makebox(0,0)[r]{\strut{}Batch $\Nreproj = 25, \Treproj=500$}}%
      \csname LTb\endcsname%
      \put(5288,3940){\makebox(0,0)[r]{\strut{}Batch $\Nreproj=25, \Treproj=1000$}}%
      \csname LTb\endcsname%
      \put(5288,3660){\makebox(0,0)[r]{\strut{}Batch $\Nreproj=25, \Treproj=10000$}}%
    }%
    \gplbacktext
    \put(0,0){\includegraphics{ConvDiff_StateErrCondensed_SchemeA_dim10}}%
    \gplfronttext
  \end{picture}%
\endgroup

%% file: figures/ConvDiff_StateErrCondensed_SchemeB_dim10.tex
\begingroup
  \makeatletter
  \providecommand\color[2][]{%
    \GenericError{(gnuplot) \space\space\space\@spaces}{%
      Package color not loaded in conjunction with
      terminal option `colourtext'%
    }{See the gnuplot documentation for explanation.%
    }{Either use 'blacktext' in gnuplot or load the package
      color.sty in LaTeX.}%
    \renewcommand\color[2][]{}%
  }%
  \providecommand\includegraphics[2][]{%
    \GenericError{(gnuplot) \space\space\space\@spaces}{%
      Package graphicx or graphics not loaded%
    }{See the gnuplot documentation for explanation.%
    }{The gnuplot epslatex terminal needs graphicx.sty or graphics.sty.}%
    \renewcommand\includegraphics[2][]{}%
  }%
  \providecommand\rotatebox[2]{#2}%
  \@ifundefined{ifGPcolor}{%
    \newif\ifGPcolor
    \GPcolortrue
  }{}%
  \@ifundefined{ifGPblacktext}{%
    \newif\ifGPblacktext
    \GPblacktexttrue
  }{}%
  \let\gplgaddtomacro\g@addto@macro
  \gdef\gplbacktext{}%
  \gdef\gplfronttext{}%
  \makeatother
  \ifGPblacktext
    \def\colorrgb#1{}%
    \def\colorgray#1{}%
  \else
    \ifGPcolor
      \def\colorrgb#1{\color[rgb]{#1}}%
      \def\colorgray#1{\color[gray]{#1}}%
      \expandafter\def\csname LTw\endcsname{\color{white}}%
      \expandafter\def\csname LTb\endcsname{\color{black}}%
      \expandafter\def\csname LTa\endcsname{\color{black}}%
      \expandafter\def\csname LT0\endcsname{\color[rgb]{1,0,0}}%
      \expandafter\def\csname LT1\endcsname{\color[rgb]{0,1,0}}%
      \expandafter\def\csname LT2\endcsname{\color[rgb]{0,0,1}}%
      \expandafter\def\csname LT3\endcsname{\color[rgb]{1,0,1}}%
      \expandafter\def\csname LT4\endcsname{\color[rgb]{0,1,1}}%
      \expandafter\def\csname LT5\endcsname{\color[rgb]{1,1,0}}%
      \expandafter\def\csname LT6\endcsname{\color[rgb]{0,0,0}}%
      \expandafter\def\csname LT7\endcsname{\color[rgb]{1,0.3,0}}%
      \expandafter\def\csname LT8\endcsname{\color[rgb]{0.5,0.5,0.5}}%
    \else
      \def\colorrgb#1{\color{black}}%
      \def\colorgray#1{\color[gray]{#1}}%
      \expandafter\def\csname LTw\endcsname{\color{white}}%
      \expandafter\def\csname LTb\endcsname{\color{black}}%
      \expandafter\def\csname LTa\endcsname{\color{black}}%
      \expandafter\def\csname LT0\endcsname{\color{black}}%
      \expandafter\def\csname LT1\endcsname{\color{black}}%
      \expandafter\def\csname LT2\endcsname{\color{black}}%
      \expandafter\def\csname LT3\endcsname{\color{black}}%
      \expandafter\def\csname LT4\endcsname{\color{black}}%
      \expandafter\def\csname LT5\endcsname{\color{black}}%
      \expandafter\def\csname LT6\endcsname{\color{black}}%
      \expandafter\def\csname LT7\endcsname{\color{black}}%
      \expandafter\def\csname LT8\endcsname{\color{black}}%
    \fi
  \fi
    \setlength{\unitlength}{0.0500bp}%
    \ifx\gptboxheight\undefined%
      \newlength{\gptboxheight}%
      \newlength{\gptboxwidth}%
      \newsavebox{\gptboxtext}%
    \fi%
    \setlength{\fboxrule}{0.5pt}%
    \setlength{\fboxsep}{1pt}%
\begin{picture}(7200.00,5040.00)%
    \gplgaddtomacro\gplbacktext{%
      \csname LTb\endcsname%
      \put(1372,896){\makebox(0,0)[r]{\strut{}1e-03}}%
      \csname LTb\endcsname%
      \put(1372,2165){\makebox(0,0)[r]{\strut{}1e-02}}%
      \csname LTb\endcsname%
      \put(1372,3434){\makebox(0,0)[r]{\strut{}1e-01}}%
      \csname LTb\endcsname%
      \put(1372,4703){\makebox(0,0)[r]{\strut{}1e+00}}%
      \csname LTb\endcsname%
      \put(1540,616){\makebox(0,0){\strut{}$20$}}%
      \csname LTb\endcsname%
      \put(2399,616){\makebox(0,0){\strut{}$30$}}%
      \csname LTb\endcsname%
      \put(3258,616){\makebox(0,0){\strut{}$40$}}%
      \csname LTb\endcsname%
      \put(4118,616){\makebox(0,0){\strut{}$50$}}%
      \csname LTb\endcsname%
      \put(4977,616){\makebox(0,0){\strut{}$60$}}%
      \csname LTb\endcsname%
      \put(5836,616){\makebox(0,0){\strut{}$70$}}%
      \csname LTb\endcsname%
      \put(6695,616){\makebox(0,0){\strut{}$80$}}%
    }%
    \gplgaddtomacro\gplfronttext{%
      \csname LTb\endcsname%
      \put(224,2799){\rotatebox{-270}{\makebox(0,0){\strut{}observation error}}}%
      \put(4117,196){\makebox(0,0){\strut{}$\%$ of observed state components}}%
      \csname LTb\endcsname%
      \put(5288,4500){\makebox(0,0)[r]{\strut{}Stage $\Nreproj=25, \Treproj=102$}}%
      \csname LTb\endcsname%
      \put(5288,4220){\makebox(0,0)[r]{\strut{}Batch $\Nreproj=12, \Treproj=250$}}%
      \csname LTb\endcsname%
      \put(5288,3940){\makebox(0,0)[r]{\strut{}Batch $\Nreproj=12, \Treproj=1000$}}%
      \csname LTb\endcsname%
      \put(5288,3660){\makebox(0,0)[r]{\strut{}Batch $\Nreproj=12, \Treproj=10000$}}%
    }%
    \gplbacktext
    \put(0,0){\includegraphics{ConvDiff_StateErrCondensed_SchemeB_dim10}}%
    \gplfronttext
  \end{picture}%
\endgroup

%% file: figures/ConvDiff_OutputDetailed_PO60_dim8.tex
\begingroup
  \makeatletter
  \providecommand\color[2][]{%
    \GenericError{(gnuplot) \space\space\space\@spaces}{%
      Package color not loaded in conjunction with
      terminal option `colourtext'%
    }{See the gnuplot documentation for explanation.%
    }{Either use 'blacktext' in gnuplot or load the package
      color.sty in LaTeX.}%
    \renewcommand\color[2][]{}%
  }%
  \providecommand\includegraphics[2][]{%
    \GenericError{(gnuplot) \space\space\space\@spaces}{%
      Package graphicx or graphics not loaded%
    }{See the gnuplot documentation for explanation.%
    }{The gnuplot epslatex terminal needs graphicx.sty or graphics.sty.}%
    \renewcommand\includegraphics[2][]{}%
  }%
  \providecommand\rotatebox[2]{#2}%
  \@ifundefined{ifGPcolor}{%
    \newif\ifGPcolor
    \GPcolortrue
  }{}%
  \@ifundefined{ifGPblacktext}{%
    \newif\ifGPblacktext
    \GPblacktexttrue
  }{}%
  \let\gplgaddtomacro\g@addto@macro
  \gdef\gplbacktext{}%
  \gdef\gplfronttext{}%
  \makeatother
  \ifGPblacktext
    \def\colorrgb#1{}%
    \def\colorgray#1{}%
  \else
    \ifGPcolor
      \def\colorrgb#1{\color[rgb]{#1}}%
      \def\colorgray#1{\color[gray]{#1}}%
      \expandafter\def\csname LTw\endcsname{\color{white}}%
      \expandafter\def\csname LTb\endcsname{\color{black}}%
      \expandafter\def\csname LTa\endcsname{\color{black}}%
      \expandafter\def\csname LT0\endcsname{\color[rgb]{1,0,0}}%
      \expandafter\def\csname LT1\endcsname{\color[rgb]{0,1,0}}%
      \expandafter\def\csname LT2\endcsname{\color[rgb]{0,0,1}}%
      \expandafter\def\csname LT3\endcsname{\color[rgb]{1,0,1}}%
      \expandafter\def\csname LT4\endcsname{\color[rgb]{0,1,1}}%
      \expandafter\def\csname LT5\endcsname{\color[rgb]{1,1,0}}%
      \expandafter\def\csname LT6\endcsname{\color[rgb]{0,0,0}}%
      \expandafter\def\csname LT7\endcsname{\color[rgb]{1,0.3,0}}%
      \expandafter\def\csname LT8\endcsname{\color[rgb]{0.5,0.5,0.5}}%
    \else
      \def\colorrgb#1{\color{black}}%
      \def\colorgray#1{\color[gray]{#1}}%
      \expandafter\def\csname LTw\endcsname{\color{white}}%
      \expandafter\def\csname LTb\endcsname{\color{black}}%
      \expandafter\def\csname LTa\endcsname{\color{black}}%
      \expandafter\def\csname LT0\endcsname{\color{black}}%
      \expandafter\def\csname LT1\endcsname{\color{black}}%
      \expandafter\def\csname LT2\endcsname{\color{black}}%
      \expandafter\def\csname LT3\endcsname{\color{black}}%
      \expandafter\def\csname LT4\endcsname{\color{black}}%
      \expandafter\def\csname LT5\endcsname{\color{black}}%
      \expandafter\def\csname LT6\endcsname{\color{black}}%
      \expandafter\def\csname LT7\endcsname{\color{black}}%
      \expandafter\def\csname LT8\endcsname{\color{black}}%
    \fi
  \fi
    \setlength{\unitlength}{0.0500bp}%
    \ifx\gptboxheight\undefined%
      \newlength{\gptboxheight}%
      \newlength{\gptboxwidth}%
      \newsavebox{\gptboxtext}%
    \fi%
    \setlength{\fboxrule}{0.5pt}%
    \setlength{\fboxsep}{1pt}%
\begin{picture}(7200.00,5040.00)%
    \gplgaddtomacro\gplbacktext{%
      \csname LTb\endcsname%
      \put(1372,896){\makebox(0,0)[r]{\strut{}$-0.02$}}%
      \csname LTb\endcsname%
      \put(1372,1588){\makebox(0,0)[r]{\strut{}$-0.01$}}%
      \csname LTb\endcsname%
      \put(1372,2280){\makebox(0,0)[r]{\strut{}$0$}}%
      \csname LTb\endcsname%
      \put(1372,2973){\makebox(0,0)[r]{\strut{}$0.01$}}%
      \csname LTb\endcsname%
      \put(1372,3665){\makebox(0,0)[r]{\strut{}$0.02$}}%
      \csname LTb\endcsname%
      \put(1372,4357){\makebox(0,0)[r]{\strut{}$0.03$}}%
      \csname LTb\endcsname%
      \put(1540,616){\makebox(0,0){\strut{}$0$}}%
      \csname LTb\endcsname%
      \put(2571,616){\makebox(0,0){\strut{}$0.1$}}%
      \csname LTb\endcsname%
      \put(3602,616){\makebox(0,0){\strut{}$0.2$}}%
      \csname LTb\endcsname%
      \put(4633,616){\makebox(0,0){\strut{}$0.3$}}%
      \csname LTb\endcsname%
      \put(5664,616){\makebox(0,0){\strut{}$0.4$}}%
      \csname LTb\endcsname%
      \put(6695,616){\makebox(0,0){\strut{}$0.5$}}%
    }%
    \gplgaddtomacro\gplfronttext{%
      \csname LTb\endcsname%
      \put(224,2799){\rotatebox{-270}{\makebox(0,0){\strut{}output}}}%
      \put(4117,196){\makebox(0,0){\strut{}time step $k$}}%
      \csname LTb\endcsname%
      \put(5792,4500){\makebox(0,0)[r]{\strut{}full model}}%
      \csname LTb\endcsname%
      \put(5792,4220){\makebox(0,0)[r]{\strut{}Markovian reduced model}}%
      \csname LTb\endcsname%
      \put(5792,3940){\makebox(0,0)[r]{\strut{}non-Markovian reduced model (stage)}}%
    }%
    \gplbacktext
    \put(0,0){\includegraphics{ConvDiff_OutputDetailed_PO60_dim8}}%
    \gplfronttext
  \end{picture}%
\endgroup

%% file: figures/ConvDiff_OutputDetailed_PO20_dim10.tex
\begingroup
  \makeatletter
  \providecommand\color[2][]{%
    \GenericError{(gnuplot) \space\space\space\@spaces}{%
      Package color not loaded in conjunction with
      terminal option `colourtext'%
    }{See the gnuplot documentation for explanation.%
    }{Either use 'blacktext' in gnuplot or load the package
      color.sty in LaTeX.}%
    \renewcommand\color[2][]{}%
  }%
  \providecommand\includegraphics[2][]{%
    \GenericError{(gnuplot) \space\space\space\@spaces}{%
      Package graphicx or graphics not loaded%
    }{See the gnuplot documentation for explanation.%
    }{The gnuplot epslatex terminal needs graphicx.sty or graphics.sty.}%
    \renewcommand\includegraphics[2][]{}%
  }%
  \providecommand\rotatebox[2]{#2}%
  \@ifundefined{ifGPcolor}{%
    \newif\ifGPcolor
    \GPcolortrue
  }{}%
  \@ifundefined{ifGPblacktext}{%
    \newif\ifGPblacktext
    \GPblacktexttrue
  }{}%
  \let\gplgaddtomacro\g@addto@macro
  \gdef\gplbacktext{}%
  \gdef\gplfronttext{}%
  \makeatother
  \ifGPblacktext
    \def\colorrgb#1{}%
    \def\colorgray#1{}%
  \else
    \ifGPcolor
      \def\colorrgb#1{\color[rgb]{#1}}%
      \def\colorgray#1{\color[gray]{#1}}%
      \expandafter\def\csname LTw\endcsname{\color{white}}%
      \expandafter\def\csname LTb\endcsname{\color{black}}%
      \expandafter\def\csname LTa\endcsname{\color{black}}%
      \expandafter\def\csname LT0\endcsname{\color[rgb]{1,0,0}}%
      \expandafter\def\csname LT1\endcsname{\color[rgb]{0,1,0}}%
      \expandafter\def\csname LT2\endcsname{\color[rgb]{0,0,1}}%
      \expandafter\def\csname LT3\endcsname{\color[rgb]{1,0,1}}%
      \expandafter\def\csname LT4\endcsname{\color[rgb]{0,1,1}}%
      \expandafter\def\csname LT5\endcsname{\color[rgb]{1,1,0}}%
      \expandafter\def\csname LT6\endcsname{\color[rgb]{0,0,0}}%
      \expandafter\def\csname LT7\endcsname{\color[rgb]{1,0.3,0}}%
      \expandafter\def\csname LT8\endcsname{\color[rgb]{0.5,0.5,0.5}}%
    \else
      \def\colorrgb#1{\color{black}}%
      \def\colorgray#1{\color[gray]{#1}}%
      \expandafter\def\csname LTw\endcsname{\color{white}}%
      \expandafter\def\csname LTb\endcsname{\color{black}}%
      \expandafter\def\csname LTa\endcsname{\color{black}}%
      \expandafter\def\csname LT0\endcsname{\color{black}}%
      \expandafter\def\csname LT1\endcsname{\color{black}}%
      \expandafter\def\csname LT2\endcsname{\color{black}}%
      \expandafter\def\csname LT3\endcsname{\color{black}}%
      \expandafter\def\csname LT4\endcsname{\color{black}}%
      \expandafter\def\csname LT5\endcsname{\color{black}}%
      \expandafter\def\csname LT6\endcsname{\color{black}}%
      \expandafter\def\csname LT7\endcsname{\color{black}}%
      \expandafter\def\csname LT8\endcsname{\color{black}}%
    \fi
  \fi
    \setlength{\unitlength}{0.0500bp}%
    \ifx\gptboxheight\undefined%
      \newlength{\gptboxheight}%
      \newlength{\gptboxwidth}%
      \newsavebox{\gptboxtext}%
    \fi%
    \setlength{\fboxrule}{0.5pt}%
    \setlength{\fboxsep}{1pt}%
\begin{picture}(7200.00,5040.00)%
    \gplgaddtomacro\gplbacktext{%
      \csname LTb\endcsname%
      \put(1372,896){\makebox(0,0)[r]{\strut{}$-0.02$}}%
      \csname LTb\endcsname%
      \put(1372,1588){\makebox(0,0)[r]{\strut{}$-0.01$}}%
      \csname LTb\endcsname%
      \put(1372,2280){\makebox(0,0)[r]{\strut{}$0$}}%
      \csname LTb\endcsname%
      \put(1372,2973){\makebox(0,0)[r]{\strut{}$0.01$}}%
      \csname LTb\endcsname%
      \put(1372,3665){\makebox(0,0)[r]{\strut{}$0.02$}}%
      \csname LTb\endcsname%
      \put(1372,4357){\makebox(0,0)[r]{\strut{}$0.03$}}%
      \csname LTb\endcsname%
      \put(1540,616){\makebox(0,0){\strut{}$0$}}%
      \csname LTb\endcsname%
      \put(2571,616){\makebox(0,0){\strut{}$0.1$}}%
      \csname LTb\endcsname%
      \put(3602,616){\makebox(0,0){\strut{}$0.2$}}%
      \csname LTb\endcsname%
      \put(4633,616){\makebox(0,0){\strut{}$0.3$}}%
      \csname LTb\endcsname%
      \put(5664,616){\makebox(0,0){\strut{}$0.4$}}%
      \csname LTb\endcsname%
      \put(6695,616){\makebox(0,0){\strut{}$0.5$}}%
    }%
    \gplgaddtomacro\gplfronttext{%
      \csname LTb\endcsname%
      \put(224,2799){\rotatebox{-270}{\makebox(0,0){\strut{}output}}}%
      \put(4117,196){\makebox(0,0){\strut{}time step $k$}}%
      \csname LTb\endcsname%
      \put(5792,4500){\makebox(0,0)[r]{\strut{}full model}}%
      \csname LTb\endcsname%
      \put(5792,4220){\makebox(0,0)[r]{\strut{}Markovian reduced model}}%
      \csname LTb\endcsname%
      \put(5792,3940){\makebox(0,0)[r]{\strut{}non-Markovian reduced model (stage)}}%
    }%
    \gplbacktext
    \put(0,0){\includegraphics{ConvDiff_OutputDetailed_PO20_dim10}}%
    \gplfronttext
  \end{picture}%
\endgroup

%% file: figures/ConvDiff_OutputCondensed_dim8.tex
\begingroup
  \makeatletter
  \providecommand\color[2][]{%
    \GenericError{(gnuplot) \space\space\space\@spaces}{%
      Package color not loaded in conjunction with
      terminal option `colourtext'%
    }{See the gnuplot documentation for explanation.%
    }{Either use 'blacktext' in gnuplot or load the package
      color.sty in LaTeX.}%
    \renewcommand\color[2][]{}%
  }%
  \providecommand\includegraphics[2][]{%
    \GenericError{(gnuplot) \space\space\space\@spaces}{%
      Package graphicx or graphics not loaded%
    }{See the gnuplot documentation for explanation.%
    }{The gnuplot epslatex terminal needs graphicx.sty or graphics.sty.}%
    \renewcommand\includegraphics[2][]{}%
  }%
  \providecommand\rotatebox[2]{#2}%
  \@ifundefined{ifGPcolor}{%
    \newif\ifGPcolor
    \GPcolortrue
  }{}%
  \@ifundefined{ifGPblacktext}{%
    \newif\ifGPblacktext
    \GPblacktexttrue
  }{}%
  \let\gplgaddtomacro\g@addto@macro
  \gdef\gplbacktext{}%
  \gdef\gplfronttext{}%
  \makeatother
  \ifGPblacktext
    \def\colorrgb#1{}%
    \def\colorgray#1{}%
  \else
    \ifGPcolor
      \def\colorrgb#1{\color[rgb]{#1}}%
      \def\colorgray#1{\color[gray]{#1}}%
      \expandafter\def\csname LTw\endcsname{\color{white}}%
      \expandafter\def\csname LTb\endcsname{\color{black}}%
      \expandafter\def\csname LTa\endcsname{\color{black}}%
      \expandafter\def\csname LT0\endcsname{\color[rgb]{1,0,0}}%
      \expandafter\def\csname LT1\endcsname{\color[rgb]{0,1,0}}%
      \expandafter\def\csname LT2\endcsname{\color[rgb]{0,0,1}}%
      \expandafter\def\csname LT3\endcsname{\color[rgb]{1,0,1}}%
      \expandafter\def\csname LT4\endcsname{\color[rgb]{0,1,1}}%
      \expandafter\def\csname LT5\endcsname{\color[rgb]{1,1,0}}%
      \expandafter\def\csname LT6\endcsname{\color[rgb]{0,0,0}}%
      \expandafter\def\csname LT7\endcsname{\color[rgb]{1,0.3,0}}%
      \expandafter\def\csname LT8\endcsname{\color[rgb]{0.5,0.5,0.5}}%
    \else
      \def\colorrgb#1{\color{black}}%
      \def\colorgray#1{\color[gray]{#1}}%
      \expandafter\def\csname LTw\endcsname{\color{white}}%
      \expandafter\def\csname LTb\endcsname{\color{black}}%
      \expandafter\def\csname LTa\endcsname{\color{black}}%
      \expandafter\def\csname LT0\endcsname{\color{black}}%
      \expandafter\def\csname LT1\endcsname{\color{black}}%
      \expandafter\def\csname LT2\endcsname{\color{black}}%
      \expandafter\def\csname LT3\endcsname{\color{black}}%
      \expandafter\def\csname LT4\endcsname{\color{black}}%
      \expandafter\def\csname LT5\endcsname{\color{black}}%
      \expandafter\def\csname LT6\endcsname{\color{black}}%
      \expandafter\def\csname LT7\endcsname{\color{black}}%
      \expandafter\def\csname LT8\endcsname{\color{black}}%
    \fi
  \fi
    \setlength{\unitlength}{0.0500bp}%
    \ifx\gptboxheight\undefined%
      \newlength{\gptboxheight}%
      \newlength{\gptboxwidth}%
      \newsavebox{\gptboxtext}%
    \fi%
    \setlength{\fboxrule}{0.5pt}%
    \setlength{\fboxsep}{1pt}%
\begin{picture}(7200.00,5040.00)%
    \gplgaddtomacro\gplbacktext{%
      \csname LTb\endcsname%
      \put(1372,896){\makebox(0,0)[r]{\strut{}1e-05}}%
      \csname LTb\endcsname%
      \put(1372,1848){\makebox(0,0)[r]{\strut{}1e-04}}%
      \csname LTb\endcsname%
      \put(1372,2800){\makebox(0,0)[r]{\strut{}1e-03}}%
      \csname LTb\endcsname%
      \put(1372,3751){\makebox(0,0)[r]{\strut{}1e-02}}%
      \csname LTb\endcsname%
      \put(1372,4703){\makebox(0,0)[r]{\strut{}1e-01}}%
      \csname LTb\endcsname%
      \put(1540,616){\makebox(0,0){\strut{}$20$}}%
      \csname LTb\endcsname%
      \put(2399,616){\makebox(0,0){\strut{}$30$}}%
      \csname LTb\endcsname%
      \put(3258,616){\makebox(0,0){\strut{}$40$}}%
      \csname LTb\endcsname%
      \put(4118,616){\makebox(0,0){\strut{}$50$}}%
      \csname LTb\endcsname%
      \put(4977,616){\makebox(0,0){\strut{}$60$}}%
      \csname LTb\endcsname%
      \put(5836,616){\makebox(0,0){\strut{}$70$}}%
      \csname LTb\endcsname%
      \put(6695,616){\makebox(0,0){\strut{}$80$}}%
    }%
    \gplgaddtomacro\gplfronttext{%
      \csname LTb\endcsname%
      \put(224,2799){\rotatebox{-270}{\makebox(0,0){\strut{}output error \eqref{eq:OutErr}}}}%
      \put(4117,196){\makebox(0,0){\strut{}$\%$ of observed state components}}%
      \csname LTb\endcsname%
      \put(6296,4563){\makebox(0,0)[r]{\strut{}Markovian reduced model}}%
      \csname LTb\endcsname%
      \put(6296,4283){\makebox(0,0)[r]{\strut{}non-Markovian reduced model (stage)}}%
    }%
    \gplbacktext
    \put(0,0){\includegraphics{ConvDiff_OutputCondensed_dim8}}%
    \gplfronttext
  \end{picture}%
\endgroup

%% file: figures/ConvDiff_OutputCondensed_dim10.tex
\begingroup
  \makeatletter
  \providecommand\color[2][]{%
    \GenericError{(gnuplot) \space\space\space\@spaces}{%
      Package color not loaded in conjunction with
      terminal option `colourtext'%
    }{See the gnuplot documentation for explanation.%
    }{Either use 'blacktext' in gnuplot or load the package
      color.sty in LaTeX.}%
    \renewcommand\color[2][]{}%
  }%
  \providecommand\includegraphics[2][]{%
    \GenericError{(gnuplot) \space\space\space\@spaces}{%
      Package graphicx or graphics not loaded%
    }{See the gnuplot documentation for explanation.%
    }{The gnuplot epslatex terminal needs graphicx.sty or graphics.sty.}%
    \renewcommand\includegraphics[2][]{}%
  }%
  \providecommand\rotatebox[2]{#2}%
  \@ifundefined{ifGPcolor}{%
    \newif\ifGPcolor
    \GPcolortrue
  }{}%
  \@ifundefined{ifGPblacktext}{%
    \newif\ifGPblacktext
    \GPblacktexttrue
  }{}%
  \let\gplgaddtomacro\g@addto@macro
  \gdef\gplbacktext{}%
  \gdef\gplfronttext{}%
  \makeatother
  \ifGPblacktext
    \def\colorrgb#1{}%
    \def\colorgray#1{}%
  \else
    \ifGPcolor
      \def\colorrgb#1{\color[rgb]{#1}}%
      \def\colorgray#1{\color[gray]{#1}}%
      \expandafter\def\csname LTw\endcsname{\color{white}}%
      \expandafter\def\csname LTb\endcsname{\color{black}}%
      \expandafter\def\csname LTa\endcsname{\color{black}}%
      \expandafter\def\csname LT0\endcsname{\color[rgb]{1,0,0}}%
      \expandafter\def\csname LT1\endcsname{\color[rgb]{0,1,0}}%
      \expandafter\def\csname LT2\endcsname{\color[rgb]{0,0,1}}%
      \expandafter\def\csname LT3\endcsname{\color[rgb]{1,0,1}}%
      \expandafter\def\csname LT4\endcsname{\color[rgb]{0,1,1}}%
      \expandafter\def\csname LT5\endcsname{\color[rgb]{1,1,0}}%
      \expandafter\def\csname LT6\endcsname{\color[rgb]{0,0,0}}%
      \expandafter\def\csname LT7\endcsname{\color[rgb]{1,0.3,0}}%
      \expandafter\def\csname LT8\endcsname{\color[rgb]{0.5,0.5,0.5}}%
    \else
      \def\colorrgb#1{\color{black}}%
      \def\colorgray#1{\color[gray]{#1}}%
      \expandafter\def\csname LTw\endcsname{\color{white}}%
      \expandafter\def\csname LTb\endcsname{\color{black}}%
      \expandafter\def\csname LTa\endcsname{\color{black}}%
      \expandafter\def\csname LT0\endcsname{\color{black}}%
      \expandafter\def\csname LT1\endcsname{\color{black}}%
      \expandafter\def\csname LT2\endcsname{\color{black}}%
      \expandafter\def\csname LT3\endcsname{\color{black}}%
      \expandafter\def\csname LT4\endcsname{\color{black}}%
      \expandafter\def\csname LT5\endcsname{\color{black}}%
      \expandafter\def\csname LT6\endcsname{\color{black}}%
      \expandafter\def\csname LT7\endcsname{\color{black}}%
      \expandafter\def\csname LT8\endcsname{\color{black}}%
    \fi
  \fi
    \setlength{\unitlength}{0.0500bp}%
    \ifx\gptboxheight\undefined%
      \newlength{\gptboxheight}%
      \newlength{\gptboxwidth}%
      \newsavebox{\gptboxtext}%
    \fi%
    \setlength{\fboxrule}{0.5pt}%
    \setlength{\fboxsep}{1pt}%
\begin{picture}(7200.00,5040.00)%
    \gplgaddtomacro\gplbacktext{%
      \csname LTb\endcsname%
      \put(1372,896){\makebox(0,0)[r]{\strut{}1e-05}}%
      \csname LTb\endcsname%
      \put(1372,1848){\makebox(0,0)[r]{\strut{}1e-04}}%
      \csname LTb\endcsname%
      \put(1372,2800){\makebox(0,0)[r]{\strut{}1e-03}}%
      \csname LTb\endcsname%
      \put(1372,3751){\makebox(0,0)[r]{\strut{}1e-02}}%
      \csname LTb\endcsname%
      \put(1372,4703){\makebox(0,0)[r]{\strut{}1e-01}}%
      \csname LTb\endcsname%
      \put(1540,616){\makebox(0,0){\strut{}$20$}}%
      \csname LTb\endcsname%
      \put(2399,616){\makebox(0,0){\strut{}$30$}}%
      \csname LTb\endcsname%
      \put(3258,616){\makebox(0,0){\strut{}$40$}}%
      \csname LTb\endcsname%
      \put(4118,616){\makebox(0,0){\strut{}$50$}}%
      \csname LTb\endcsname%
      \put(4977,616){\makebox(0,0){\strut{}$60$}}%
      \csname LTb\endcsname%
      \put(5836,616){\makebox(0,0){\strut{}$70$}}%
      \csname LTb\endcsname%
      \put(6695,616){\makebox(0,0){\strut{}$80$}}%
    }%
    \gplgaddtomacro\gplfronttext{%
      \csname LTb\endcsname%
      \put(224,2799){\rotatebox{-270}{\makebox(0,0){\strut{}output error \eqref{eq:OutErr}}}}%
      \put(4117,196){\makebox(0,0){\strut{}$\%$ of observed state components}}%
      \csname LTb\endcsname%
      \put(6296,4563){\makebox(0,0)[r]{\strut{}Markovian reduced model}}%
      \csname LTb\endcsname%
      \put(6296,4283){\makebox(0,0)[r]{\strut{}non-Markovian reduced model (stage)}}%
    }%
    \gplbacktext
    \put(0,0){\includegraphics{ConvDiff_OutputCondensed_dim10}}%
    \gplfronttext
  \end{picture}%
\endgroup

%% file: figures/ConvDiff_OutputDetailed_PO99_dim10.tex
\begingroup
  \makeatletter
  \providecommand\color[2][]{%
    \GenericError{(gnuplot) \space\space\space\@spaces}{%
      Package color not loaded in conjunction with
      terminal option `colourtext'%
    }{See the gnuplot documentation for explanation.%
    }{Either use 'blacktext' in gnuplot or load the package
      color.sty in LaTeX.}%
    \renewcommand\color[2][]{}%
  }%
  \providecommand\includegraphics[2][]{%
    \GenericError{(gnuplot) \space\space\space\@spaces}{%
      Package graphicx or graphics not loaded%
    }{See the gnuplot documentation for explanation.%
    }{The gnuplot epslatex terminal needs graphicx.sty or graphics.sty.}%
    \renewcommand\includegraphics[2][]{}%
  }%
  \providecommand\rotatebox[2]{#2}%
  \@ifundefined{ifGPcolor}{%
    \newif\ifGPcolor
    \GPcolortrue
  }{}%
  \@ifundefined{ifGPblacktext}{%
    \newif\ifGPblacktext
    \GPblacktexttrue
  }{}%
  \let\gplgaddtomacro\g@addto@macro
  \gdef\gplbacktext{}%
  \gdef\gplfronttext{}%
  \makeatother
  \ifGPblacktext
    \def\colorrgb#1{}%
    \def\colorgray#1{}%
  \else
    \ifGPcolor
      \def\colorrgb#1{\color[rgb]{#1}}%
      \def\colorgray#1{\color[gray]{#1}}%
      \expandafter\def\csname LTw\endcsname{\color{white}}%
      \expandafter\def\csname LTb\endcsname{\color{black}}%
      \expandafter\def\csname LTa\endcsname{\color{black}}%
      \expandafter\def\csname LT0\endcsname{\color[rgb]{1,0,0}}%
      \expandafter\def\csname LT1\endcsname{\color[rgb]{0,1,0}}%
      \expandafter\def\csname LT2\endcsname{\color[rgb]{0,0,1}}%
      \expandafter\def\csname LT3\endcsname{\color[rgb]{1,0,1}}%
      \expandafter\def\csname LT4\endcsname{\color[rgb]{0,1,1}}%
      \expandafter\def\csname LT5\endcsname{\color[rgb]{1,1,0}}%
      \expandafter\def\csname LT6\endcsname{\color[rgb]{0,0,0}}%
      \expandafter\def\csname LT7\endcsname{\color[rgb]{1,0.3,0}}%
      \expandafter\def\csname LT8\endcsname{\color[rgb]{0.5,0.5,0.5}}%
    \else
      \def\colorrgb#1{\color{black}}%
      \def\colorgray#1{\color[gray]{#1}}%
      \expandafter\def\csname LTw\endcsname{\color{white}}%
      \expandafter\def\csname LTb\endcsname{\color{black}}%
      \expandafter\def\csname LTa\endcsname{\color{black}}%
      \expandafter\def\csname LT0\endcsname{\color{black}}%
      \expandafter\def\csname LT1\endcsname{\color{black}}%
      \expandafter\def\csname LT2\endcsname{\color{black}}%
      \expandafter\def\csname LT3\endcsname{\color{black}}%
      \expandafter\def\csname LT4\endcsname{\color{black}}%
      \expandafter\def\csname LT5\endcsname{\color{black}}%
      \expandafter\def\csname LT6\endcsname{\color{black}}%
      \expandafter\def\csname LT7\endcsname{\color{black}}%
      \expandafter\def\csname LT8\endcsname{\color{black}}%
    \fi
  \fi
    \setlength{\unitlength}{0.0500bp}%
    \ifx\gptboxheight\undefined%
      \newlength{\gptboxheight}%
      \newlength{\gptboxwidth}%
      \newsavebox{\gptboxtext}%
    \fi%
    \setlength{\fboxrule}{0.5pt}%
    \setlength{\fboxsep}{1pt}%
\begin{picture}(7200.00,5040.00)%
    \gplgaddtomacro\gplbacktext{%
      \csname LTb\endcsname%
      \put(1372,896){\makebox(0,0)[r]{\strut{}$-0.02$}}%
      \csname LTb\endcsname%
      \put(1372,1588){\makebox(0,0)[r]{\strut{}$-0.01$}}%
      \csname LTb\endcsname%
      \put(1372,2280){\makebox(0,0)[r]{\strut{}$0$}}%
      \csname LTb\endcsname%
      \put(1372,2973){\makebox(0,0)[r]{\strut{}$0.01$}}%
      \csname LTb\endcsname%
      \put(1372,3665){\makebox(0,0)[r]{\strut{}$0.02$}}%
      \csname LTb\endcsname%
      \put(1372,4357){\makebox(0,0)[r]{\strut{}$0.03$}}%
      \csname LTb\endcsname%
      \put(1540,616){\makebox(0,0){\strut{}$0$}}%
      \csname LTb\endcsname%
      \put(2571,616){\makebox(0,0){\strut{}$0.1$}}%
      \csname LTb\endcsname%
      \put(3602,616){\makebox(0,0){\strut{}$0.2$}}%
      \csname LTb\endcsname%
      \put(4633,616){\makebox(0,0){\strut{}$0.3$}}%
      \csname LTb\endcsname%
      \put(5664,616){\makebox(0,0){\strut{}$0.4$}}%
      \csname LTb\endcsname%
      \put(6695,616){\makebox(0,0){\strut{}$0.5$}}%
    }%
    \gplgaddtomacro\gplfronttext{%
      \csname LTb\endcsname%
      \put(224,2799){\rotatebox{-270}{\makebox(0,0){\strut{}output}}}%
      \put(4117,196){\makebox(0,0){\strut{}time step $k$}}%
      \csname LTb\endcsname%
      \put(5792,4500){\makebox(0,0)[r]{\strut{}full model}}%
      \csname LTb\endcsname%
      \put(5792,4220){\makebox(0,0)[r]{\strut{}Markovian reduced model (99\%)}}%
      \csname LTb\endcsname%
      \put(5792,3940){\makebox(0,0)[r]{\strut{}non-Markovian reduced model (20\%)}}%
    }%
    \gplbacktext
    \put(0,0){\includegraphics{ConvDiff_OutputDetailed_PO99_dim10}}%
    \gplfronttext
  \end{picture}%
\endgroup

%% file: figures/ConvDiff_OutputDetailed_PO99_dim8.tex
\begingroup
  \makeatletter
  \providecommand\color[2][]{%
    \GenericError{(gnuplot) \space\space\space\@spaces}{%
      Package color not loaded in conjunction with
      terminal option `colourtext'%
    }{See the gnuplot documentation for explanation.%
    }{Either use 'blacktext' in gnuplot or load the package
      color.sty in LaTeX.}%
    \renewcommand\color[2][]{}%
  }%
  \providecommand\includegraphics[2][]{%
    \GenericError{(gnuplot) \space\space\space\@spaces}{%
      Package graphicx or graphics not loaded%
    }{See the gnuplot documentation for explanation.%
    }{The gnuplot epslatex terminal needs graphicx.sty or graphics.sty.}%
    \renewcommand\includegraphics[2][]{}%
  }%
  \providecommand\rotatebox[2]{#2}%
  \@ifundefined{ifGPcolor}{%
    \newif\ifGPcolor
    \GPcolortrue
  }{}%
  \@ifundefined{ifGPblacktext}{%
    \newif\ifGPblacktext
    \GPblacktexttrue
  }{}%
  \let\gplgaddtomacro\g@addto@macro
  \gdef\gplbacktext{}%
  \gdef\gplfronttext{}%
  \makeatother
  \ifGPblacktext
    \def\colorrgb#1{}%
    \def\colorgray#1{}%
  \else
    \ifGPcolor
      \def\colorrgb#1{\color[rgb]{#1}}%
      \def\colorgray#1{\color[gray]{#1}}%
      \expandafter\def\csname LTw\endcsname{\color{white}}%
      \expandafter\def\csname LTb\endcsname{\color{black}}%
      \expandafter\def\csname LTa\endcsname{\color{black}}%
      \expandafter\def\csname LT0\endcsname{\color[rgb]{1,0,0}}%
      \expandafter\def\csname LT1\endcsname{\color[rgb]{0,1,0}}%
      \expandafter\def\csname LT2\endcsname{\color[rgb]{0,0,1}}%
      \expandafter\def\csname LT3\endcsname{\color[rgb]{1,0,1}}%
      \expandafter\def\csname LT4\endcsname{\color[rgb]{0,1,1}}%
      \expandafter\def\csname LT5\endcsname{\color[rgb]{1,1,0}}%
      \expandafter\def\csname LT6\endcsname{\color[rgb]{0,0,0}}%
      \expandafter\def\csname LT7\endcsname{\color[rgb]{1,0.3,0}}%
      \expandafter\def\csname LT8\endcsname{\color[rgb]{0.5,0.5,0.5}}%
    \else
      \def\colorrgb#1{\color{black}}%
      \def\colorgray#1{\color[gray]{#1}}%
      \expandafter\def\csname LTw\endcsname{\color{white}}%
      \expandafter\def\csname LTb\endcsname{\color{black}}%
      \expandafter\def\csname LTa\endcsname{\color{black}}%
      \expandafter\def\csname LT0\endcsname{\color{black}}%
      \expandafter\def\csname LT1\endcsname{\color{black}}%
      \expandafter\def\csname LT2\endcsname{\color{black}}%
      \expandafter\def\csname LT3\endcsname{\color{black}}%
      \expandafter\def\csname LT4\endcsname{\color{black}}%
      \expandafter\def\csname LT5\endcsname{\color{black}}%
      \expandafter\def\csname LT6\endcsname{\color{black}}%
      \expandafter\def\csname LT7\endcsname{\color{black}}%
      \expandafter\def\csname LT8\endcsname{\color{black}}%
    \fi
  \fi
    \setlength{\unitlength}{0.0500bp}%
    \ifx\gptboxheight\undefined%
      \newlength{\gptboxheight}%
      \newlength{\gptboxwidth}%
      \newsavebox{\gptboxtext}%
    \fi%
    \setlength{\fboxrule}{0.5pt}%
    \setlength{\fboxsep}{1pt}%
\begin{picture}(7200.00,5040.00)%
    \gplgaddtomacro\gplbacktext{%
      \csname LTb\endcsname%
      \put(1372,896){\makebox(0,0)[r]{\strut{}$-0.02$}}%
      \csname LTb\endcsname%
      \put(1372,1588){\makebox(0,0)[r]{\strut{}$-0.01$}}%
      \csname LTb\endcsname%
      \put(1372,2280){\makebox(0,0)[r]{\strut{}$0$}}%
      \csname LTb\endcsname%
      \put(1372,2973){\makebox(0,0)[r]{\strut{}$0.01$}}%
      \csname LTb\endcsname%
      \put(1372,3665){\makebox(0,0)[r]{\strut{}$0.02$}}%
      \csname LTb\endcsname%
      \put(1372,4357){\makebox(0,0)[r]{\strut{}$0.03$}}%
      \csname LTb\endcsname%
      \put(1540,616){\makebox(0,0){\strut{}$0$}}%
      \csname LTb\endcsname%
      \put(2571,616){\makebox(0,0){\strut{}$0.1$}}%
      \csname LTb\endcsname%
      \put(3602,616){\makebox(0,0){\strut{}$0.2$}}%
      \csname LTb\endcsname%
      \put(4633,616){\makebox(0,0){\strut{}$0.3$}}%
      \csname LTb\endcsname%
      \put(5664,616){\makebox(0,0){\strut{}$0.4$}}%
      \csname LTb\endcsname%
      \put(6695,616){\makebox(0,0){\strut{}$0.5$}}%
    }%
    \gplgaddtomacro\gplfronttext{%
      \csname LTb\endcsname%
      \put(224,2799){\rotatebox{-270}{\makebox(0,0){\strut{}output}}}%
      \put(4117,196){\makebox(0,0){\strut{}time step $k$}}%
      \csname LTb\endcsname%
      \put(5792,4500){\makebox(0,0)[r]{\strut{}full model}}%
      \csname LTb\endcsname%
      \put(5792,4220){\makebox(0,0)[r]{\strut{}Markovian reduced model (99\%)}}%
      \csname LTb\endcsname%
      \put(5792,3940){\makebox(0,0)[r]{\strut{}non-Markovian reduced model (60\%)}}%
    }%
    \gplbacktext
    \put(0,0){\includegraphics{ConvDiff_OutputDetailed_PO99_dim8}}%
    \gplfronttext
  \end{picture}%
\endgroup

%% file: figures/React2D_StateErrCondensed_SchemeA_dim8.tex
\begingroup
  \makeatletter
  \providecommand\color[2][]{%
    \GenericError{(gnuplot) \space\space\space\@spaces}{%
      Package color not loaded in conjunction with
      terminal option `colourtext'%
    }{See the gnuplot documentation for explanation.%
    }{Either use 'blacktext' in gnuplot or load the package
      color.sty in LaTeX.}%
    \renewcommand\color[2][]{}%
  }%
  \providecommand\includegraphics[2][]{%
    \GenericError{(gnuplot) \space\space\space\@spaces}{%
      Package graphicx or graphics not loaded%
    }{See the gnuplot documentation for explanation.%
    }{The gnuplot epslatex terminal needs graphicx.sty or graphics.sty.}%
    \renewcommand\includegraphics[2][]{}%
  }%
  \providecommand\rotatebox[2]{#2}%
  \@ifundefined{ifGPcolor}{%
    \newif\ifGPcolor
    \GPcolortrue
  }{}%
  \@ifundefined{ifGPblacktext}{%
    \newif\ifGPblacktext
    \GPblacktexttrue
  }{}%
  \let\gplgaddtomacro\g@addto@macro
  \gdef\gplbacktext{}%
  \gdef\gplfronttext{}%
  \makeatother
  \ifGPblacktext
    \def\colorrgb#1{}%
    \def\colorgray#1{}%
  \else
    \ifGPcolor
      \def\colorrgb#1{\color[rgb]{#1}}%
      \def\colorgray#1{\color[gray]{#1}}%
      \expandafter\def\csname LTw\endcsname{\color{white}}%
      \expandafter\def\csname LTb\endcsname{\color{black}}%
      \expandafter\def\csname LTa\endcsname{\color{black}}%
      \expandafter\def\csname LT0\endcsname{\color[rgb]{1,0,0}}%
      \expandafter\def\csname LT1\endcsname{\color[rgb]{0,1,0}}%
      \expandafter\def\csname LT2\endcsname{\color[rgb]{0,0,1}}%
      \expandafter\def\csname LT3\endcsname{\color[rgb]{1,0,1}}%
      \expandafter\def\csname LT4\endcsname{\color[rgb]{0,1,1}}%
      \expandafter\def\csname LT5\endcsname{\color[rgb]{1,1,0}}%
      \expandafter\def\csname LT6\endcsname{\color[rgb]{0,0,0}}%
      \expandafter\def\csname LT7\endcsname{\color[rgb]{1,0.3,0}}%
      \expandafter\def\csname LT8\endcsname{\color[rgb]{0.5,0.5,0.5}}%
    \else
      \def\colorrgb#1{\color{black}}%
      \def\colorgray#1{\color[gray]{#1}}%
      \expandafter\def\csname LTw\endcsname{\color{white}}%
      \expandafter\def\csname LTb\endcsname{\color{black}}%
      \expandafter\def\csname LTa\endcsname{\color{black}}%
      \expandafter\def\csname LT0\endcsname{\color{black}}%
      \expandafter\def\csname LT1\endcsname{\color{black}}%
      \expandafter\def\csname LT2\endcsname{\color{black}}%
      \expandafter\def\csname LT3\endcsname{\color{black}}%
      \expandafter\def\csname LT4\endcsname{\color{black}}%
      \expandafter\def\csname LT5\endcsname{\color{black}}%
      \expandafter\def\csname LT6\endcsname{\color{black}}%
      \expandafter\def\csname LT7\endcsname{\color{black}}%
      \expandafter\def\csname LT8\endcsname{\color{black}}%
    \fi
  \fi
    \setlength{\unitlength}{0.0500bp}%
    \ifx\gptboxheight\undefined%
      \newlength{\gptboxheight}%
      \newlength{\gptboxwidth}%
      \newsavebox{\gptboxtext}%
    \fi%
    \setlength{\fboxrule}{0.5pt}%
    \setlength{\fboxsep}{1pt}%
\begin{picture}(7200.00,5040.00)%
    \gplgaddtomacro\gplbacktext{%
      \csname LTb\endcsname%
      \put(1372,1166){\makebox(0,0)[r]{\strut{}1e-03}}%
      \csname LTb\endcsname%
      \put(1372,2385){\makebox(0,0)[r]{\strut{}1e-02}}%
      \csname LTb\endcsname%
      \put(1372,3603){\makebox(0,0)[r]{\strut{}1e-01}}%
      \csname LTb\endcsname%
      \put(1540,616){\makebox(0,0){\strut{}$20$}}%
      \csname LTb\endcsname%
      \put(2399,616){\makebox(0,0){\strut{}$30$}}%
      \csname LTb\endcsname%
      \put(3258,616){\makebox(0,0){\strut{}$40$}}%
      \csname LTb\endcsname%
      \put(4118,616){\makebox(0,0){\strut{}$50$}}%
      \csname LTb\endcsname%
      \put(4977,616){\makebox(0,0){\strut{}$60$}}%
      \csname LTb\endcsname%
      \put(5836,616){\makebox(0,0){\strut{}$70$}}%
      \csname LTb\endcsname%
      \put(6695,616){\makebox(0,0){\strut{}$80$}}%
    }%
    \gplgaddtomacro\gplfronttext{%
      \csname LTb\endcsname%
      \put(224,2799){\rotatebox{-270}{\makebox(0,0){\strut{}observation error}}}%
      \put(4117,196){\makebox(0,0){\strut{}$\%$ of observed state components}}%
      \csname LTb\endcsname%
      \put(5792,4500){\makebox(0,0)[r]{\strut{}Stage $\Nreproj = 100, \Treproj = 42$}}%
      \csname LTb\endcsname%
      \put(5792,4220){\makebox(0,0)[r]{\strut{}Batch $\Nreproj = 100, \Treproj=42$}}%
      \csname LTb\endcsname%
      \put(5792,3940){\makebox(0,0)[r]{\strut{}Batch $\Nreproj = 100, \Treproj=60$}}%
      \csname LTb\endcsname%
      \put(5792,3660){\makebox(0,0)[r]{\strut{}Batch $\Nreproj = 100, \Treproj=80$}}%
      \csname LTb\endcsname%
      \put(5792,3380){\makebox(0,0)[r]{\strut{}Proj err \eqref{eq:ProjErrorParam}}}%
    }%
    \gplbacktext
    \put(0,0){\includegraphics{React2D_StateErrCondensed_SchemeA_dim8}}%
    \gplfronttext
  \end{picture}%
\endgroup

%% file: figures/React2D_StateErrCondensed_SchemeA_dim10.tex
\begingroup
  \makeatletter
  \providecommand\color[2][]{%
    \GenericError{(gnuplot) \space\space\space\@spaces}{%
      Package color not loaded in conjunction with
      terminal option `colourtext'%
    }{See the gnuplot documentation for explanation.%
    }{Either use 'blacktext' in gnuplot or load the package
      color.sty in LaTeX.}%
    \renewcommand\color[2][]{}%
  }%
  \providecommand\includegraphics[2][]{%
    \GenericError{(gnuplot) \space\space\space\@spaces}{%
      Package graphicx or graphics not loaded%
    }{See the gnuplot documentation for explanation.%
    }{The gnuplot epslatex terminal needs graphicx.sty or graphics.sty.}%
    \renewcommand\includegraphics[2][]{}%
  }%
  \providecommand\rotatebox[2]{#2}%
  \@ifundefined{ifGPcolor}{%
    \newif\ifGPcolor
    \GPcolortrue
  }{}%
  \@ifundefined{ifGPblacktext}{%
    \newif\ifGPblacktext
    \GPblacktexttrue
  }{}%
  \let\gplgaddtomacro\g@addto@macro
  \gdef\gplbacktext{}%
  \gdef\gplfronttext{}%
  \makeatother
  \ifGPblacktext
    \def\colorrgb#1{}%
    \def\colorgray#1{}%
  \else
    \ifGPcolor
      \def\colorrgb#1{\color[rgb]{#1}}%
      \def\colorgray#1{\color[gray]{#1}}%
      \expandafter\def\csname LTw\endcsname{\color{white}}%
      \expandafter\def\csname LTb\endcsname{\color{black}}%
      \expandafter\def\csname LTa\endcsname{\color{black}}%
      \expandafter\def\csname LT0\endcsname{\color[rgb]{1,0,0}}%
      \expandafter\def\csname LT1\endcsname{\color[rgb]{0,1,0}}%
      \expandafter\def\csname LT2\endcsname{\color[rgb]{0,0,1}}%
      \expandafter\def\csname LT3\endcsname{\color[rgb]{1,0,1}}%
      \expandafter\def\csname LT4\endcsname{\color[rgb]{0,1,1}}%
      \expandafter\def\csname LT5\endcsname{\color[rgb]{1,1,0}}%
      \expandafter\def\csname LT6\endcsname{\color[rgb]{0,0,0}}%
      \expandafter\def\csname LT7\endcsname{\color[rgb]{1,0.3,0}}%
      \expandafter\def\csname LT8\endcsname{\color[rgb]{0.5,0.5,0.5}}%
    \else
      \def\colorrgb#1{\color{black}}%
      \def\colorgray#1{\color[gray]{#1}}%
      \expandafter\def\csname LTw\endcsname{\color{white}}%
      \expandafter\def\csname LTb\endcsname{\color{black}}%
      \expandafter\def\csname LTa\endcsname{\color{black}}%
      \expandafter\def\csname LT0\endcsname{\color{black}}%
      \expandafter\def\csname LT1\endcsname{\color{black}}%
      \expandafter\def\csname LT2\endcsname{\color{black}}%
      \expandafter\def\csname LT3\endcsname{\color{black}}%
      \expandafter\def\csname LT4\endcsname{\color{black}}%
      \expandafter\def\csname LT5\endcsname{\color{black}}%
      \expandafter\def\csname LT6\endcsname{\color{black}}%
      \expandafter\def\csname LT7\endcsname{\color{black}}%
      \expandafter\def\csname LT8\endcsname{\color{black}}%
    \fi
  \fi
    \setlength{\unitlength}{0.0500bp}%
    \ifx\gptboxheight\undefined%
      \newlength{\gptboxheight}%
      \newlength{\gptboxwidth}%
      \newsavebox{\gptboxtext}%
    \fi%
    \setlength{\fboxrule}{0.5pt}%
    \setlength{\fboxsep}{1pt}%
\begin{picture}(7200.00,5040.00)%
    \gplgaddtomacro\gplbacktext{%
      \csname LTb\endcsname%
      \put(1372,1439){\makebox(0,0)[r]{\strut{}1e-03}}%
      \csname LTb\endcsname%
      \put(1372,2476){\makebox(0,0)[r]{\strut{}1e-02}}%
      \csname LTb\endcsname%
      \put(1372,3514){\makebox(0,0)[r]{\strut{}1e-01}}%
      \csname LTb\endcsname%
      \put(1372,4551){\makebox(0,0)[r]{\strut{}1e+00}}%
      \csname LTb\endcsname%
      \put(1540,616){\makebox(0,0){\strut{}$20$}}%
      \csname LTb\endcsname%
      \put(2399,616){\makebox(0,0){\strut{}$30$}}%
      \csname LTb\endcsname%
      \put(3258,616){\makebox(0,0){\strut{}$40$}}%
      \csname LTb\endcsname%
      \put(4118,616){\makebox(0,0){\strut{}$50$}}%
      \csname LTb\endcsname%
      \put(4977,616){\makebox(0,0){\strut{}$60$}}%
      \csname LTb\endcsname%
      \put(5836,616){\makebox(0,0){\strut{}$70$}}%
      \csname LTb\endcsname%
      \put(6695,616){\makebox(0,0){\strut{}$80$}}%
    }%
    \gplgaddtomacro\gplfronttext{%
      \csname LTb\endcsname%
      \put(224,2799){\rotatebox{-270}{\makebox(0,0){\strut{}observation error}}}%
      \put(4117,196){\makebox(0,0){\strut{}$\%$ of observed state components}}%
      \csname LTb\endcsname%
      \put(5792,4500){\makebox(0,0)[r]{\strut{}Stage $\Nreproj = 100, \Treproj = 42$}}%
      \csname LTb\endcsname%
      \put(5792,4220){\makebox(0,0)[r]{\strut{}Batch $\Nreproj = 100, \Treproj=42$}}%
      \csname LTb\endcsname%
      \put(5792,3940){\makebox(0,0)[r]{\strut{}Batch $\Nreproj = 100, \Treproj=60$}}%
      \csname LTb\endcsname%
      \put(5792,3660){\makebox(0,0)[r]{\strut{}Batch $\Nreproj = 100, \Treproj=80$}}%
      \csname LTb\endcsname%
      \put(5792,3380){\makebox(0,0)[r]{\strut{}Proj err \eqref{eq:ProjErrorParam}}}%
    }%
    \gplbacktext
    \put(0,0){\includegraphics{React2D_StateErrCondensed_SchemeA_dim10}}%
    \gplfronttext
  \end{picture}%
\endgroup

%% file: figures/React2D_StateErrCondensed_SchemeBTrain_dim10.tex
\begingroup
  \makeatletter
  \providecommand\color[2][]{%
    \GenericError{(gnuplot) \space\space\space\@spaces}{%
      Package color not loaded in conjunction with
      terminal option `colourtext'%
    }{See the gnuplot documentation for explanation.%
    }{Either use 'blacktext' in gnuplot or load the package
      color.sty in LaTeX.}%
    \renewcommand\color[2][]{}%
  }%
  \providecommand\includegraphics[2][]{%
    \GenericError{(gnuplot) \space\space\space\@spaces}{%
      Package graphicx or graphics not loaded%
    }{See the gnuplot documentation for explanation.%
    }{The gnuplot epslatex terminal needs graphicx.sty or graphics.sty.}%
    \renewcommand\includegraphics[2][]{}%
  }%
  \providecommand\rotatebox[2]{#2}%
  \@ifundefined{ifGPcolor}{%
    \newif\ifGPcolor
    \GPcolortrue
  }{}%
  \@ifundefined{ifGPblacktext}{%
    \newif\ifGPblacktext
    \GPblacktexttrue
  }{}%
  \let\gplgaddtomacro\g@addto@macro
  \gdef\gplbacktext{}%
  \gdef\gplfronttext{}%
  \makeatother
  \ifGPblacktext
    \def\colorrgb#1{}%
    \def\colorgray#1{}%
  \else
    \ifGPcolor
      \def\colorrgb#1{\color[rgb]{#1}}%
      \def\colorgray#1{\color[gray]{#1}}%
      \expandafter\def\csname LTw\endcsname{\color{white}}%
      \expandafter\def\csname LTb\endcsname{\color{black}}%
      \expandafter\def\csname LTa\endcsname{\color{black}}%
      \expandafter\def\csname LT0\endcsname{\color[rgb]{1,0,0}}%
      \expandafter\def\csname LT1\endcsname{\color[rgb]{0,1,0}}%
      \expandafter\def\csname LT2\endcsname{\color[rgb]{0,0,1}}%
      \expandafter\def\csname LT3\endcsname{\color[rgb]{1,0,1}}%
      \expandafter\def\csname LT4\endcsname{\color[rgb]{0,1,1}}%
      \expandafter\def\csname LT5\endcsname{\color[rgb]{1,1,0}}%
      \expandafter\def\csname LT6\endcsname{\color[rgb]{0,0,0}}%
      \expandafter\def\csname LT7\endcsname{\color[rgb]{1,0.3,0}}%
      \expandafter\def\csname LT8\endcsname{\color[rgb]{0.5,0.5,0.5}}%
    \else
      \def\colorrgb#1{\color{black}}%
      \def\colorgray#1{\color[gray]{#1}}%
      \expandafter\def\csname LTw\endcsname{\color{white}}%
      \expandafter\def\csname LTb\endcsname{\color{black}}%
      \expandafter\def\csname LTa\endcsname{\color{black}}%
      \expandafter\def\csname LT0\endcsname{\color{black}}%
      \expandafter\def\csname LT1\endcsname{\color{black}}%
      \expandafter\def\csname LT2\endcsname{\color{black}}%
      \expandafter\def\csname LT3\endcsname{\color{black}}%
      \expandafter\def\csname LT4\endcsname{\color{black}}%
      \expandafter\def\csname LT5\endcsname{\color{black}}%
      \expandafter\def\csname LT6\endcsname{\color{black}}%
      \expandafter\def\csname LT7\endcsname{\color{black}}%
      \expandafter\def\csname LT8\endcsname{\color{black}}%
    \fi
  \fi
    \setlength{\unitlength}{0.0500bp}%
    \ifx\gptboxheight\undefined%
      \newlength{\gptboxheight}%
      \newlength{\gptboxwidth}%
      \newsavebox{\gptboxtext}%
    \fi%
    \setlength{\fboxrule}{0.5pt}%
    \setlength{\fboxsep}{1pt}%
\begin{picture}(7200.00,5040.00)%
    \gplgaddtomacro\gplbacktext{%
      \csname LTb\endcsname%
      \put(1372,953){\makebox(0,0)[r]{\strut{}1e-03}}%
      \csname LTb\endcsname%
      \put(1372,2194){\makebox(0,0)[r]{\strut{}1e-02}}%
      \csname LTb\endcsname%
      \put(1372,3435){\makebox(0,0)[r]{\strut{}1e-01}}%
      \csname LTb\endcsname%
      \put(1372,4677){\makebox(0,0)[r]{\strut{}1e+00}}%
      \csname LTb\endcsname%
      \put(1540,616){\makebox(0,0){\strut{}$20$}}%
      \csname LTb\endcsname%
      \put(2399,616){\makebox(0,0){\strut{}$30$}}%
      \csname LTb\endcsname%
      \put(3258,616){\makebox(0,0){\strut{}$40$}}%
      \csname LTb\endcsname%
      \put(4118,616){\makebox(0,0){\strut{}$50$}}%
      \csname LTb\endcsname%
      \put(4977,616){\makebox(0,0){\strut{}$60$}}%
      \csname LTb\endcsname%
      \put(5836,616){\makebox(0,0){\strut{}$70$}}%
      \csname LTb\endcsname%
      \put(6695,616){\makebox(0,0){\strut{}$80$}}%
    }%
    \gplgaddtomacro\gplfronttext{%
      \csname LTb\endcsname%
      \put(224,2799){\rotatebox{-270}{\makebox(0,0){\strut{}observation error \eqref{eq:StateErrParam}}}}%
      \put(4117,196){\makebox(0,0){\strut{}$\%$ of observed state components}}%
      \csname LTb\endcsname%
      \put(5792,4500){\makebox(0,0)[r]{\strut{}Stage $\Nreproj = 100, \Treproj = 42$}}%
      \csname LTb\endcsname%
      \put(5792,4220){\makebox(0,0)[r]{\strut{}Batch $\Nreproj = 12, \Treproj=100$}}%
      \csname LTb\endcsname%
      \put(5792,3940){\makebox(0,0)[r]{\strut{}Batch $\Nreproj = 12, \Treproj=500$}}%
      \csname LTb\endcsname%
      \put(5792,3660){\makebox(0,0)[r]{\strut{}Batch $\Nreproj = 12, \Treproj=2500$}}%
      \csname LTb\endcsname%
      \put(5792,3380){\makebox(0,0)[r]{\strut{}Proj err \eqref{eq:ProjErrorParam}}}%
    }%
    \gplbacktext
    \put(0,0){\includegraphics{React2D_StateErrCondensed_SchemeBTrain_dim10}}%
    \gplfronttext
  \end{picture}%
\endgroup

%% file: figures/React2D_StateErrCondensed_SchemeBTest_dim10.tex
\begingroup
  \makeatletter
  \providecommand\color[2][]{%
    \GenericError{(gnuplot) \space\space\space\@spaces}{%
      Package color not loaded in conjunction with
      terminal option `colourtext'%
    }{See the gnuplot documentation for explanation.%
    }{Either use 'blacktext' in gnuplot or load the package
      color.sty in LaTeX.}%
    \renewcommand\color[2][]{}%
  }%
  \providecommand\includegraphics[2][]{%
    \GenericError{(gnuplot) \space\space\space\@spaces}{%
      Package graphicx or graphics not loaded%
    }{See the gnuplot documentation for explanation.%
    }{The gnuplot epslatex terminal needs graphicx.sty or graphics.sty.}%
    \renewcommand\includegraphics[2][]{}%
  }%
  \providecommand\rotatebox[2]{#2}%
  \@ifundefined{ifGPcolor}{%
    \newif\ifGPcolor
    \GPcolortrue
  }{}%
  \@ifundefined{ifGPblacktext}{%
    \newif\ifGPblacktext
    \GPblacktexttrue
  }{}%
  \let\gplgaddtomacro\g@addto@macro
  \gdef\gplbacktext{}%
  \gdef\gplfronttext{}%
  \makeatother
  \ifGPblacktext
    \def\colorrgb#1{}%
    \def\colorgray#1{}%
  \else
    \ifGPcolor
      \def\colorrgb#1{\color[rgb]{#1}}%
      \def\colorgray#1{\color[gray]{#1}}%
      \expandafter\def\csname LTw\endcsname{\color{white}}%
      \expandafter\def\csname LTb\endcsname{\color{black}}%
      \expandafter\def\csname LTa\endcsname{\color{black}}%
      \expandafter\def\csname LT0\endcsname{\color[rgb]{1,0,0}}%
      \expandafter\def\csname LT1\endcsname{\color[rgb]{0,1,0}}%
      \expandafter\def\csname LT2\endcsname{\color[rgb]{0,0,1}}%
      \expandafter\def\csname LT3\endcsname{\color[rgb]{1,0,1}}%
      \expandafter\def\csname LT4\endcsname{\color[rgb]{0,1,1}}%
      \expandafter\def\csname LT5\endcsname{\color[rgb]{1,1,0}}%
      \expandafter\def\csname LT6\endcsname{\color[rgb]{0,0,0}}%
      \expandafter\def\csname LT7\endcsname{\color[rgb]{1,0.3,0}}%
      \expandafter\def\csname LT8\endcsname{\color[rgb]{0.5,0.5,0.5}}%
    \else
      \def\colorrgb#1{\color{black}}%
      \def\colorgray#1{\color[gray]{#1}}%
      \expandafter\def\csname LTw\endcsname{\color{white}}%
      \expandafter\def\csname LTb\endcsname{\color{black}}%
      \expandafter\def\csname LTa\endcsname{\color{black}}%
      \expandafter\def\csname LT0\endcsname{\color{black}}%
      \expandafter\def\csname LT1\endcsname{\color{black}}%
      \expandafter\def\csname LT2\endcsname{\color{black}}%
      \expandafter\def\csname LT3\endcsname{\color{black}}%
      \expandafter\def\csname LT4\endcsname{\color{black}}%
      \expandafter\def\csname LT5\endcsname{\color{black}}%
      \expandafter\def\csname LT6\endcsname{\color{black}}%
      \expandafter\def\csname LT7\endcsname{\color{black}}%
      \expandafter\def\csname LT8\endcsname{\color{black}}%
    \fi
  \fi
    \setlength{\unitlength}{0.0500bp}%
    \ifx\gptboxheight\undefined%
      \newlength{\gptboxheight}%
      \newlength{\gptboxwidth}%
      \newsavebox{\gptboxtext}%
    \fi%
    \setlength{\fboxrule}{0.5pt}%
    \setlength{\fboxsep}{1pt}%
\begin{picture}(7200.00,5040.00)%
    \gplgaddtomacro\gplbacktext{%
      \csname LTb\endcsname%
      \put(1372,1449){\makebox(0,0)[r]{\strut{}1e-03}}%
      \csname LTb\endcsname%
      \put(1372,2506){\makebox(0,0)[r]{\strut{}1e-02}}%
      \csname LTb\endcsname%
      \put(1372,3562){\makebox(0,0)[r]{\strut{}1e-01}}%
      \csname LTb\endcsname%
      \put(1372,4619){\makebox(0,0)[r]{\strut{}1e+00}}%
      \csname LTb\endcsname%
      \put(1540,616){\makebox(0,0){\strut{}$20$}}%
      \csname LTb\endcsname%
      \put(2399,616){\makebox(0,0){\strut{}$30$}}%
      \csname LTb\endcsname%
      \put(3258,616){\makebox(0,0){\strut{}$40$}}%
      \csname LTb\endcsname%
      \put(4118,616){\makebox(0,0){\strut{}$50$}}%
      \csname LTb\endcsname%
      \put(4977,616){\makebox(0,0){\strut{}$60$}}%
      \csname LTb\endcsname%
      \put(5836,616){\makebox(0,0){\strut{}$70$}}%
      \csname LTb\endcsname%
      \put(6695,616){\makebox(0,0){\strut{}$80$}}%
    }%
    \gplgaddtomacro\gplfronttext{%
      \csname LTb\endcsname%
      \put(224,2799){\rotatebox{-270}{\makebox(0,0){\strut{}observation error \eqref{eq:StateErrParam}}}}%
      \put(4117,196){\makebox(0,0){\strut{}$\%$ of observed state components}}%
      \csname LTb\endcsname%
      \put(5792,4500){\makebox(0,0)[r]{\strut{}Stage $\Nreproj = 100, \Treproj = 42$}}%
      \csname LTb\endcsname%
      \put(5792,4220){\makebox(0,0)[r]{\strut{}Batch $\Nreproj = 12, \Treproj=100$}}%
      \csname LTb\endcsname%
      \put(5792,3940){\makebox(0,0)[r]{\strut{}Batch $\Nreproj = 12, \Treproj=500$}}%
      \csname LTb\endcsname%
      \put(5792,3660){\makebox(0,0)[r]{\strut{}Batch $\Nreproj = 12, \Treproj=2500$}}%
      \csname LTb\endcsname%
      \put(5792,3380){\makebox(0,0)[r]{\strut{}Proj err \eqref{eq:ProjErrorParam}}}%
    }%
    \gplbacktext
    \put(0,0){\includegraphics{React2D_StateErrCondensed_SchemeBTest_dim10}}%
    \gplfronttext
  \end{picture}%
\endgroup

%% file: figures/Chafee_OutErrCondensed_dim8.tex
\begingroup
  \makeatletter
  \providecommand\color[2][]{%
    \GenericError{(gnuplot) \space\space\space\@spaces}{%
      Package color not loaded in conjunction with
      terminal option `colourtext'%
    }{See the gnuplot documentation for explanation.%
    }{Either use 'blacktext' in gnuplot or load the package
      color.sty in LaTeX.}%
    \renewcommand\color[2][]{}%
  }%
  \providecommand\includegraphics[2][]{%
    \GenericError{(gnuplot) \space\space\space\@spaces}{%
      Package graphicx or graphics not loaded%
    }{See the gnuplot documentation for explanation.%
    }{The gnuplot epslatex terminal needs graphicx.sty or graphics.sty.}%
    \renewcommand\includegraphics[2][]{}%
  }%
  \providecommand\rotatebox[2]{#2}%
  \@ifundefined{ifGPcolor}{%
    \newif\ifGPcolor
    \GPcolortrue
  }{}%
  \@ifundefined{ifGPblacktext}{%
    \newif\ifGPblacktext
    \GPblacktexttrue
  }{}%
  \let\gplgaddtomacro\g@addto@macro
  \gdef\gplbacktext{}%
  \gdef\gplfronttext{}%
  \makeatother
  \ifGPblacktext
    \def\colorrgb#1{}%
    \def\colorgray#1{}%
  \else
    \ifGPcolor
      \def\colorrgb#1{\color[rgb]{#1}}%
      \def\colorgray#1{\color[gray]{#1}}%
      \expandafter\def\csname LTw\endcsname{\color{white}}%
      \expandafter\def\csname LTb\endcsname{\color{black}}%
      \expandafter\def\csname LTa\endcsname{\color{black}}%
      \expandafter\def\csname LT0\endcsname{\color[rgb]{1,0,0}}%
      \expandafter\def\csname LT1\endcsname{\color[rgb]{0,1,0}}%
      \expandafter\def\csname LT2\endcsname{\color[rgb]{0,0,1}}%
      \expandafter\def\csname LT3\endcsname{\color[rgb]{1,0,1}}%
      \expandafter\def\csname LT4\endcsname{\color[rgb]{0,1,1}}%
      \expandafter\def\csname LT5\endcsname{\color[rgb]{1,1,0}}%
      \expandafter\def\csname LT6\endcsname{\color[rgb]{0,0,0}}%
      \expandafter\def\csname LT7\endcsname{\color[rgb]{1,0.3,0}}%
      \expandafter\def\csname LT8\endcsname{\color[rgb]{0.5,0.5,0.5}}%
    \else
      \def\colorrgb#1{\color{black}}%
      \def\colorgray#1{\color[gray]{#1}}%
      \expandafter\def\csname LTw\endcsname{\color{white}}%
      \expandafter\def\csname LTb\endcsname{\color{black}}%
      \expandafter\def\csname LTa\endcsname{\color{black}}%
      \expandafter\def\csname LT0\endcsname{\color{black}}%
      \expandafter\def\csname LT1\endcsname{\color{black}}%
      \expandafter\def\csname LT2\endcsname{\color{black}}%
      \expandafter\def\csname LT3\endcsname{\color{black}}%
      \expandafter\def\csname LT4\endcsname{\color{black}}%
      \expandafter\def\csname LT5\endcsname{\color{black}}%
      \expandafter\def\csname LT6\endcsname{\color{black}}%
      \expandafter\def\csname LT7\endcsname{\color{black}}%
      \expandafter\def\csname LT8\endcsname{\color{black}}%
    \fi
  \fi
    \setlength{\unitlength}{0.0500bp}%
    \ifx\gptboxheight\undefined%
      \newlength{\gptboxheight}%
      \newlength{\gptboxwidth}%
      \newsavebox{\gptboxtext}%
    \fi%
    \setlength{\fboxrule}{0.5pt}%
    \setlength{\fboxsep}{1pt}%
\begin{picture}(7200.00,5040.00)%
    \gplgaddtomacro\gplbacktext{%
      \csname LTb\endcsname%
      \put(1372,896){\makebox(0,0)[r]{\strut{}1e-02}}%
      \csname LTb\endcsname%
      \put(1372,2800){\makebox(0,0)[r]{\strut{}1e-01}}%
      \csname LTb\endcsname%
      \put(1372,4703){\makebox(0,0)[r]{\strut{}1e+00}}%
      \csname LTb\endcsname%
      \put(1540,616){\makebox(0,0){\strut{}$40$}}%
      \csname LTb\endcsname%
      \put(2184,616){\makebox(0,0){\strut{}$45$}}%
      \csname LTb\endcsname%
      \put(2829,616){\makebox(0,0){\strut{}$50$}}%
      \csname LTb\endcsname%
      \put(3473,616){\makebox(0,0){\strut{}$55$}}%
      \csname LTb\endcsname%
      \put(4118,616){\makebox(0,0){\strut{}$60$}}%
      \csname LTb\endcsname%
      \put(4762,616){\makebox(0,0){\strut{}$65$}}%
      \csname LTb\endcsname%
      \put(5406,616){\makebox(0,0){\strut{}$70$}}%
      \csname LTb\endcsname%
      \put(6051,616){\makebox(0,0){\strut{}$75$}}%
      \csname LTb\endcsname%
      \put(6695,616){\makebox(0,0){\strut{}$80$}}%
    }%
    \gplgaddtomacro\gplfronttext{%
      \csname LTb\endcsname%
      \put(224,2799){\rotatebox{-270}{\makebox(0,0){\strut{}output error \eqref{eq:OutErr}}}}%
      \put(4117,196){\makebox(0,0){\strut{}$\%$ of observed state components}}%
      \csname LTb\endcsname%
      \put(5792,4500){\makebox(0,0)[r]{\strut{}Markovian reduced model}}%
      \csname LTb\endcsname%
      \put(5792,4220){\makebox(0,0)[r]{\strut{}Stage $\Nreproj = 600, \Treproj = 62$}}%
      \csname LTb\endcsname%
      \put(5792,3940){\makebox(0,0)[r]{\strut{}Batch $\Nreproj = 600, \Treproj = 100$}}%
    }%
    \gplbacktext
    \put(0,0){\includegraphics{Chafee_OutErrCondensed_dim8}}%
    \gplfronttext
  \end{picture}%
\endgroup

%% file: figures/Chafee_OutErrCondensed_dim10.tex
\begingroup
  \makeatletter
  \providecommand\color[2][]{%
    \GenericError{(gnuplot) \space\space\space\@spaces}{%
      Package color not loaded in conjunction with
      terminal option `colourtext'%
    }{See the gnuplot documentation for explanation.%
    }{Either use 'blacktext' in gnuplot or load the package
      color.sty in LaTeX.}%
    \renewcommand\color[2][]{}%
  }%
  \providecommand\includegraphics[2][]{%
    \GenericError{(gnuplot) \space\space\space\@spaces}{%
      Package graphicx or graphics not loaded%
    }{See the gnuplot documentation for explanation.%
    }{The gnuplot epslatex terminal needs graphicx.sty or graphics.sty.}%
    \renewcommand\includegraphics[2][]{}%
  }%
  \providecommand\rotatebox[2]{#2}%
  \@ifundefined{ifGPcolor}{%
    \newif\ifGPcolor
    \GPcolortrue
  }{}%
  \@ifundefined{ifGPblacktext}{%
    \newif\ifGPblacktext
    \GPblacktexttrue
  }{}%
  \let\gplgaddtomacro\g@addto@macro
  \gdef\gplbacktext{}%
  \gdef\gplfronttext{}%
  \makeatother
  \ifGPblacktext
    \def\colorrgb#1{}%
    \def\colorgray#1{}%
  \else
    \ifGPcolor
      \def\colorrgb#1{\color[rgb]{#1}}%
      \def\colorgray#1{\color[gray]{#1}}%
      \expandafter\def\csname LTw\endcsname{\color{white}}%
      \expandafter\def\csname LTb\endcsname{\color{black}}%
      \expandafter\def\csname LTa\endcsname{\color{black}}%
      \expandafter\def\csname LT0\endcsname{\color[rgb]{1,0,0}}%
      \expandafter\def\csname LT1\endcsname{\color[rgb]{0,1,0}}%
      \expandafter\def\csname LT2\endcsname{\color[rgb]{0,0,1}}%
      \expandafter\def\csname LT3\endcsname{\color[rgb]{1,0,1}}%
      \expandafter\def\csname LT4\endcsname{\color[rgb]{0,1,1}}%
      \expandafter\def\csname LT5\endcsname{\color[rgb]{1,1,0}}%
      \expandafter\def\csname LT6\endcsname{\color[rgb]{0,0,0}}%
      \expandafter\def\csname LT7\endcsname{\color[rgb]{1,0.3,0}}%
      \expandafter\def\csname LT8\endcsname{\color[rgb]{0.5,0.5,0.5}}%
    \else
      \def\colorrgb#1{\color{black}}%
      \def\colorgray#1{\color[gray]{#1}}%
      \expandafter\def\csname LTw\endcsname{\color{white}}%
      \expandafter\def\csname LTb\endcsname{\color{black}}%
      \expandafter\def\csname LTa\endcsname{\color{black}}%
      \expandafter\def\csname LT0\endcsname{\color{black}}%
      \expandafter\def\csname LT1\endcsname{\color{black}}%
      \expandafter\def\csname LT2\endcsname{\color{black}}%
      \expandafter\def\csname LT3\endcsname{\color{black}}%
      \expandafter\def\csname LT4\endcsname{\color{black}}%
      \expandafter\def\csname LT5\endcsname{\color{black}}%
      \expandafter\def\csname LT6\endcsname{\color{black}}%
      \expandafter\def\csname LT7\endcsname{\color{black}}%
      \expandafter\def\csname LT8\endcsname{\color{black}}%
    \fi
  \fi
    \setlength{\unitlength}{0.0500bp}%
    \ifx\gptboxheight\undefined%
      \newlength{\gptboxheight}%
      \newlength{\gptboxwidth}%
      \newsavebox{\gptboxtext}%
    \fi%
    \setlength{\fboxrule}{0.5pt}%
    \setlength{\fboxsep}{1pt}%
\begin{picture}(7200.00,5040.00)%
    \gplgaddtomacro\gplbacktext{%
      \csname LTb\endcsname%
      \put(1372,896){\makebox(0,0)[r]{\strut{}1e-02}}%
      \csname LTb\endcsname%
      \put(1372,2800){\makebox(0,0)[r]{\strut{}1e-01}}%
      \csname LTb\endcsname%
      \put(1372,4703){\makebox(0,0)[r]{\strut{}1e+00}}%
      \csname LTb\endcsname%
      \put(1540,616){\makebox(0,0){\strut{}$40$}}%
      \csname LTb\endcsname%
      \put(2184,616){\makebox(0,0){\strut{}$45$}}%
      \csname LTb\endcsname%
      \put(2829,616){\makebox(0,0){\strut{}$50$}}%
      \csname LTb\endcsname%
      \put(3473,616){\makebox(0,0){\strut{}$55$}}%
      \csname LTb\endcsname%
      \put(4118,616){\makebox(0,0){\strut{}$60$}}%
      \csname LTb\endcsname%
      \put(4762,616){\makebox(0,0){\strut{}$65$}}%
      \csname LTb\endcsname%
      \put(5406,616){\makebox(0,0){\strut{}$70$}}%
      \csname LTb\endcsname%
      \put(6051,616){\makebox(0,0){\strut{}$75$}}%
      \csname LTb\endcsname%
      \put(6695,616){\makebox(0,0){\strut{}$80$}}%
    }%
    \gplgaddtomacro\gplfronttext{%
      \csname LTb\endcsname%
      \put(224,2799){\rotatebox{-270}{\makebox(0,0){\strut{}output error \eqref{eq:OutErr}}}}%
      \put(4117,196){\makebox(0,0){\strut{}$\%$ of observed state components}}%
      \csname LTb\endcsname%
      \put(5792,4500){\makebox(0,0)[r]{\strut{}Markovian reduced model}}%
      \csname LTb\endcsname%
      \put(5792,4220){\makebox(0,0)[r]{\strut{}Stage $\Nreproj = 600, \Treproj = 62$}}%
      \csname LTb\endcsname%
      \put(5792,3940){\makebox(0,0)[r]{\strut{}Batch $\Nreproj = 600, \Treproj = 100$}}%
    }%
    \gplbacktext
    \put(0,0){\includegraphics{Chafee_OutErrCondensed_dim10}}%
    \gplfronttext
  \end{picture}%
\endgroup

%% file: figures/Chafee_OutDetailed_PO40_dim10.tex
\begingroup
  \makeatletter
  \providecommand\color[2][]{%
    \GenericError{(gnuplot) \space\space\space\@spaces}{%
      Package color not loaded in conjunction with
      terminal option `colourtext'%
    }{See the gnuplot documentation for explanation.%
    }{Either use 'blacktext' in gnuplot or load the package
      color.sty in LaTeX.}%
    \renewcommand\color[2][]{}%
  }%
  \providecommand\includegraphics[2][]{%
    \GenericError{(gnuplot) \space\space\space\@spaces}{%
      Package graphicx or graphics not loaded%
    }{See the gnuplot documentation for explanation.%
    }{The gnuplot epslatex terminal needs graphicx.sty or graphics.sty.}%
    \renewcommand\includegraphics[2][]{}%
  }%
  \providecommand\rotatebox[2]{#2}%
  \@ifundefined{ifGPcolor}{%
    \newif\ifGPcolor
    \GPcolortrue
  }{}%
  \@ifundefined{ifGPblacktext}{%
    \newif\ifGPblacktext
    \GPblacktexttrue
  }{}%
  \let\gplgaddtomacro\g@addto@macro
  \gdef\gplbacktext{}%
  \gdef\gplfronttext{}%
  \makeatother
  \ifGPblacktext
    \def\colorrgb#1{}%
    \def\colorgray#1{}%
  \else
    \ifGPcolor
      \def\colorrgb#1{\color[rgb]{#1}}%
      \def\colorgray#1{\color[gray]{#1}}%
      \expandafter\def\csname LTw\endcsname{\color{white}}%
      \expandafter\def\csname LTb\endcsname{\color{black}}%
      \expandafter\def\csname LTa\endcsname{\color{black}}%
      \expandafter\def\csname LT0\endcsname{\color[rgb]{1,0,0}}%
      \expandafter\def\csname LT1\endcsname{\color[rgb]{0,1,0}}%
      \expandafter\def\csname LT2\endcsname{\color[rgb]{0,0,1}}%
      \expandafter\def\csname LT3\endcsname{\color[rgb]{1,0,1}}%
      \expandafter\def\csname LT4\endcsname{\color[rgb]{0,1,1}}%
      \expandafter\def\csname LT5\endcsname{\color[rgb]{1,1,0}}%
      \expandafter\def\csname LT6\endcsname{\color[rgb]{0,0,0}}%
      \expandafter\def\csname LT7\endcsname{\color[rgb]{1,0.3,0}}%
      \expandafter\def\csname LT8\endcsname{\color[rgb]{0.5,0.5,0.5}}%
    \else
      \def\colorrgb#1{\color{black}}%
      \def\colorgray#1{\color[gray]{#1}}%
      \expandafter\def\csname LTw\endcsname{\color{white}}%
      \expandafter\def\csname LTb\endcsname{\color{black}}%
      \expandafter\def\csname LTa\endcsname{\color{black}}%
      \expandafter\def\csname LT0\endcsname{\color{black}}%
      \expandafter\def\csname LT1\endcsname{\color{black}}%
      \expandafter\def\csname LT2\endcsname{\color{black}}%
      \expandafter\def\csname LT3\endcsname{\color{black}}%
      \expandafter\def\csname LT4\endcsname{\color{black}}%
      \expandafter\def\csname LT5\endcsname{\color{black}}%
      \expandafter\def\csname LT6\endcsname{\color{black}}%
      \expandafter\def\csname LT7\endcsname{\color{black}}%
      \expandafter\def\csname LT8\endcsname{\color{black}}%
    \fi
  \fi
    \setlength{\unitlength}{0.0500bp}%
    \ifx\gptboxheight\undefined%
      \newlength{\gptboxheight}%
      \newlength{\gptboxwidth}%
      \newsavebox{\gptboxtext}%
    \fi%
    \setlength{\fboxrule}{0.5pt}%
    \setlength{\fboxsep}{1pt}%
\begin{picture}(7200.00,5040.00)%
    \gplgaddtomacro\gplbacktext{%
      \csname LTb\endcsname%
      \put(700,1077){\makebox(0,0)[r]{\strut{}$0$}}%
      \csname LTb\endcsname%
      \put(700,1802){\makebox(0,0)[r]{\strut{}$1$}}%
      \csname LTb\endcsname%
      \put(700,2528){\makebox(0,0)[r]{\strut{}$2$}}%
      \csname LTb\endcsname%
      \put(700,3253){\makebox(0,0)[r]{\strut{}$3$}}%
      \csname LTb\endcsname%
      \put(700,3978){\makebox(0,0)[r]{\strut{}$4$}}%
      \csname LTb\endcsname%
      \put(700,4703){\makebox(0,0)[r]{\strut{}$5$}}%
      \csname LTb\endcsname%
      \put(868,616){\makebox(0,0){\strut{}$0$}}%
      \csname LTb\endcsname%
      \put(1596,616){\makebox(0,0){\strut{}$0.5$}}%
      \csname LTb\endcsname%
      \put(2325,616){\makebox(0,0){\strut{}$1$}}%
      \csname LTb\endcsname%
      \put(3053,616){\makebox(0,0){\strut{}$1.5$}}%
      \csname LTb\endcsname%
      \put(3782,616){\makebox(0,0){\strut{}$2$}}%
      \csname LTb\endcsname%
      \put(4510,616){\makebox(0,0){\strut{}$2.5$}}%
      \csname LTb\endcsname%
      \put(5238,616){\makebox(0,0){\strut{}$3$}}%
      \csname LTb\endcsname%
      \put(5967,616){\makebox(0,0){\strut{}$3.5$}}%
      \csname LTb\endcsname%
      \put(6695,616){\makebox(0,0){\strut{}$4$}}%
      \put(2033,3180){\makebox(0,0)[l]{\strut{}\textbf{max train time}}}%
    }%
    \gplgaddtomacro\gplfronttext{%
      \csname LTb\endcsname%
      \put(224,2799){\rotatebox{-270}{\makebox(0,0){\strut{}output}}}%
      \put(3781,196){\makebox(0,0){\strut{}time}}%
      \csname LTb\endcsname%
      \put(5792,4500){\makebox(0,0)[r]{\strut{}full model}}%
      \csname LTb\endcsname%
      \put(5792,4220){\makebox(0,0)[r]{\strut{}Markovian reduced model}}%
      \csname LTb\endcsname%
      \put(5792,3940){\makebox(0,0)[r]{\strut{}non-Markovian reduced model (stage)}}%
      \csname LTb\endcsname%
      \put(5792,3660){\makebox(0,0)[r]{\strut{}non-Markovian reduced model (batch)}}%
    }%
    \gplbacktext
    \put(0,0){\includegraphics{Chafee_OutDetailed_PO40_dim10}}%
    \gplfronttext
  \end{picture}%
\endgroup

%% file: figures/Chafee_OutDetailed_PO60_dim10.tex
\begingroup
  \makeatletter
  \providecommand\color[2][]{%
    \GenericError{(gnuplot) \space\space\space\@spaces}{%
      Package color not loaded in conjunction with
      terminal option `colourtext'%
    }{See the gnuplot documentation for explanation.%
    }{Either use 'blacktext' in gnuplot or load the package
      color.sty in LaTeX.}%
    \renewcommand\color[2][]{}%
  }%
  \providecommand\includegraphics[2][]{%
    \GenericError{(gnuplot) \space\space\space\@spaces}{%
      Package graphicx or graphics not loaded%
    }{See the gnuplot documentation for explanation.%
    }{The gnuplot epslatex terminal needs graphicx.sty or graphics.sty.}%
    \renewcommand\includegraphics[2][]{}%
  }%
  \providecommand\rotatebox[2]{#2}%
  \@ifundefined{ifGPcolor}{%
    \newif\ifGPcolor
    \GPcolortrue
  }{}%
  \@ifundefined{ifGPblacktext}{%
    \newif\ifGPblacktext
    \GPblacktexttrue
  }{}%
  \let\gplgaddtomacro\g@addto@macro
  \gdef\gplbacktext{}%
  \gdef\gplfronttext{}%
  \makeatother
  \ifGPblacktext
    \def\colorrgb#1{}%
    \def\colorgray#1{}%
  \else
    \ifGPcolor
      \def\colorrgb#1{\color[rgb]{#1}}%
      \def\colorgray#1{\color[gray]{#1}}%
      \expandafter\def\csname LTw\endcsname{\color{white}}%
      \expandafter\def\csname LTb\endcsname{\color{black}}%
      \expandafter\def\csname LTa\endcsname{\color{black}}%
      \expandafter\def\csname LT0\endcsname{\color[rgb]{1,0,0}}%
      \expandafter\def\csname LT1\endcsname{\color[rgb]{0,1,0}}%
      \expandafter\def\csname LT2\endcsname{\color[rgb]{0,0,1}}%
      \expandafter\def\csname LT3\endcsname{\color[rgb]{1,0,1}}%
      \expandafter\def\csname LT4\endcsname{\color[rgb]{0,1,1}}%
      \expandafter\def\csname LT5\endcsname{\color[rgb]{1,1,0}}%
      \expandafter\def\csname LT6\endcsname{\color[rgb]{0,0,0}}%
      \expandafter\def\csname LT7\endcsname{\color[rgb]{1,0.3,0}}%
      \expandafter\def\csname LT8\endcsname{\color[rgb]{0.5,0.5,0.5}}%
    \else
      \def\colorrgb#1{\color{black}}%
      \def\colorgray#1{\color[gray]{#1}}%
      \expandafter\def\csname LTw\endcsname{\color{white}}%
      \expandafter\def\csname LTb\endcsname{\color{black}}%
      \expandafter\def\csname LTa\endcsname{\color{black}}%
      \expandafter\def\csname LT0\endcsname{\color{black}}%
      \expandafter\def\csname LT1\endcsname{\color{black}}%
      \expandafter\def\csname LT2\endcsname{\color{black}}%
      \expandafter\def\csname LT3\endcsname{\color{black}}%
      \expandafter\def\csname LT4\endcsname{\color{black}}%
      \expandafter\def\csname LT5\endcsname{\color{black}}%
      \expandafter\def\csname LT6\endcsname{\color{black}}%
      \expandafter\def\csname LT7\endcsname{\color{black}}%
      \expandafter\def\csname LT8\endcsname{\color{black}}%
    \fi
  \fi
    \setlength{\unitlength}{0.0500bp}%
    \ifx\gptboxheight\undefined%
      \newlength{\gptboxheight}%
      \newlength{\gptboxwidth}%
      \newsavebox{\gptboxtext}%
    \fi%
    \setlength{\fboxrule}{0.5pt}%
    \setlength{\fboxsep}{1pt}%
\begin{picture}(7200.00,5040.00)%
    \gplgaddtomacro\gplbacktext{%
      \csname LTb\endcsname%
      \put(700,1077){\makebox(0,0)[r]{\strut{}$0$}}%
      \csname LTb\endcsname%
      \put(700,1802){\makebox(0,0)[r]{\strut{}$1$}}%
      \csname LTb\endcsname%
      \put(700,2528){\makebox(0,0)[r]{\strut{}$2$}}%
      \csname LTb\endcsname%
      \put(700,3253){\makebox(0,0)[r]{\strut{}$3$}}%
      \csname LTb\endcsname%
      \put(700,3978){\makebox(0,0)[r]{\strut{}$4$}}%
      \csname LTb\endcsname%
      \put(700,4703){\makebox(0,0)[r]{\strut{}$5$}}%
      \csname LTb\endcsname%
      \put(868,616){\makebox(0,0){\strut{}$0$}}%
      \csname LTb\endcsname%
      \put(1596,616){\makebox(0,0){\strut{}$0.5$}}%
      \csname LTb\endcsname%
      \put(2325,616){\makebox(0,0){\strut{}$1$}}%
      \csname LTb\endcsname%
      \put(3053,616){\makebox(0,0){\strut{}$1.5$}}%
      \csname LTb\endcsname%
      \put(3782,616){\makebox(0,0){\strut{}$2$}}%
      \csname LTb\endcsname%
      \put(4510,616){\makebox(0,0){\strut{}$2.5$}}%
      \csname LTb\endcsname%
      \put(5238,616){\makebox(0,0){\strut{}$3$}}%
      \csname LTb\endcsname%
      \put(5967,616){\makebox(0,0){\strut{}$3.5$}}%
      \csname LTb\endcsname%
      \put(6695,616){\makebox(0,0){\strut{}$4$}}%
      \put(2033,3180){\makebox(0,0)[l]{\strut{}\textbf{max train time}}}%
    }%
    \gplgaddtomacro\gplfronttext{%
      \csname LTb\endcsname%
      \put(224,2799){\rotatebox{-270}{\makebox(0,0){\strut{}output}}}%
      \put(3781,196){\makebox(0,0){\strut{}time}}%
      \csname LTb\endcsname%
      \put(5792,4500){\makebox(0,0)[r]{\strut{}full model}}%
      \csname LTb\endcsname%
      \put(5792,4220){\makebox(0,0)[r]{\strut{}Markovian reduced model}}%
      \csname LTb\endcsname%
      \put(5792,3940){\makebox(0,0)[r]{\strut{}non-Markovian reduced model (stage)}}%
      \csname LTb\endcsname%
      \put(5792,3660){\makebox(0,0)[r]{\strut{}non-Markovian reduced model (batch)}}%
    }%
    \gplbacktext
    \put(0,0){\includegraphics{Chafee_OutDetailed_PO60_dim10}}%
    \gplfronttext
  \end{picture}%
\endgroup

%% file: figures/Chafee_OutDetailed_PO80_dim10.tex
\begingroup
  \makeatletter
  \providecommand\color[2][]{%
    \GenericError{(gnuplot) \space\space\space\@spaces}{%
      Package color not loaded in conjunction with
      terminal option `colourtext'%
    }{See the gnuplot documentation for explanation.%
    }{Either use 'blacktext' in gnuplot or load the package
      color.sty in LaTeX.}%
    \renewcommand\color[2][]{}%
  }%
  \providecommand\includegraphics[2][]{%
    \GenericError{(gnuplot) \space\space\space\@spaces}{%
      Package graphicx or graphics not loaded%
    }{See the gnuplot documentation for explanation.%
    }{The gnuplot epslatex terminal needs graphicx.sty or graphics.sty.}%
    \renewcommand\includegraphics[2][]{}%
  }%
  \providecommand\rotatebox[2]{#2}%
  \@ifundefined{ifGPcolor}{%
    \newif\ifGPcolor
    \GPcolortrue
  }{}%
  \@ifundefined{ifGPblacktext}{%
    \newif\ifGPblacktext
    \GPblacktexttrue
  }{}%
  \let\gplgaddtomacro\g@addto@macro
  \gdef\gplbacktext{}%
  \gdef\gplfronttext{}%
  \makeatother
  \ifGPblacktext
    \def\colorrgb#1{}%
    \def\colorgray#1{}%
  \else
    \ifGPcolor
      \def\colorrgb#1{\color[rgb]{#1}}%
      \def\colorgray#1{\color[gray]{#1}}%
      \expandafter\def\csname LTw\endcsname{\color{white}}%
      \expandafter\def\csname LTb\endcsname{\color{black}}%
      \expandafter\def\csname LTa\endcsname{\color{black}}%
      \expandafter\def\csname LT0\endcsname{\color[rgb]{1,0,0}}%
      \expandafter\def\csname LT1\endcsname{\color[rgb]{0,1,0}}%
      \expandafter\def\csname LT2\endcsname{\color[rgb]{0,0,1}}%
      \expandafter\def\csname LT3\endcsname{\color[rgb]{1,0,1}}%
      \expandafter\def\csname LT4\endcsname{\color[rgb]{0,1,1}}%
      \expandafter\def\csname LT5\endcsname{\color[rgb]{1,1,0}}%
      \expandafter\def\csname LT6\endcsname{\color[rgb]{0,0,0}}%
      \expandafter\def\csname LT7\endcsname{\color[rgb]{1,0.3,0}}%
      \expandafter\def\csname LT8\endcsname{\color[rgb]{0.5,0.5,0.5}}%
    \else
      \def\colorrgb#1{\color{black}}%
      \def\colorgray#1{\color[gray]{#1}}%
      \expandafter\def\csname LTw\endcsname{\color{white}}%
      \expandafter\def\csname LTb\endcsname{\color{black}}%
      \expandafter\def\csname LTa\endcsname{\color{black}}%
      \expandafter\def\csname LT0\endcsname{\color{black}}%
      \expandafter\def\csname LT1\endcsname{\color{black}}%
      \expandafter\def\csname LT2\endcsname{\color{black}}%
      \expandafter\def\csname LT3\endcsname{\color{black}}%
      \expandafter\def\csname LT4\endcsname{\color{black}}%
      \expandafter\def\csname LT5\endcsname{\color{black}}%
      \expandafter\def\csname LT6\endcsname{\color{black}}%
      \expandafter\def\csname LT7\endcsname{\color{black}}%
      \expandafter\def\csname LT8\endcsname{\color{black}}%
    \fi
  \fi
    \setlength{\unitlength}{0.0500bp}%
    \ifx\gptboxheight\undefined%
      \newlength{\gptboxheight}%
      \newlength{\gptboxwidth}%
      \newsavebox{\gptboxtext}%
    \fi%
    \setlength{\fboxrule}{0.5pt}%
    \setlength{\fboxsep}{1pt}%
\begin{picture}(7200.00,5040.00)%
    \gplgaddtomacro\gplbacktext{%
      \csname LTb\endcsname%
      \put(700,1077){\makebox(0,0)[r]{\strut{}$0$}}%
      \csname LTb\endcsname%
      \put(700,1802){\makebox(0,0)[r]{\strut{}$1$}}%
      \csname LTb\endcsname%
      \put(700,2528){\makebox(0,0)[r]{\strut{}$2$}}%
      \csname LTb\endcsname%
      \put(700,3253){\makebox(0,0)[r]{\strut{}$3$}}%
      \csname LTb\endcsname%
      \put(700,3978){\makebox(0,0)[r]{\strut{}$4$}}%
      \csname LTb\endcsname%
      \put(700,4703){\makebox(0,0)[r]{\strut{}$5$}}%
      \csname LTb\endcsname%
      \put(868,616){\makebox(0,0){\strut{}$0$}}%
      \csname LTb\endcsname%
      \put(1596,616){\makebox(0,0){\strut{}$0.5$}}%
      \csname LTb\endcsname%
      \put(2325,616){\makebox(0,0){\strut{}$1$}}%
      \csname LTb\endcsname%
      \put(3053,616){\makebox(0,0){\strut{}$1.5$}}%
      \csname LTb\endcsname%
      \put(3782,616){\makebox(0,0){\strut{}$2$}}%
      \csname LTb\endcsname%
      \put(4510,616){\makebox(0,0){\strut{}$2.5$}}%
      \csname LTb\endcsname%
      \put(5238,616){\makebox(0,0){\strut{}$3$}}%
      \csname LTb\endcsname%
      \put(5967,616){\makebox(0,0){\strut{}$3.5$}}%
      \csname LTb\endcsname%
      \put(6695,616){\makebox(0,0){\strut{}$4$}}%
      \put(2033,3180){\makebox(0,0)[l]{\strut{}\textbf{max train time}}}%
    }%
    \gplgaddtomacro\gplfronttext{%
      \csname LTb\endcsname%
      \put(224,2799){\rotatebox{-270}{\makebox(0,0){\strut{}output}}}%
      \put(3781,196){\makebox(0,0){\strut{}time}}%
      \csname LTb\endcsname%
      \put(5792,4500){\makebox(0,0)[r]{\strut{}full model}}%
      \csname LTb\endcsname%
      \put(5792,4220){\makebox(0,0)[r]{\strut{}Markovian reduced model}}%
      \csname LTb\endcsname%
      \put(5792,3940){\makebox(0,0)[r]{\strut{}non-Markovian reduced model (stage)}}%
      \csname LTb\endcsname%
      \put(5792,3660){\makebox(0,0)[r]{\strut{}non-Markovian reduced model (batch)}}%
    }%
    \gplbacktext
    \put(0,0){\includegraphics{Chafee_OutDetailed_PO80_dim10}}%
    \gplfronttext
  \end{picture}%
\endgroup

%% file: figures/LowPOLowDim.tex
\begingroup
  \makeatletter
  \providecommand\color[2][]{%
    \GenericError{(gnuplot) \space\space\space\@spaces}{%
      Package color not loaded in conjunction with
      terminal option `colourtext'%
    }{See the gnuplot documentation for explanation.%
    }{Either use 'blacktext' in gnuplot or load the package
      color.sty in LaTeX.}%
    \renewcommand\color[2][]{}%
  }%
  \providecommand\includegraphics[2][]{%
    \GenericError{(gnuplot) \space\space\space\@spaces}{%
      Package graphicx or graphics not loaded%
    }{See the gnuplot documentation for explanation.%
    }{The gnuplot epslatex terminal needs graphicx.sty or graphics.sty.}%
    \renewcommand\includegraphics[2][]{}%
  }%
  \providecommand\rotatebox[2]{#2}%
  \@ifundefined{ifGPcolor}{%
    \newif\ifGPcolor
    \GPcolortrue
  }{}%
  \@ifundefined{ifGPblacktext}{%
    \newif\ifGPblacktext
    \GPblacktexttrue
  }{}%
  \let\gplgaddtomacro\g@addto@macro
  \gdef\gplbacktext{}%
  \gdef\gplfronttext{}%
  \makeatother
  \ifGPblacktext
    \def\colorrgb#1{}%
    \def\colorgray#1{}%
  \else
    \ifGPcolor
      \def\colorrgb#1{\color[rgb]{#1}}%
      \def\colorgray#1{\color[gray]{#1}}%
      \expandafter\def\csname LTw\endcsname{\color{white}}%
      \expandafter\def\csname LTb\endcsname{\color{black}}%
      \expandafter\def\csname LTa\endcsname{\color{black}}%
      \expandafter\def\csname LT0\endcsname{\color[rgb]{1,0,0}}%
      \expandafter\def\csname LT1\endcsname{\color[rgb]{0,1,0}}%
      \expandafter\def\csname LT2\endcsname{\color[rgb]{0,0,1}}%
      \expandafter\def\csname LT3\endcsname{\color[rgb]{1,0,1}}%
      \expandafter\def\csname LT4\endcsname{\color[rgb]{0,1,1}}%
      \expandafter\def\csname LT5\endcsname{\color[rgb]{1,1,0}}%
      \expandafter\def\csname LT6\endcsname{\color[rgb]{0,0,0}}%
      \expandafter\def\csname LT7\endcsname{\color[rgb]{1,0.3,0}}%
      \expandafter\def\csname LT8\endcsname{\color[rgb]{0.5,0.5,0.5}}%
    \else
      \def\colorrgb#1{\color{black}}%
      \def\colorgray#1{\color[gray]{#1}}%
      \expandafter\def\csname LTw\endcsname{\color{white}}%
      \expandafter\def\csname LTb\endcsname{\color{black}}%
      \expandafter\def\csname LTa\endcsname{\color{black}}%
      \expandafter\def\csname LT0\endcsname{\color{black}}%
      \expandafter\def\csname LT1\endcsname{\color{black}}%
      \expandafter\def\csname LT2\endcsname{\color{black}}%
      \expandafter\def\csname LT3\endcsname{\color{black}}%
      \expandafter\def\csname LT4\endcsname{\color{black}}%
      \expandafter\def\csname LT5\endcsname{\color{black}}%
      \expandafter\def\csname LT6\endcsname{\color{black}}%
      \expandafter\def\csname LT7\endcsname{\color{black}}%
      \expandafter\def\csname LT8\endcsname{\color{black}}%
    \fi
  \fi
    \setlength{\unitlength}{0.0500bp}%
    \ifx\gptboxheight\undefined%
      \newlength{\gptboxheight}%
      \newlength{\gptboxwidth}%
      \newsavebox{\gptboxtext}%
    \fi%
    \setlength{\fboxrule}{0.5pt}%
    \setlength{\fboxsep}{1pt}%
\begin{picture}(7200.00,5040.00)%
    \gplgaddtomacro\gplbacktext{%
      \csname LTb\endcsname%
      \put(1372,896){\makebox(0,0)[r]{\strut{}$-0.15$}}%
      \csname LTb\endcsname%
      \put(1372,1242){\makebox(0,0)[r]{\strut{}$-0.1$}}%
      \csname LTb\endcsname%
      \put(1372,1588){\makebox(0,0)[r]{\strut{}$-0.05$}}%
      \csname LTb\endcsname%
      \put(1372,1934){\makebox(0,0)[r]{\strut{}$0$}}%
      \csname LTb\endcsname%
      \put(1372,2280){\makebox(0,0)[r]{\strut{}$0.05$}}%
      \csname LTb\endcsname%
      \put(1372,2626){\makebox(0,0)[r]{\strut{}$0.1$}}%
      \csname LTb\endcsname%
      \put(1372,2973){\makebox(0,0)[r]{\strut{}$0.15$}}%
      \csname LTb\endcsname%
      \put(1372,3319){\makebox(0,0)[r]{\strut{}$0.2$}}%
      \csname LTb\endcsname%
      \put(1372,3665){\makebox(0,0)[r]{\strut{}$0.25$}}%
      \csname LTb\endcsname%
      \put(1372,4011){\makebox(0,0)[r]{\strut{}$0.3$}}%
      \csname LTb\endcsname%
      \put(1372,4357){\makebox(0,0)[r]{\strut{}$0.35$}}%
      \csname LTb\endcsname%
      \put(1372,4703){\makebox(0,0)[r]{\strut{}$0.4$}}%
      \csname LTb\endcsname%
      \put(1540,616){\makebox(0,0){\strut{}$0$}}%
      \csname LTb\endcsname%
      \put(2056,616){\makebox(0,0){\strut{}$5$}}%
      \csname LTb\endcsname%
      \put(2571,616){\makebox(0,0){\strut{}$10$}}%
      \csname LTb\endcsname%
      \put(3087,616){\makebox(0,0){\strut{}$15$}}%
      \csname LTb\endcsname%
      \put(3602,616){\makebox(0,0){\strut{}$20$}}%
      \csname LTb\endcsname%
      \put(4118,616){\makebox(0,0){\strut{}$25$}}%
      \csname LTb\endcsname%
      \put(4633,616){\makebox(0,0){\strut{}$30$}}%
      \csname LTb\endcsname%
      \put(5149,616){\makebox(0,0){\strut{}$35$}}%
      \csname LTb\endcsname%
      \put(5664,616){\makebox(0,0){\strut{}$40$}}%
      \csname LTb\endcsname%
      \put(6180,616){\makebox(0,0){\strut{}$45$}}%
      \csname LTb\endcsname%
      \put(6695,616){\makebox(0,0){\strut{}$50$}}%
    }%
    \gplgaddtomacro\gplfronttext{%
      \csname LTb\endcsname%
      \put(224,2799){\rotatebox{-270}{\makebox(0,0){\strut{}difference in relative error}}}%
      \put(4117,196){\makebox(0,0){\strut{}time step}}%
    }%
    \gplbacktext
    \put(0,0){\includegraphics{LowPOLowDim}}%
    \gplfronttext
  \end{picture}%
\endgroup

%% file: figures/HighPOHighDim.tex
\begingroup
  \makeatletter
  \providecommand\color[2][]{%
    \GenericError{(gnuplot) \space\space\space\@spaces}{%
      Package color not loaded in conjunction with
      terminal option `colourtext'%
    }{See the gnuplot documentation for explanation.%
    }{Either use 'blacktext' in gnuplot or load the package
      color.sty in LaTeX.}%
    \renewcommand\color[2][]{}%
  }%
  \providecommand\includegraphics[2][]{%
    \GenericError{(gnuplot) \space\space\space\@spaces}{%
      Package graphicx or graphics not loaded%
    }{See the gnuplot documentation for explanation.%
    }{The gnuplot epslatex terminal needs graphicx.sty or graphics.sty.}%
    \renewcommand\includegraphics[2][]{}%
  }%
  \providecommand\rotatebox[2]{#2}%
  \@ifundefined{ifGPcolor}{%
    \newif\ifGPcolor
    \GPcolortrue
  }{}%
  \@ifundefined{ifGPblacktext}{%
    \newif\ifGPblacktext
    \GPblacktexttrue
  }{}%
  \let\gplgaddtomacro\g@addto@macro
  \gdef\gplbacktext{}%
  \gdef\gplfronttext{}%
  \makeatother
  \ifGPblacktext
    \def\colorrgb#1{}%
    \def\colorgray#1{}%
  \else
    \ifGPcolor
      \def\colorrgb#1{\color[rgb]{#1}}%
      \def\colorgray#1{\color[gray]{#1}}%
      \expandafter\def\csname LTw\endcsname{\color{white}}%
      \expandafter\def\csname LTb\endcsname{\color{black}}%
      \expandafter\def\csname LTa\endcsname{\color{black}}%
      \expandafter\def\csname LT0\endcsname{\color[rgb]{1,0,0}}%
      \expandafter\def\csname LT1\endcsname{\color[rgb]{0,1,0}}%
      \expandafter\def\csname LT2\endcsname{\color[rgb]{0,0,1}}%
      \expandafter\def\csname LT3\endcsname{\color[rgb]{1,0,1}}%
      \expandafter\def\csname LT4\endcsname{\color[rgb]{0,1,1}}%
      \expandafter\def\csname LT5\endcsname{\color[rgb]{1,1,0}}%
      \expandafter\def\csname LT6\endcsname{\color[rgb]{0,0,0}}%
      \expandafter\def\csname LT7\endcsname{\color[rgb]{1,0.3,0}}%
      \expandafter\def\csname LT8\endcsname{\color[rgb]{0.5,0.5,0.5}}%
    \else
      \def\colorrgb#1{\color{black}}%
      \def\colorgray#1{\color[gray]{#1}}%
      \expandafter\def\csname LTw\endcsname{\color{white}}%
      \expandafter\def\csname LTb\endcsname{\color{black}}%
      \expandafter\def\csname LTa\endcsname{\color{black}}%
      \expandafter\def\csname LT0\endcsname{\color{black}}%
      \expandafter\def\csname LT1\endcsname{\color{black}}%
      \expandafter\def\csname LT2\endcsname{\color{black}}%
      \expandafter\def\csname LT3\endcsname{\color{black}}%
      \expandafter\def\csname LT4\endcsname{\color{black}}%
      \expandafter\def\csname LT5\endcsname{\color{black}}%
      \expandafter\def\csname LT6\endcsname{\color{black}}%
      \expandafter\def\csname LT7\endcsname{\color{black}}%
      \expandafter\def\csname LT8\endcsname{\color{black}}%
    \fi
  \fi
    \setlength{\unitlength}{0.0500bp}%
    \ifx\gptboxheight\undefined%
      \newlength{\gptboxheight}%
      \newlength{\gptboxwidth}%
      \newsavebox{\gptboxtext}%
    \fi%
    \setlength{\fboxrule}{0.5pt}%
    \setlength{\fboxsep}{1pt}%
\begin{picture}(7200.00,5040.00)%
    \gplgaddtomacro\gplbacktext{%
      \csname LTb\endcsname%
      \put(1372,896){\makebox(0,0)[r]{\strut{}$-0.02$}}%
      \csname LTb\endcsname%
      \put(1372,1319){\makebox(0,0)[r]{\strut{}$-0.01$}}%
      \csname LTb\endcsname%
      \put(1372,1742){\makebox(0,0)[r]{\strut{}$0$}}%
      \csname LTb\endcsname%
      \put(1372,2165){\makebox(0,0)[r]{\strut{}$0.01$}}%
      \csname LTb\endcsname%
      \put(1372,2588){\makebox(0,0)[r]{\strut{}$0.02$}}%
      \csname LTb\endcsname%
      \put(1372,3011){\makebox(0,0)[r]{\strut{}$0.03$}}%
      \csname LTb\endcsname%
      \put(1372,3434){\makebox(0,0)[r]{\strut{}$0.04$}}%
      \csname LTb\endcsname%
      \put(1372,3857){\makebox(0,0)[r]{\strut{}$0.05$}}%
      \csname LTb\endcsname%
      \put(1372,4280){\makebox(0,0)[r]{\strut{}$0.06$}}%
      \csname LTb\endcsname%
      \put(1372,4703){\makebox(0,0)[r]{\strut{}$0.07$}}%
      \csname LTb\endcsname%
      \put(1540,616){\makebox(0,0){\strut{}$0$}}%
      \csname LTb\endcsname%
      \put(2056,616){\makebox(0,0){\strut{}$5$}}%
      \csname LTb\endcsname%
      \put(2571,616){\makebox(0,0){\strut{}$10$}}%
      \csname LTb\endcsname%
      \put(3087,616){\makebox(0,0){\strut{}$15$}}%
      \csname LTb\endcsname%
      \put(3602,616){\makebox(0,0){\strut{}$20$}}%
      \csname LTb\endcsname%
      \put(4118,616){\makebox(0,0){\strut{}$25$}}%
      \csname LTb\endcsname%
      \put(4633,616){\makebox(0,0){\strut{}$30$}}%
      \csname LTb\endcsname%
      \put(5149,616){\makebox(0,0){\strut{}$35$}}%
      \csname LTb\endcsname%
      \put(5664,616){\makebox(0,0){\strut{}$40$}}%
      \csname LTb\endcsname%
      \put(6180,616){\makebox(0,0){\strut{}$45$}}%
      \csname LTb\endcsname%
      \put(6695,616){\makebox(0,0){\strut{}$50$}}%
    }%
    \gplgaddtomacro\gplfronttext{%
      \csname LTb\endcsname%
      \put(224,2799){\rotatebox{-270}{\makebox(0,0){\strut{}difference in relative error}}}%
      \put(4117,196){\makebox(0,0){\strut{}time step}}%
    }%
    \gplbacktext
    \put(0,0){\includegraphics{HighPOHighDim}}%
    \gplfronttext
  \end{picture}%
\endgroup